%% file: main.tex
\documentclass{article}

\PassOptionsToPackage{numbers, compress}{natbib} %
\usepackage[final]{neurips_2024}

\usepackage[utf8]{inputenc} %
\usepackage[T1]{fontenc}    %
\usepackage{booktabs}       %
\usepackage{wrapfig}
\usepackage{nicefrac}       %
\usepackage{microtype}      %
\usepackage{xcolor}         %

\usepackage{microtype}
\usepackage{graphicx}
\usepackage{float}
\usepackage[toc,page,header]{appendix} %
\usepackage{minitoc}

\usepackage{hyperref}
\usepackage{url}            %

\usepackage{amsmath}
\usepackage{amsfonts}       %
\usepackage{amssymb}
\usepackage{mathtools}
\usepackage{amsthm}
\usepackage{multirow}
\usepackage{xspace}
\usepackage{enumitem}

\DeclareMathOperator*{\argmin}{argmin}

\usepackage[capitalize,noabbrev]{cleveref}

\theoremstyle{plain}
\newtheorem{theorem}{Theorem}[section]

\theoremstyle{definition}

\theoremstyle{remark}

\input{custom_definitions}

\usepackage{graphicx}
\usepackage{multirow}
\usepackage{booktabs}
\usepackage{makecell}
\usepackage{wrapfig}
\usepackage{subcaption}   

\usepackage[textsize=tiny]{todonotes}
\usepackage{algorithmic}
\usepackage{algorithm}
\usepackage{bm}

\newcommand{\cm}[1]{{\color{black}#1}}

\title{\proj: A Generative Approach to Control \\
Complex Physical Systems}

\author{%
    Long Wei$^{1}$\thanks{Equal contribution. $^\mathsection$Work done as an intern at Westlake University. $^\dagger$Corresponding author.} \quad Peiyan Hu$^{2*\mathsection}$ \quad Ruiqi Feng$^{1*}$ \quad
    Haodong Feng$^{1}$ \quad
    Yixuan Du$^{3\mathsection}$ \quad Tao Zhang$^{1}$ \\
    \textbf{Rui Wang$^{4\mathsection}$ \quad Yue Wang$^{5}$ \quad Zhi-Ming Ma$^{2}$ \quad Tailin Wu$^{1\dagger}$} \\
    $^1$School of Engineering, Westlake University, \\
    $^2$Academy of Mathematics and Systems Science, Chinese Academy of Sciences, \\
    $^3$Jilin University, \:
    $^4$Fudan University, \:
    $^5$Microsoft AI4Science \\
    \texttt{\{weilong,hupeiyan,fengruiqi,wutailin\}@westlake.edu.cn}
}

\begin{document}
\maketitle

\doparttoc %
\faketableofcontents %

\input{text/00_abstract}

\section{Introduction}
\input{text/01_introduction}
\section{Background}
\input{text/02_background}

\section{Method}
\input{text/03_method}
\section{Experiments}
\input{text/04_experiments}
\vspace{-0.03in}
\section{Conclusion}
\vspace{-0.02in}
\input{text/05_conclusion}
\section{Acknowledgment}
\vspace{-0.02in}
\input{text/06_acknowledgement}

\bibliography{references}
\bibliographystyle{plain}

\newpage

\clearpage
\appendix
\onecolumn

\addcontentsline{toc}{section}{Appendix} %
\part{Appendix} %
\parttoc %

\input{text/07_appendix}

\end{document}

%% file: custom_definitions.tex
\newcommand{\proj}{DiffPhyCon\xspace}

\newcommand{\ugen}{\mathbf{u}_{[1,T]}}
\newcommand{\wgen}{\mathbf{w}_{[0,T-1]}}

\renewcommand{\u}{\mathbf{u}}
\renewcommand{\c}{\mathbf{c}}
\newcommand{\w}{\mathbf{w}}
\newcommand{\z}{\mathbf{z}}
\newcommand{\x}{\mathbf{x}}
\renewcommand{\v}{\mathbf{v}}

\newcommand{\bepsilon}{\bm{\epsilon}}
\def\J{\mathcal{J}}

\def\controlobj1d{\mathcal{J}_{\text{actual}}}

%% file: text/00_abstract.tex
\begin{abstract}
Controlling the evolution of complex physical systems is a fundamental task across science and engineering. 
Classical techniques suffer from limited applicability or huge computational costs. On the other hand, recent deep learning and reinforcement learning-based approaches often struggle to optimize long-term control sequences under the constraints of system dynamics. In this work, we introduce \underline{Diff}usion \underline{Phy}sical systems \underline{Con}trol (\proj), a new class of method to address the physical systems control problem. \proj excels by simultaneously minimizing both the learned generative energy function and the predefined control objectives across the entire trajectory and control sequence. Thus, it can explore globally and plan near-optimal control sequences. Moreover, we enhance \proj with prior reweighting, enabling the discovery of control sequences that significantly deviate from the training distribution. We test our method on three tasks: 1D Burgers' equation, 2D jellyfish movement control, and 2D high-dimensional smoke control, where our generated jellyfish dataset is released as a benchmark for complex physical system control research. Our method outperforms widely applied classical approaches and state-of-the-art deep learning and reinforcement learning methods. Notably, \proj unveils an intriguing fast-close-slow-open pattern observed in the jellyfish, aligning with established findings in the field of fluid dynamics. The project website, jellyfish dataset, and code can be found at \url{https://github.com/AI4Science-WestlakeU/diffphycon}.
\end{abstract}

%% file: text/01_introduction.tex
Modeling the dynamics of complex physical systems is an important class of problems in science and engineering.
Usually, we are not only interested in predicting a physical system's behavior but also injecting time-variant signals to steer its evolution and optimize specific objectives. This gives rise to the complex physical control problem, a fundamental task with wide applications, such as controlled nuclear fusion \cite{degrave2022magnetic}, fluid control \cite{holl2020learning}, underwater devices \cite{zhang2022pde} and aviation \cite{paranjape2013pde}, among others.

Despite its importance, controlling complex physical systems efficiently presents significant challenges. 
It inherits the fundamental challenge of simulating complex physical systems, which are typically high-dimensional and highly nonlinear, as specified in Appendix \ref{app:related_simulation}. Furthermore, observed control signals and corresponding system trajectories for optimizing a control model are typically far from the optimal solutions of the specific control objective. This fact poses a significant challenge of dilemma: \emph{How to explore long-term control sequences beyond its training distribution to seek near-optimal solutions while making the resulting system trajectory faithful to the dynamics of the physical system}?

To tackle physical systems control problems, various techniques have been proposed, yet they fall short of addressing the above challenges. 
Regarding traditional control methods, the Proportional-Integral-Derivative (PID) control \cite{1580152}, is efficient 
but only suitable for a limited range of problems.
Conversely, Model Predictive Control (MPC) \cite{schwenzer2021review}, despite a wider range of applicability, suffers from high computational costs and challenges in global optimization. 
Recent advances in supervised learning (SL) \cite{holl2020learning,hwang2022solving} and reinforcement learning (RL) \cite{farahmand2017deep, pan2018reinforcement, rabault2019artificial}, trained on system trajectories and control signals data, have demonstrated impressive performance in solving physical systems control problems. 
However, existing SL and model-based RL methods either fall into myopic failure modes \cite{pmlr-v162-janner22a} that fail to achieve long-term near-optimal solutions, or produce adversarial state trajectories \cite{zhao2022learning} that violate the physical system's dynamics.
The main reason may be that they treat the continuous evolution of dynamics from an iterative view, both lacking long-term vision and struggling in global optimization. See Appendix \ref{app:related_control} for more related work on physical system control.

In this work, we introduce \underline{Diff}usion \underline{Phy}sical systems \underline{Con}trol (\proj), a \emph{new class} of method to address the physical systems control problem. 
We take an energy optimization perspective over system trajectory and control sequences across the whole horizon to implicitly capture the constraints inherent in system dynamics. We accomplish this through diffusion models, which are trained using system trajectory data and control sequences. 
In the inference stage, \proj integrates simulation and control optimization into a unified energy optimization process. 
This prevents the generated system dynamics from falling out of distribution, and offers an enhanced perspective over long-term dynamics, facilitating the discovery of control sequences that optimize the objectives.

An essential aspect of physical systems control lies in its capacity to generate near-optimal controls, even when they may deviate significantly from the training distribution. We address this challenge with the key insight that the learned generative energy landscape can be \emph{decomposed} into two components: a prior distribution representing the control sequence and a conditional distribution characterizing the system trajectories given the control sequence. Based on this insight, we develop a \emph{prior reweighting} technique to subtract the effect of the prior distribution of control sequences, with adjustable strength, from the overall joint generative energy landscape during inference. 

In summary, we contribute the following: 
\textbf{(1)} We develop \proj, a novel generative method to control complex physical systems. By optimizing trajectory and control sequences jointly in the entire horizon by diffusion models, \proj facilitates global optimization of long-term dynamics and helps to reduce myopic failure modes.
\textbf{(2)} We introduce the prior reweighting technique to plan control sequences that are superior to those in the training set. 
\textbf{(3)} We demonstrate the effectiveness of our method on 1D Burgers’ equation, 2D jellyfish movement control, and 2D high-dimensional smoke control tasks. On all three tasks, our method outperforms widely applied classical control methods and is also competitive with recent supervised learning and strong reinforcement learning baselines, particularly demonstrating advantages of \proj in scenarios with partial observations and partial/indirect control. Notably, \proj reveals the intriguing fast-close-slow-open pattern exhibited by the jellyfish, aligning with findings in the field of fluid dynamics \cite{kang2023propulsive}.
\textbf{(4)} We generate a dataset of 2D jellyfish based on a CFD (computational fluid dynamic) software to mimic the movement of a jellyfish under control signals of its flapping behavior. To advance research in controlling complex physical systems, we release our dataset as a benchmark. 

%% file: text/02_background.tex
\begin{figure*}[t]
\vspace{-0pt}
\begin{center}
    \includegraphics[scale=0.455]{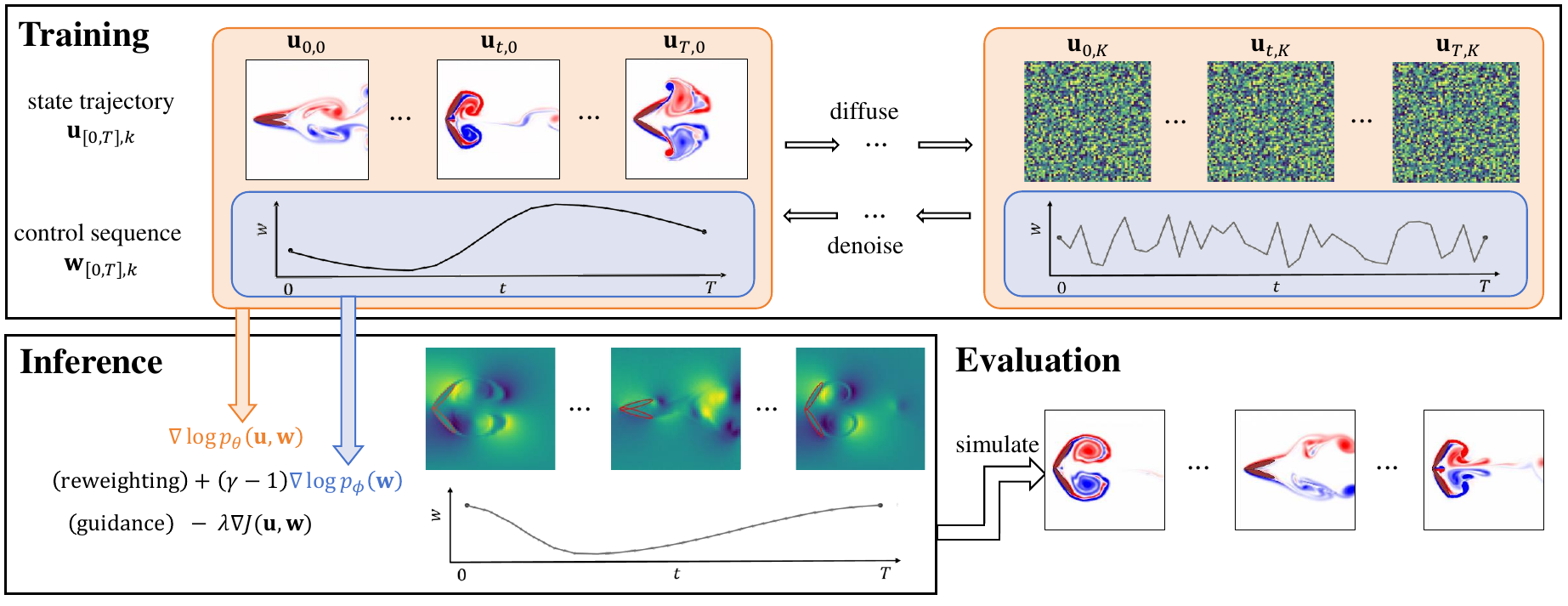}
\end{center}
\vspace{-10pt}
\caption{\textbf{Overview of \proj}. The figure depicts the training (top), inference (bottom left), and evaluation (bottom right) of \proj. Orange and blue colors respectively represent models learning the joint distribution $p_{\theta}(\u,\w)$ and the prior distribution $p_{\phi}(\w)$. Through prior reweighting and guidance, \proj is capable of generating superior control sequences.}
\vspace{-6pt}
\label{fig:overview}
\end{figure*}

\subsection{Problem Setup}
\label{sec:problem_setup}
We consider the following complex physical system control problem:
\begin{gather}
    \w^*=\argmin_\w\mathcal{J}(\u,\w)\quad\text{s.t.}\quad\mathcal{C}(\u,\w)=0.
    \label{eq:origion_optimization}
\end{gather}
Here $\u(t,\x):[0,\mathcal{T}]\times \Omega\mapsto \mathbb{R}^{d_\u}$ is the system trajectory $\{\u(t,\cdot),t\in[0,\mathcal{T}]\}$ defined on time range $[0,\mathcal{T}]\subset\mathbb{R}$ and spatial domain $\Omega\subset \mathbb{R}^{D}$, and $\w(t,\x):[0,\mathcal{T}]\times \Omega\mapsto \mathbb{R}^{d_\w}$ is the external control signal for the physical system with dimension ${d_\w}$.
$\mathcal{J}(\u,\w)$ denotes the control objective. For example, $\J$ can be designed to measure the control performance towards a target state  $\u^*$ with cost constraints: $\J\coloneqq\int \|\u-\u^*\|^2\mathrm{d}\x\mathrm{d}t+ \int\|\w\|^2\mathrm{d}\x\mathrm{d}t$.
$\mathcal{C}(\u,\w)=0$ denotes the physical constraints.
$\mathcal{C}(\u,\w)=0$ can be specified either \emph{explicitly} by a PDE dynamic which describes how the control signal $\w(t,\x)$ drives the trajectory $\u(t,\x)$ to evolve under boundary and initial conditions, or \emph{implicitly} by control sequences and trajectory data collected from observation of the physical system. In the latter case, we may even have access to only partial trajectory data or engage in partial control. These situations collectively pose significant challenges to the physical system control task.

\subsection{Preliminary: Diffusion Models}
\label{sec:ddpm}
Diffusion models \cite{ho2020denoising}  are a class of generative models that learn data distribution from data.
Diffusion models consist of two opposite processes: the forward process $q(\x_{k+1}|\x_k)=\mathcal{N}(\x_{k+1};\sqrt{\alpha_k}\x_k,(1-\alpha_k)\mathbf{I})$
to corrupt a clean data $\x_0$ to a Gaussian noise $\x_K\sim \mathcal{N}(\mathbf{0}, \mathbf{I})$, and the reverse parametrized process $p_{\theta}(\x_{k-1}|\x_k)=\mathcal{N}(\x_{k-1};\mu_\theta(\x_k,k),\sigma_k \mathbf{I})$ to denoise from standard Gaussian $\x_K\sim\mathcal{N}(\mathbf{0}, \mathbf{I})$, where $\{\alpha_k\}_{k=1}^K$ is the variance schedule. 
To train diffusion models, \cite{ho2020denoising}
propose the DDPM method to minimize the following training loss for the denoising network $\bepsilon_{\theta}$, a simplification of the evidence lower bound (ELBO) for the log-likelihood of the data: 
\begin{align}
\mathcal{L}=\mathbb{E}_{k\sim U(1,K),\x_0\sim p(x),\bepsilon\sim \mathcal{N}(\mathbf{0},\mathbf{I})}[\|\bepsilon - \bepsilon_{\theta}(\sqrt{\bar{\alpha}_k}\x_0 +  \sqrt{1-\bar{\alpha}_k} \bepsilon,k)\|_2^2],
\end{align}
where $\Bar{\alpha}_k:=\prod_{i=1}^k\alpha_i$. $\bepsilon_{\theta}$ estimates the noise to be removed to recover data $\x_0$. During inference, iterative application of $\bepsilon_{\theta}$ from a Gaussian noise could generate a new sample $\x_0$ that approximately follows the data distribution $p(\x)$. See Appendix \ref{app:related_diffusion} for related work on diffusion models.

\textbf{Notation}. 
We use $\v_{[n,m]}=[\v_n,\cdots,\v_m]$ to denote a sequence of variables. We use $\z_{[0,T-1],k}$ to denote the hidden variable of $\z_{[0,T-1]}$ in a diffusion step $k$.
For simplicity, we abbreviate $\wgen$, $\ugen$ as $\w$, $\u$. Concatenation of two variables is denoted via \emph{e.g.}, $[\u,\w]$.

%% file: text/03_method.tex
In this section, we detail our method \proj.  In Section \ref{sec:1ddpm}, we introduce our method including its training and inference. In Section \ref{sec:2ddpm}, we further propose a prior reweighting technique to improve \proj.
The overview of \proj is illustrated in Figure \ref{fig:overview}.

\subsection{Generative Control by Diffusion Models}

\label{sec:1ddpm}
\proj takes an energy optimization perspective to solve the problem Eq.~\eqref{eq:origion_optimization}, where PDE constraints can be modeled as a parameterized energy-based model (EBM) $E_{\theta}(\u,\w,\mathbf{c})$ which characterizes the distribution $p(\u,\w|\mathbf{c})$ of $\u$ and $\w$ conditioned on conditions $\mathbf{c}$ by the correspondence  $p(\u,\w|\mathbf{c})\propto \exp({-E_{\theta}(\u,\w,\mathbf{c})})$. Lower $E_\theta(\u,\w,\mathbf{c})$, or equivalently higher $p(\u,\w|\mathbf{c})$, means better satisfaction of the PDE constraints. Then the problem Eq.~\eqref{eq:origion_optimization} can be converted to:
\vspace{-2pt}
\begin{gather}
\label{eq:joint_optimization}
\u^*, \w^* = \argmin_{\u, \w}\left[E_\theta(\u,\w,\mathbf{c}) + \lambda\cdot \mathcal{J}(\u,\w)\right],
\end{gather}
where $\lambda$ is a hyperparameter. This formulation optimizes $\u$ and $\w$ of all physical time steps simultaneously. The first term encourages the generated $\w$ and $\u$ to satisfy the PDE constraints $\mathcal{C}(\u,\w)=0$. The second term guides optimization towards optimal objectives.
The advantage of this optimization framework is that it tasks a global optimization on $\u$ and $\w$ of all times steps, which may obtain better solutions that are faithful to the dynamics of the physical system. 
Details about the effect of the hyperparameter $\lambda$ are provided in Appendix \ref{app:hyperparams}.

\textbf{Training.} To train $E_{\theta}$, we exploit the diffusion model to estimate the gradient of $E_{\theta}$. For convenience, we introduce a new variable $\z$ to represent the concatenation of $\u$ and $\w$ as $\z=[\u,\w]$.
We use a denoising network $\bepsilon_\theta$ \cite{ho2020denoising}, which approximates $\nabla E_\theta(\z,\mathbf{c})$ \cite{du2023reduce}, to learn the noise that should be denoised in each diffusion step $k=1,\cdots,K$.
The training loss of  $\bepsilon_\theta$ is:
\begin{equation}
\label{eq:training_obj}
\mathcal{L}=\mathbb{E}_{k\sim U(1,K),(\z,\mathbf{c})\sim p(\z,\mathbf{c}),\bepsilon\sim \mathcal{N}(\mathbf{0},\mathbf{I})}[\|\mathbf{\bepsilon} - \bepsilon_\theta(\sqrt{\bar{\alpha}_k}\z +  \sqrt{1-\bar{\alpha}_k} \bepsilon,\mathbf{c},k)\|_2^2],\ \ \ \end{equation}
where $\bepsilon_\theta$ is conditioned on $\mathbf{c}$.
Regarding training datasets, they could be generated by designing the conditions $\mathbf{c}$ and control sequences followed by a simulation when an explicit PDE form is available. Otherwise, they should be collected from observed pairs of control sequences and trajectories.

\textbf{Control optimization.} After the denoising network is trained, the Eq. \eqref{eq:joint_optimization} can be optimized by the Langevin sampling procedure as follows. 
We start from an initial sample $\z_K\sim\mathcal{N}(\mathbf{0},\mathbf{I})$, and iteratively run the following process 
\begin{equation}
\begin{gathered}
\label{eq:denoise_iteration}
 \z_{k-1} = \z_k - \eta\left(\nabla_{\z} (E_{\theta}(\z_k,\mathbf{c}) +\lambda\mathcal{J}(\hat{\z}_k)\right) + \mathbf{\xi},  
 \quad \mathbf{\xi} \sim \mathcal{N} \bigl(\mathbf{0}, \sigma^2_k \mathbf{I} \bigl),
\end{gathered}
\end{equation}
 where $\sigma^2_k$ and $\eta$ correspond to noise schedules and scaling factors used in the diffusion process, respectively. Here $\hat{\z}_k$ is the approximate noise-free $\z_0$ estimated from $\z_k$ by:
\begin{equation}
\label{eq:estimate_x0}
\hat{\z}_k=(\z_k-\sqrt{1-\Bar{\alpha}_k}\bepsilon_\theta(\z_k, \mathbf{c}, k))/\sqrt{\Bar{\alpha}_k}.
\end{equation}
We calculate $\mathcal{J}$ in \eqref{eq:denoise_iteration} based on $\hat{\z}_k$ instead of directly using $\z_k$ because otherwise noise in $\z_k$ could bring errors to $\mathcal{J}$.
Then $\nabla E_{\theta}$ can be replaced by our trained denoising network $\bepsilon_\theta$ as follows:
\vspace{-4pt}
\begin{equation}
\label{eq:1ddpm_inference}
 \z_{k-1} = \z_k - \eta\left(\bepsilon_\theta(\z_k,\mathbf{c}, k) +\lambda\nabla_{\z}\mathcal{J}(\hat{\z}_k)\right) + \xi, \quad \mathbf{\xi} \sim \mathcal{N} \bigl(\mathbf{0}, \sigma^2_k \mathbf{I} \bigl)
\end{equation}
Iteration of this denoising process for $k=K,K-1,..., 1$ yields a final solution $\z_0=\{\u_{[1,T],0},\w_{[0,T-1],0}\}$ for the optimization problem Eq. \eqref{eq:joint_optimization}.

\textbf{Guidance conditioning.}
\label{sec:conditional_guidance}
In addition to the above introduced explicit guidance, conditioning is also widely used to guide sampling in diffusion models \cite{ho2022classifier,shuphysics2023}.
When the control objective can be naturally expressed in a conditioning form, e.g., 
the generated trajectory $\u$ is required to coincide with a desired target $\u^*$,
we can include $\u=\u^*$ as a condition in $\c$ such that the sampled trajectory $\u$ from diffusion models automatically satisfying $\u=\u^*$.

Overall, in our proposed \proj framework, the control objective $\J$ can be optimized either using the explicit guidance $\nabla \J$ or guidance conditioning depending on the specific control objectives.

\subsection{Prior Reweighting}
\label{sec:2ddpm}

\textbf{Motivation.}
In physical systems control, a critical challenge lies in obtaining control sequences superior to those in training datasets, which often deviate significantly from achieving the optimal control objective. Although guidance of the control objective is incorporated in our diffusion model, generated control sequences are still highly influenced by the prior distribution of control sequences in training datasets. This inspires us to explore strategies to mitigate the effect of this prior, aiming to generate near-optimal control sequences.

We address this challenge with the key insight that the energy-based model $E(\u,\w,\mathbf{c})$ can be decomposed into two components: one is $E^{(p)}(\w,\mathbf{c})$ derived from the prior distribution $p(\w|\mathbf{c})$ of control sequences, and the other $E^{(c)}(\u,\w,\mathbf{c})$ representing the conditional probability distribution $p(\u|\w,\mathbf{c})$ of trajectories with respect to given control sequences. This decomposition has the following form 
\vspace{-3pt}
\begin{equation}
\label{eq:engery_decomp}
E(\u,\w,\mathbf{c})
=E^{(p)}(\u,\mathbf{c})+E^{(c)}(\u,\w,\mathbf{c}),
\end{equation}
by the corresponding decomposition of distribution $p(\u,\w|\mathbf{c})=p(\w|\mathbf{c})p(\u|\w,\mathbf{c})$. We propose a \textit{prior reweighting} technique, which introduces an adjustable hyperparameter $\gamma>0$ as an exponential of $p(\w|\mathbf{c})$, allowing for the tuning of the influence of this prior distribution. Then we have a reweighted version of $p(\u,\w|\mathbf{c})$ as 
$p_\gamma(\u,\w|\mathbf{c})=p(\w|\mathbf{c})^{\gamma}p(\u|\w,\mathbf{c})/Z$,
which is also a probability distribution and can be further transformed to 
\begin{equation}
\label{eq:p_gamma}
p_\gamma(\u,\w|\mathbf{c})=p(\w|\mathbf{c})^{\gamma-1}p(\u,\w|\mathbf{c})/Z,
\end{equation}
where $Z$ is a normalization constant. In particular, when $0<\gamma<1$, this approach is advantageous as it flattens the distribution $p(\u,\w|\mathbf{c})$, thereby increasing the likelihood of sampling from low probability region of $p(\u,\w|\mathbf{c})$, where the optimal solutions of the problem Eq. (\ref{eq:joint_optimization}) probably lie. 

Integrating Eq. (\ref{eq:p_gamma}) into Eq. (\ref{eq:engery_decomp}), we have
\vspace{-3pt}
\begin{equation}
\label{eq:energy_gamma}
E^{(\gamma)}(\u,\w,\mathbf{c})=(\gamma-1)E^{(p)}(\w,\mathbf{c}) + E_{\theta}(\u,\w,\mathbf{c})-\log{Z},
\end{equation}
where $E^{(\gamma)}(\u,\w,\mathbf{c})=-\log{(p_\gamma(\u,\w|\mathbf{c}))} +{const}$
is the reweighted energy-based model associated with $E_{\theta}(\u,\w,\mathbf{c})$ in Eq. \eqref{eq:joint_optimization}, relying on the hyperparameter $\gamma$. Then the optimization problem Eq. \eqref{eq:joint_optimization} can be transformed to 
\begin{gather}
\label{eq:reweight_joint_optimization}
\u^*, \w^* = \argmin_{\u, \w}\left[E^{(\gamma)}(\u,\w,\mathbf{c}) + \lambda\cdot \mathcal{J}(\u,\w)\right].
\end{gather}

\begin{wrapfigure}{r}{0.5\textwidth}
\vspace{-25pt}
\begin{center}
    \includegraphics[scale=0.22]{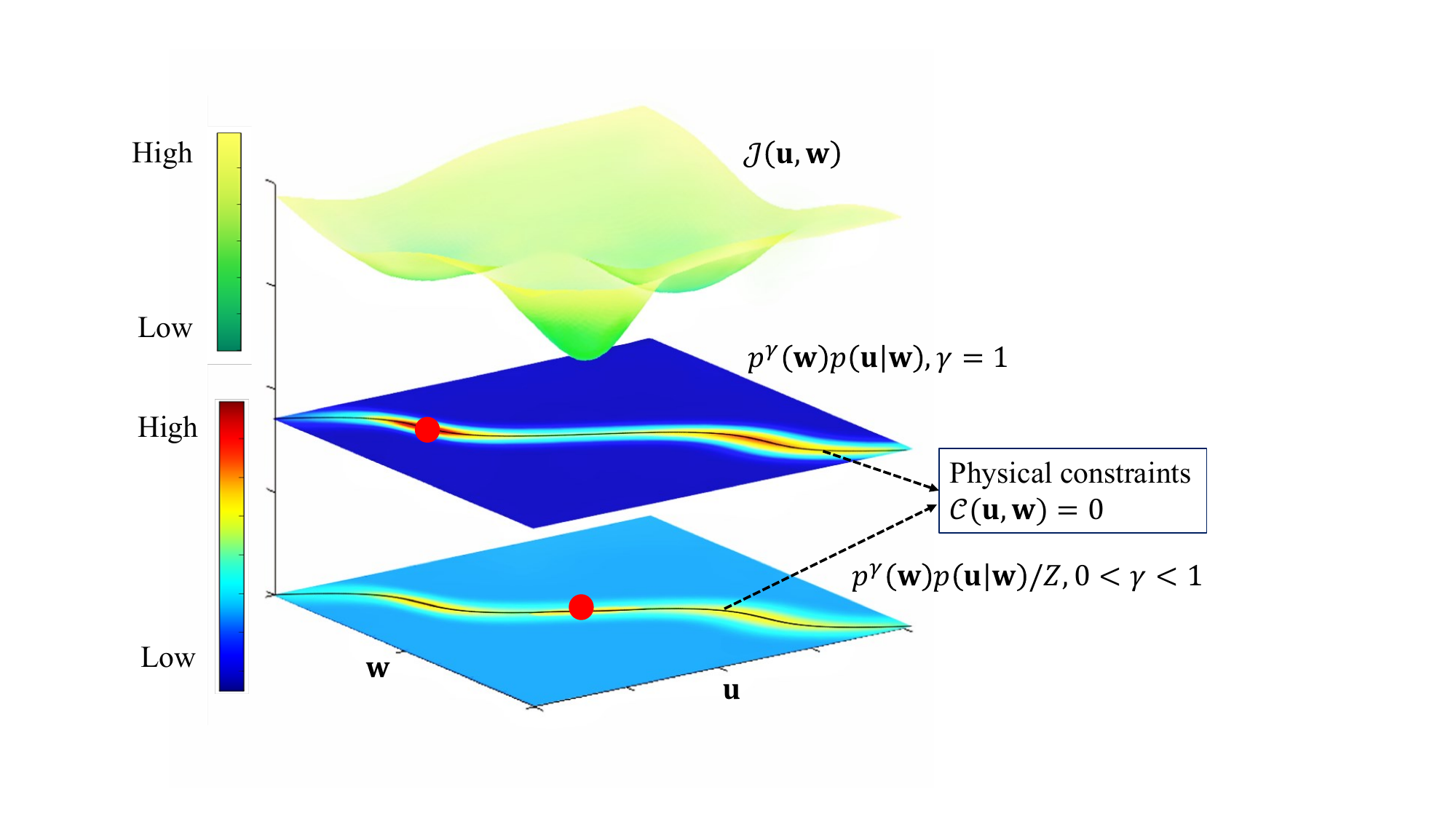}
\end{center}
\vspace{-8pt}
\caption{\textbf{Intuition of Prior Reweighting}. The top surface illustrates the landscape of $\mathcal{J}(\u,\w)$, where the high-dimensional variables $\u$ and $\w$ are represented using one dimension. The middle and lower planes depict probability heatmaps for the reweighted distribution $p^\gamma(\w)p(\u|\w)/Z$.
Adjusting $\gamma$ from $\gamma=1$ (middle plane) to $0<\gamma<1$ (lower plane),  a better minimal of $\mathcal{J}$ (red dot in the lower plane) gains the chance to be sampled. This contrasts with the suboptimal red point in the middle plane highly influenced by the prior $p(\w)$.}
\vspace{-20pt}
\label{fig:figure2}
\end{wrapfigure}

Optimization of this problem encourages sampling from the low likelihood region of $p(\w|\mathbf{c})$ while minimizing the control objective, which possesses the capability to generate control sequences that are more likely to be near-optimal than its degenerate version $\gamma=1$ in the original optimization problem Eq. \eqref{eq:joint_optimization}. The intuition of prior reweighting is illustrated in Figure \ref{fig:figure2}.

\textbf{Training}.
To learn the reweighted energy $E^{(\gamma)}(\u,\w,\mathbf{c})$, we parameterize its gradient as a summation of two parts by taking the gradient of both sides of Eq. \eqref{eq:energy_gamma}: 
\begin{align*}
\nabla E_{\phi,\theta}^{(\gamma)}(\u,\w,\mathbf{c})&=(\gamma-1)\nabla E^{(p)}_{\phi}(\w,\mathbf{c}) \\
&+ \nabla E_{\theta}(\u,\w,\mathbf{c}),
\end{align*}
where $E^{(p)}_{\phi}(\w,\mathbf{c})$ parameterizes the energy based model $E^{(p)}(\w,\mathbf{c})$ corresponding to $p(\w|\mathbf{c})$. Note that $\nabla\log{Z}$ vanishes here because it is a constant.
Notice that $\nabla E_{\theta}(\u,\w,\mathbf{c})$ has already been trained by Eq. \eqref{eq:training_obj}. 
$\nabla E^{(p)}_{\phi}(\w,\mathbf{c})$ can be trained similarly by the following loss function
\begin{equation}
\label{eq:training_p_w}
\mathcal{L}=\mathbb{E}_{k\sim U(1,K),(\w,\mathbf{c})\sim p(\w,\mathbf{c}),\bepsilon\sim \mathcal{N}(\mathbf{0},\mathbf{I})}[\|\mathbf{\bepsilon} - \bepsilon_\phi(\sqrt{\bar{\alpha}_t}\w +  \sqrt{1-\bar{\alpha}_k} \bepsilon,\mathbf{c},k)\|_2^2],\ \ \ \end{equation}
where $\bepsilon_\phi$ is the conditional denoising network that approximates $\nabla E^{(p)}_{\phi}(\w,\mathbf{c})$.

\textbf{Control optimization.} With both $\bepsilon_\theta$ and $\bepsilon_\phi$ trained, Eq. \eqref{eq:reweight_joint_optimization} can be optimized by running: 

\vspace{-15pt}
\begin{align}
\label{}
 \z_{k-1} &= \z_k - \eta(\bepsilon_\theta(\z_k,\mathbf{c}, k) +
 \lambda\nabla_{\z}\mathcal{J}(\hat{\z}_k)) + \xi_1, \quad \xi_1 \sim \mathcal{N} \bigl(0, \sigma^2_k \mathbf{I} \bigl) \\
\w_{k-1} &= \w_{k-1} - \eta(\gamma-1)\bepsilon_\phi(\w_k,\mathbf{c}, k) + \xi_2, \quad \xi_2 \sim \mathcal{N} \bigl(0, \sigma^2_k \mathbf{I} \bigl),
\end{align}
iteratively, where $\z_k=[\u_{k},\w_k]$.
The difference between this iteration scheme and Eq. \eqref{eq:1ddpm_inference} is that it uses an additional step to update $\w_k$ based on the predicted noise of $\bepsilon_\phi$. This guides $\z_k=[\u_{k},\w_k]$ to move towards the reweighted distribution $p_\gamma(\u,\w|\mathbf{c})$ while aligning with the direction to decrease the objective by its guidance in the iteration of $\z_k$. The complete algorithm is present in Algorithm \ref{alg:inference}. Detailed discussion and results about how to set the hyperparameter $\gamma$ are provided in Appendix \ref{app:hyperparams}. 

\textbf{Theoretical Analysis.} 
Consider a pair $[\u,\w]$ sampled using the prior reweighting technique with hyperparameter $\gamma$: $[\u,\w]\sim p_{\gamma}(\u,\w) = p(\u|\w) p^{\gamma}(\w)/C_{\gamma}$. 
Denote $\J^*$ as the global minimum of the control objective $\J$. Define $Q(\varepsilon)$ to be the "$\epsilon$-optimal" solution set of $[\u,\w]$ such that $\J(\u, \w) - \J^* \leq (\epsilon) $, whose complement set is $Q(\varepsilon)^c$, and denote $\mathbb{I}_{Q(\varepsilon)}(\u,\w)$ as its indicator function, i.e. $\mathbb{I}_{Q(\varepsilon)}(\u,\w)=1$ if $[\u,\w]\in Q(\varepsilon)$; otherwise 0.
Define $Y$ to be the random variable of ``whether use $\J$ as a guidance for sampling'', namely,
\begin{equation}
p(Y|\u,\w)=\begin{cases}
e^{-\J(\u,\w)} / Z, &Y=1 \\\\
1 - e^{-\J(\u,\w)} / Z, & Y=0.
\end{cases}
\end{equation}

Consider $E(\gamma)=\mathbb{E}_{(\u,\w)\sim p_{\gamma}(\u,\w)}[\mathbb{I}_{Q(\varepsilon)}(\u,\w)|Y=1]$, which indicates the expectation of getting an $\epsilon$-optimal solution by using the prior reweighting technique with $\gamma$ under the guidance of $\J$. Define
\begin{equation*}
F(\gamma) = \frac{\mathbb{E}_{\u,\w}[\mathbb{I}_{Q(\varepsilon)}(\u,\w) p(Y=1|\u,\w) \ln(p(\w))]}{\mathbb{E}_{\u,\w}[\mathbb{I}_{Q(\varepsilon)}(\u,\w) p(Y=1|\u,\w)]} - \frac{\mathbb{E}_{\u,\w}[\mathbb{I}_{Q(\varepsilon)^c}(\u,\w) p(Y=1|\u,\w) \ln(p(\w))]}{\mathbb{E}_{\u,\w}[\mathbb{I}_{Q(\varepsilon)^c}(\u,\w) p(Y=1|\u,\w)]},
\end{equation*}
then we have the following theorem (please refer to Appendix \ref{app:theory} for its proof):
\begin{theorem}
\label{prop:1}
Assume $E(\gamma)$ is a smooth function, then the following hold:
\begin{itemize}
    \item If $F(1) < 0$, there exists a $\gamma_{-} < 1$, s.t., $E(\gamma_{-})> E(1)$; 
    \item If $F(1) > 0$, there exists a $\gamma_{+} > 1$, s.t., $E(\gamma_{+})> E(1)$.
\end{itemize}
\end{theorem}
\textbf{Remark}: Here $F(\gamma)$ can be interpreted as some kind of difference between "entropies" in $Q(\varepsilon)^c$ and $Q(\varepsilon)$. When $F(1) < 0$, it means that $Q(\varepsilon)^c$ has higher "entropies",  implying that the training trajectories are far from optimal. As a result, we may need to flatten the distribution of training trajectories, which corresponds to using the prior reweighting technique with $\gamma<1$. Since this is the most common case in real scenarios, we usually set  $\gamma<1$.

\textbf{\proj-lite.}
The introduction of the prior reweighting technique in \proj involves training and evaluation of two models, thus bringing in additional computational cost. 
It is gratifying to note that we can balance the control performance and computational overhead of \proj by adjusting the parameter $\gamma$. When $\gamma=1$, the model $\bepsilon_{\phi}$ is not needed, and we denote this simplified version of \proj as \proj-lite.

\begin{algorithm}[t]
    \vspace{-2pt}
    \small
    \caption{Inference for \proj}
    \label{alg:inference}
    \begin{algorithmic}[1]
    \STATE \textbf{Require} Diffusion models $\bepsilon_\theta(\z_k,\mathbf{c},k)$ and $\bepsilon_\phi(\w_k,\mathbf{c},k)$, control objective $\J(\cdot)$, covariance matrix $\sigma_k^2 I$, control conditions $\c$, schedule $\Bar{\alpha}_k$, hyperparameters $\lambda, \gamma, K$ \\
    \STATE \textbf{Initialize} optimization variables  $\z_K=[\u_K, \w_K] \sim \mathcal{N}(\bm{0}, \mathbf{I})$  \\
    \FOR{$k = K, \ldots, 1$} 
        \STATE $\hat{\z}_k=(\z_k-\sqrt{1-\Bar{\alpha}_k}\bepsilon_\theta(\z_k, \mathbf{c}, k))/{\sqrt{\Bar{\alpha}_k}}$\\
        \STATE \cm{$\z_{k-1} = \z_k - \eta(\bepsilon_\theta(\z_k,\mathbf{c}, k) +
\lambda\nabla_{\z}\mathcal{J}(\hat{\z}_k)) + \xi_1, \xi_1 \sim \mathcal{N} \bigl(0, \sigma^2_k \mathbf{I} \bigl)$}  \small{\color{gray}// transition to next diffusion step}
        \STATE \cm{$\w_{k-1} = \w_{k-1} - \eta(\gamma-1)\bepsilon_\phi(\w_k,\mathbf{c}, k) + \xi_2, \xi_2 \sim \mathcal{N} \bigl(0, \sigma^2_k \mathbf{I} \bigl)$}  \small{\color{gray}// prior reweighting}
\hspace{0.8cm}  \\
    \ENDFOR \\
    \STATE \textbf{return} $\u^*,\w^*= \z_0$
    \end{algorithmic}
\end{algorithm}

%% file: text/04_experiments.tex
In this section, we aim to answer the following questions: (1) Can \proj present superiority over traditional, supervised learning, and reinforcement learning methods for physical systems control? (2) Does the proposed prior reweighting technique help achieve better control objectives? (3) Could answers to (1) and (2) be generalized to more challenging partial observation or partial control scenarios? To answer these questions, we conduct experiments on three vital and challenging problems: 1D Burgers' equation, 2D jellyfish movement control, and 2D smoke control problems.

The following state-of-the-art control methods are selected as baselines. For the 1D Burgers' equation, we use 
(1) the classical and widely used control algorithm Proportional-Integral-Derivative (PID) \cite{1580152} interacting with 
our trained surrogate model of the solver; 
(2) Supervised Learning method (SL) \cite{hwang2022solving}; RL methods including (3) Soft Actor-Critic (SAC) \cite{haarnoja2018soft} with offline and surrogate-solver versions; 
(4) Behaviour Cloning (BC) \cite{pomerleau1988alvinn}; and
(5) Behavior Proximal Policy Optimization (BPPO) \cite{zhuang2023behavior}. 
Specifically, the surrogate-solver version of SAC interacts with our trained surrogate model of the solver, while the offline version only uses given data. BC and BPPO are also in offline versions.
For 2D jellyfish movement control, baselines include SL, SAC (offline), SAC (surrogate-solver), BC, BPPO, and an additional classical multi-input multi-output algorithm Model Predictive Control (MPC) \cite{schwenzer2021review}. PID is inapplicable to this data-driven task \citep{147ea8517f15447798b6c73b263eb1e6}.
Detailed descriptions of baselines are provided in Appendix \ref{app:baseline1d} and Appendix \ref{app:baseline2d}. 

\subsection{1D Burgers' Equation Control}
\textbf{Experiment settings.}
The Burgers' equation is a governing law occurring in various physical systems. 
We consider the 1D Burgers’ equation with the Dirichlet boundary condition and external force $w(t,x)$, which is also studied in \cite{hwang2022solving,mowlavi2023optimal}.
\begin{eqnarray}
\begin{cases}
\label{eq:burgers}
    \frac{\partial u}{\partial t} = -u\cdot \frac{\partial u}{\partial x}+\nu\frac{\partial^2 u}{\partial x^2} +w(t,x)  &\text{in } [0,T] \times \Omega \\
    u(t,x) =0  \quad\quad\quad\quad\quad\quad\quad\quad &\text{on } [0,T] \times \partial\Omega    \\
    u(0,x) =u_0(x) \quad\quad\quad\quad\quad\quad &\text{in } \{t=0\} \times \Omega.
\end{cases}
\end{eqnarray}
Here $\nu$ is the viscosity parameter, and $u_0(\x)$ is the initial condition. Subject to Eq. \eqref{eq:burgers}, given a target state $u_d(x)$, the objective of control is to minimize the control error $\controlobj1d$ between $u_T$ and $u_d$, while constraining the energy cost $\mathcal{J}_\text{energy}$ of the control sequence $w(t,x)$:
\vspace{-5pt}
\begin{equation}
\label{eq:burgers_obj_J_actual}
\controlobj1d\coloneqq\int_{\Omega}|u(T,x)-u_d(x)|^2\mathrm{d}x, \ \ \mathcal{J}_\text{energy}\coloneqq\int_{[0,T]\times\Omega}|w(t,x)|^2\mathrm{d}t\mathrm{d}x.
\end{equation}
\vspace{-15pt}

To make the evaluation challenging, we select three different experiment settings that correspond to different real-life scenarios: partial observation, full control (PO-FC), full observation, partial control (FO-PC), and partial observation, partial control (PO-PC),
which are elaborated on in Appendix 
\ref{app:1d_experiment_four_settings} and illustrated in Appendix \ref{app:1d_vis}. 
These settings are challenging for classical control methods such as PID since they require capturing the long-range dependencies in the system dynamics. Note that the reported metrics in different settings are not directly comparable.
In this experiment, \proj uses the guidance conditioning (Section \ref{sec:conditional_guidance}) to optimize $\controlobj1d$ and explicit guidance to optimize the energy cost, which is elaborated in Appendix \ref{app:1d_model_descriptions}.

\begin{table}[!t]
\centering
\caption{\textbf{Best $\controlobj1d$ achieved in 1D Burgers's equation control.} Bold font denotes the best model, and underline denotes the second best model. 
}
{
\begin{tabular}{l|c|c|c}
    \hline
    \toprule
    & \multicolumn{1}{c|}{PO-FC} & \multicolumn{1}{c|}{FO-PC} & \multicolumn{1}{c}{PO-PC}\\
    \midrule
    PID (surrogate-solver) & - & 0.09115 & 0.09631 \\
    SL & 0.09752 & \underline{0.00078} & 0.02328 \\
    SAC (surrogate-solver) & 0.01577 & 0.03426 & 0.02149 \\
    SAC (offline) & 0.03201 & 0.04333 & 0.03328 \\
    BC & 0.02836 & 0.00856 & 0.00952 \\
    BPPO & 0.02771 & 0.00852 & \underline{0.00891} \\
    \midrule
    \textbf{\proj-lite (ours)} & \underline{0.01139} & \textbf{0.00037} & \textbf{0.00494} \\
    \textbf{\proj (ours)} & \textbf{0.01103} & \textbf{0.00037} & \textbf{0.00494} \\
    \bottomrule
\end{tabular}
\label{tab:1d_main_table}
\vspace{-5pt}
}
\end{table}

\begin{figure*}[!t]
    \centering    
    \includegraphics[width=1\textwidth]{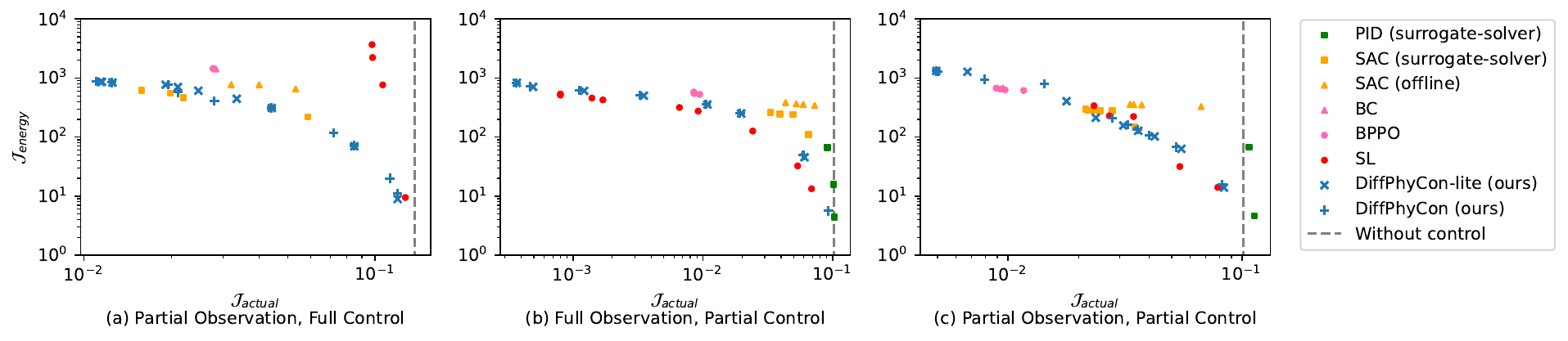}
    \vspace{-10pt}
    \caption{\textbf{Pareto frontier of $\mathcal{J}_\text{energy}$ vs. $\controlobj1d$ of different methods for 1D Burgers' equation.}}  
    \label{fig:pareto1d}  
    \vspace{-2pt}
\end{figure*}

\textbf{Results.} In Table \ref{tab:1d_main_table}, we report results of the control error $\controlobj1d$ of different methods. It can be observed that \proj delivers the best results compared to all baselines.
Specifically, \proj decreases $\controlobj1d$ of the best baseline by 30.1\%, 52.6\%, and 44.6\% in the PO-FC, FO-PC, and PO-PC settings respectively.
From Table \ref{tab:1d_main_table}, \proj and \proj-lite show little performance gap.
This is because, the prior distribution of $w$ that our \proj-lite learned is conditioned on both $u_0$ and $u_T$, which fully determines the optimal $w$.
Therefore, $p(w|u_0,u_T)$ is intrinsically the optimal distribution and thus \proj-lite can already deliver satisfactory performance.

To compare the ability of different methods to optimize $\controlobj1d$ with constrained energy cost $\mathcal{J}_\text{energy}$, we compare the Pareto frontiers of different methods in Figure \ref{fig:pareto1d}. We vary the hyperparameter $\lambda$ to control the tradeoff between $\controlobj1d$ and the energy cost since most baselines have this hyperparameter.
As can be observed in Figure \ref{fig:pareto1d}, the Pareto frontiers of \proj are consistently among the best, achieving the \emph{lowest} $\controlobj1d$ for most settings of the energy budget. 
Although SL performs well in full observation setting (b) where the system dynamics can be more easily predicted, it encounters difficulty in partial observation scenarios (a)(c). The results demonstrate \proj's ability to generate near-optimal control sequences compared to baselines. More visualization results are provided in Appendix \ref{app:1d_vis}. More results of evaluation are presented in Appendix \ref{app:1d_more}. For efficiency evaluation of training and test phases, please refer to Table \ref{tab:1d_efficiency} in Appendix \ref{app:efficiency}.

\subsection{2D Jellyfish Movement Control}\label{subsec:exp_2D}
\textbf{Experiment settings.}
This task is to control the movement of a flapping jellyfish with two wings in a 2D fluid field where fluid flows at a constant speed. The jellyfish is propelled by the fluid when its wings flap. Its moving speed and efficiency are determined by the mode of flapping. 
This task is an important source of inspiration for the design of underwater and aerial devices, and its challenges come from complex vortice behavior and fluid-solid coupling dynamics \cite{ristroph2014stable,kang2023propulsive}.
For this task, fluid dynamics follows the 2D incompressible Navier-Stokes Equation:
\begin{eqnarray}
\begin{cases}
\label{eq:ns_eq}
    \frac{\partial \v}{\partial t} + \v\cdot \nabla \v-\nu\nabla^2 \v +\nabla p=0\\
    \nabla\cdot \v=0 \\
    \v(0,\x) =\v_0(\x),
\end{cases}
\end{eqnarray}
where $\v$ represents the 2D velocity of the fluid, and $p$ represents the pressure, constituting the PDE state $\u=(\v,p)$. The initial velocity condition is $\v_0(\x)$ and the kinematic viscosity is $\nu$. 
We assume that each wing is rigid, so the jellyfish's boundary can be parameterized by the opening angle $\w_{t}$ of wings. Long-term movement of jellyfish usually presents a periodic flapping mode. Consequently, the control objective is to maximize its average moving speed $\bar v$ determined by the pressure of the fluid, under the energy cost constraint $R(\w)$  and the periodic constraint $d(\w_T, \w_0)$ of the movement:
\begin{equation}
\label{eq:jellyfish_obj}
\mathcal{J}=-\bar{v}+\zeta \cdot R(\w)+d(\w_T, \w_0), 
\end{equation}
subject to Eq. \eqref{eq:ns_eq} and the boundary condition that the velocity of fluid vanishes near the boundary.
The hyperparameter $\zeta$ is set to be 1000. 
We evaluate in two settings: full observation, where the full state $\u=(\v,p)$ is observed; and partial observation, where only pressure is observed. 
This task is very challenging due to the complicated dynamics of fluid-solid interactions and vortices behaviour \cite{ristroph2014stable,kang2023propulsive}, 
especially in the scenario of partial observation where the missing of $\v$ in Eq. \eqref{eq:ns_eq} restricts the information available to generate well-informed control signals. Details of the experiment are provided in Appendix \ref{app:2d_experiment}.

\textbf{Open-source Dataset Description.} We use the Lily-Pad simulator \cite{weymouth2015lily} to generate a dataset describing the movement of jellyfish under the control of its flapping behavior. Statistics about the dataset are listed in Table \ref{tab:2d_dataset_stat}. The feature of this dataset is that it contains a rich variation of vortices behavior and fluid-solid interaction dynamics, determined by violent changes in opening angles of wings and open-close phase ratio. It would serve as an important benchmark for studying complex physical system control problems. Details about the dataset are provided in Appendix \ref{app:2d_dataset_details}.

\begin{table*}[t]
\vskip -0.05in
\centering
\caption{\textbf{Jellyfish movement dataset outline.}}
\begin{tabular}{ccccc}
\hline
\toprule
training trajectories & test trajectories & resolution & \#fluid features & trajectory length \\
\midrule
30000                 & 1000              & 128$\times$128      & 3                & 40        \\
\bottomrule 
\end{tabular}
\label{tab:2d_dataset_stat}
\vskip -0.1in
\end{table*}

\begin{table*}[t]
\centering
\caption{\textbf{2D jellyfish movement control results.} Bold font denotes the best model, and underline denotes the second best model.}
\begin{tabular}{l|ccc|ccc}
\hline
\toprule
             & \multicolumn{3}{c|}{Full observation}             & \multicolumn{3}{c}{Partial observation}                          \\
            & $\bar{v}$ $\uparrow$ &  $R(\w)$ $\downarrow$ & $\mathcal{J}$ $\downarrow$ & $\bar{v}$ $\uparrow$ &  $R(\w)$ $\downarrow$ & $\mathcal{J}$ $\downarrow$  \\
            \midrule
MPC                    &      25.72     &    0.0112       &     109.17          &    -150.51       &    0.1791       &     329.59        \\
SL                     &     -76.94     &     0.1286      &     205.57          &    -102.98       &   0.1188        &     221.79        \\ 
SAC (surrogate-solver)    &     -166.96    &     0.0069      &      18.14          &     -153.09      &    0.0057       &     158.82        \\
SAC (offline)          &    -158.66     &   0.0069        &     165.58          &    -206.21       &    0.0058       &     211.96        \\
BC  &   30.48  &  0.0629  &   32.44  & 20.08  & 0.0556 &  35.48           \\ 
BPPO  & \underline{107.67} &  0.0867 &  \underline{-20.93}    &   \underline{54.83}   &  0.0518 &       \underline{-3.02}  \\ 

\midrule
    \textbf{\proj-lite (ours)}            &     95.04    &  0.0746         &     -20.47          &   2.92          &      0.0779     &        74.97      \\
\textbf{\proj (ours)}             &     \textbf{279.87}     &   0.2058        &    \textbf{-74.11}           &    \textbf{150.21}        &     0.1269      &      \textbf{-23.32}       \\ 
\bottomrule 
\end{tabular}
\label{tab:2D_results}
\vskip -0.1in
\end{table*}

\begin{figure*}[t]
\vspace{-0pt}
\begin{center}
    \includegraphics[width=\textwidth]{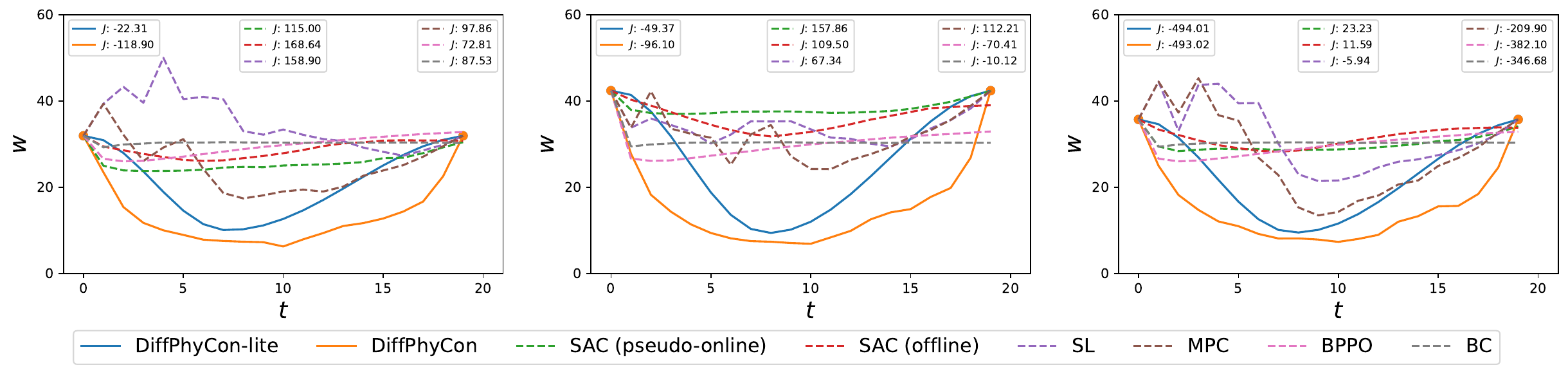}
\end{center}
\vspace{-10pt}
\caption{\textbf{
Comparison of generated control curves of three test jellyfish.} The resulting control objective $\J$ for each curve is presented.}
\label{fig:2d_control_curve}
\vspace{-17pt}
\end{figure*}

\begin{figure*}[ht]
\vspace{0pt}  
\begin{center}
    \includegraphics[width=\textwidth]{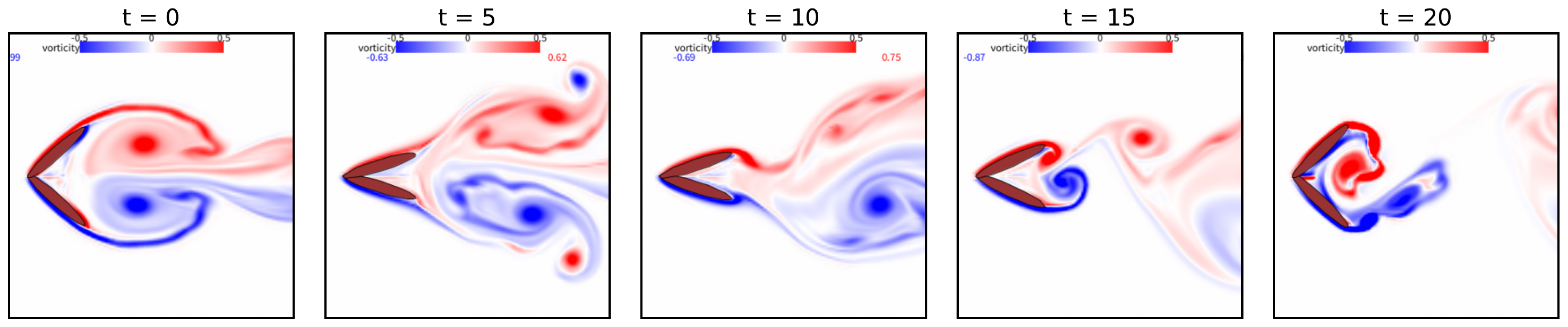}
\end{center}
\vspace{-8pt}
\caption{\textbf{Visualization of jellyfish movement and fluid field controlled by \proj as in the middle subfigure of Figure \ref{fig:2d_control_curve}.}}
\vspace{-10pt}
\label{fig:2d_fluid_field}
\end{figure*}

\begin{figure*}[ht]
\vspace{-3pt} 
\begin{center}
    \includegraphics[width=\textwidth]{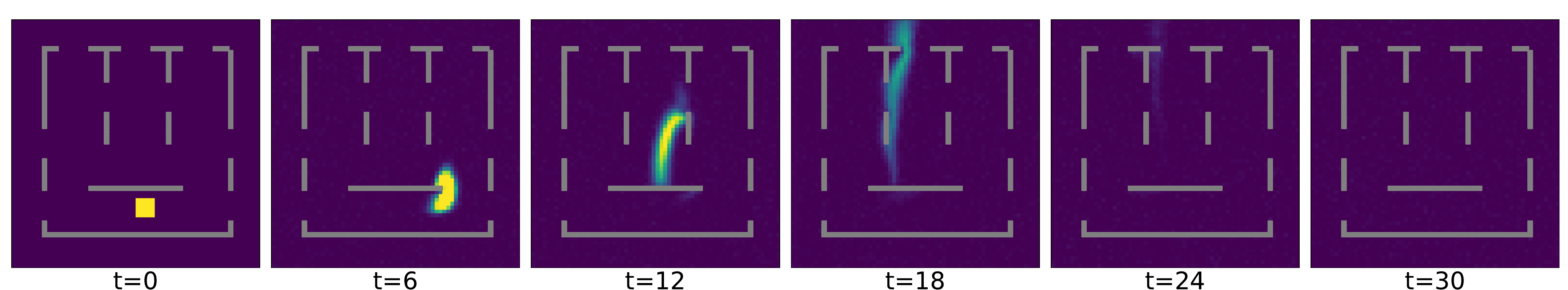}
\end{center}
\vspace{-8pt}
\caption{\textbf{Visualization of smoke density and fluid field dynamics controlled by \proj.}}
\vspace{-15pt}
\label{fig:2d_smoke_vis}
\end{figure*}

\textbf{Results.}
Evaluation results are presented in Table \ref{tab:2D_results}. It can be seen that our method outperforms the baselines by a large margin in optimizing the control objective $\J$. 
At a cost of slightly increasing the control cost $R(\w)$, control sequences generated by our method achieve a much faster average speed than baselines. 
In the full observation setting, the control objective achieved by \proj is -53.18 lower than the best baseline BPPO, while the average speed exhibits an increment of 172.2 over it. 
This demonstrates that diffusion models are effective for this challenging control task by performing global optimization of trajectory and control sequences in a generative approach. 
Comparison between \proj-lite and \proj reveals that by flattening the prior distribution of control sequences, another significant improvement is further achieved. 
Even in the more challenging partial observation setting, \proj still exhibits substantial advantages over existing methods. This reflects our method has a strong control capability under inadequate information. Configuration of the hyperparameter $\gamma$ in \proj and performance with respect to varying $\gamma$ is presented in Figure \ref{fig:figure_effect_gamma} in Appendix \ref{app:hyperparam_gamma}. Details about the hyperparameter $\lambda$ can be found in Table \ref{tab:effect_lambda_2d} in Appendix \ref{app:hyperparam_lambda}.

Figure \ref{fig:2d_control_curve} visualizes generated opening angle curves of different methods on three test jellyfish. Opening angle curves of \proj-lite show an obvious fast-close-slow-open shape, which is proven to produce high speed in jellyfish movement \cite{kang2023propulsive}. 
The reason is that the fast closing of wings leads to a thrust from fluid in the early state, resulting in a long-term high speed, and followed by a slow opening to reduce resistance. 
While this mode of movement appears rarely in the training dataset, \proj-lite could generate such control sequences for most test samples. 
This reflects that diffusion models under guidance are effective in optimizing the control objective. 
Furthermore, \proj makes this mode of movement more aggressive, with sharper change of the opening angle in the beginning and end stage within a period. 
This provides strong evidence that reweighting the prior distribution of control sequences allows for flexible sampling over a flattened distribution, thus control sequences with low prior but good objectives are more likely to be sampled. 
The movement and the resulting fluid field of the jellyfish corresponding to the middle subfigure of Figure \ref{fig:2d_control_curve} controlled by \proj is illustrated in Figure \ref{fig:2d_fluid_field} and 
more examples are provided in Figure \ref{fig:2d_vis_our} in Appendix \ref{app:2d_vis}. Conversely, opening angles obtained by baselines are inferior. 
The reason may be that they predict opening angles sequentially and hard to globally optimize the objective with three conflicting terms: average speed, $R(\w)$, and $d(\w_T, \w_0)$. We further study such myopic failure mode of SAC in Appendix \ref{app:myopic_failure}. 
More comparisons about the variation of the weight $\zeta$ are listed in Table \ref{tab:2D_full} and Table \ref{tab:2D_part} in Appendix \ref{app:2d_more}.
For efficiency evaluation of training and test phases, please refer to Table \ref{tab:2d_efficiency} in Appendix \ref{app:efficiency}. 
We also extend this experiment to a high-dimensional control signal setting, where the wings of the jellyfish are assumed to be soft. We find that our method is still competitive with baselines. Details about this evaluation are provided in Appendix \ref{app:finer_2d}.

\subsection{2D Smoke Indirect Control}\label{subsec:exp_smoke}
\textbf{Experiment settings.}
This task is to control smoke control in an incompressible fluid environment, following a similar (but more challenging) setup of \citep{holl2020learning}. Control forces were applied within a $64\times64$ grid flow field, excluding a semi-enclosed region, to minimize the smoke failing to pass through the top middle exit (seven exits in total). For illustration of our settings, please refer to Figure \ref{fig:smoke} in Appendix \ref{app:smoke_experiment}. This high-dimensional indirect control problem involves managing 2D forces at approximately 1,700 grid points every time step, resulting in about 100,000 control variables across 32 time steps, making it highly challenging.

\begin{wraptable}{tr}{0.5\textwidth}
\small
\vspace{-12pt}
\centering
\centering
\caption{\textbf{2D smoke movement control results.} Bold font denotes the best model, and underline denotes the second best model.} 
\vspace{2pt}
{
\begin{tabular}{l|c}
    \hline
    \toprule
    Method & \multicolumn{1}{c}{$\J$ $\downarrow$}  \\
    \midrule
    BC & 0.3085   \\
    BPPO & 0.3066  \\
    SAC (surrogate-solver) &	0.3212 \\
    SAC(offline) &	0.6503 \\\midrule
    \textbf{\proj -lite (ours)} & \underline{0.2324} \\
    \textbf{\proj (ours)} & \textbf{0.2254} \\

    \bottomrule
\end{tabular}
\label{tab:smoke_main_table}
\vspace{-10pt}
}
\normalsize
\end{wraptable}

\textbf{Results.}
Evaluation results are presented in Table \ref{tab:smoke_main_table}. Our method still has significant advantages over baselines in minimizing the control objective. Furthermore, the prior reweighting technique achieves extra improvement over \proj-lite. 
One test sample of smoke density and fluid field dynamics is illustrated in Figure \ref{fig:2d_smoke_vis}. More visualization results of our method are presented in Figure \ref{fig:smoke_vis_our} in Appendix \ref{app:smoke_experiment}. These results demonstrate that \proj is capable of controlling high dimensional physical systems even when control signals are indirectly applied to the system.

%% file: text/05_conclusion.tex
In this work, we have introduced \proj, a novel methodology for controlling complex physical systems. It generates control sequences and state trajectories by jointly optimizing the generative energy and control objective. We further introduced prior reweighting to enable the discovery of control sequences that diverge significantly from training. Through comprehensive experiments, we demonstrated our method's superior performance compared to classical, deep learning, and reinforcement learning baselines in challenging physical systems control tasks.
We discuss limitation and future work in Appendix \ref{app:future_work} and state social impact in \ref{app:social_impact}. 

%% file: text/06_acknowledgement.tex
We thank Yuchen Yang for insightful discussions on theoretical analysis. We thank the anonymous reviewers for providing valuable feedback on our manuscript.
We also gratefully acknowledge the support of Westlake University Research Center for Industries of the Future and
Westlake University Center for High-performance Computing.
The content is solely the responsibility of the authors and does not necessarily represent the official views of the funding entities.

%% file: text/07_appendix.tex
\section{Additional Related Work}
\label{app:related}
\subsection{Physical Systems Simulation}
\label{app:related_simulation}
Complex physical systems simulation forms the foundation of systems control. While classical numerical techniques for simulating physical systems are renowned for their accuracy, they are often associated with significant computational expenses \cite{morton2005numerical, lapidus1999numerical}.
Recently, neural network-based solvers show a significant advantage over classical solvers in accelerating simulations.
They could be roughly divided into three primary classes: data-driven methods \cite{DBLP:conf/iclr/LiKALBSA21,sanchez2020learning,pfaff2020learning,DBLP:conf/iclr/BrandstetterWW22,DBLP:conf/nips/WuML22,brandstetter2023clifford,lam2023learning}, Physics-Informed Neural Networks (PINNs) \cite{DBLP:journals/jcphy/RaissiPK19,DBLP:journals/corr/abs-2105-09506,wang2021learning}, and solver-in-the-loop methods \cite{um2020solver,vlachas2022multiscale}.
Most of them
use an iterative horizontal prediction framework. Instead, we treat the system trajectory as a whole variable and use diffusion models to learn an explicit simulator conditioned on control sequences. A notable work is by \cite{cachay2023dyffusion}, which introduces diffusion models for temporal forecasting. While both our work and \cite{cachay2023dyffusion}'s employ diffusion models, we tackle a different task of physical system control. Furthermore, we incorporate the control objective into the inference and introduce prior reweighting to tune the influence of the prior.

\subsection{Physical Systems Control}
\label{app:related_control}
For physical systems whose dynamics are described by PDEs, the adjoint methods \cite{lions1971optimal,mcnamara2004fluid,protas2008adjoint} have been the most widely used approach for system control in the last decades. It is accurate but computationally expensive. Deep learning-based methods have emerged as a powerful
tool for modeling physical systems' dynamics. Supervised learning (SL) \cite{holl2020learning,hwang2022solving} trains parameterized models to directly optimize control using backpropagation through time over the entire trajectory. For example, \cite{holl2020learning} proposes a hierarchical predictor-corrector scheme to control complex nonlinear physical systems over long time frames. A more recent work proposed by \cite{hwang2022solving} designs two stages which respectively learn the solution operator and search for optimal control. Different from these methods, we do not use the surrogate model, and learn both state trajectories and control sequences in an integrated way. 
Reinforcement learning (RL) \cite{farahmand2017deep, pan2018reinforcement, rabault2019artificial} treats control signals as actions and learns policies to make sequential decisions. 
Particularly in the field of fluid dynamics \cite{viquerat2022review}, reinforcement learning has been applied to a multitude of specific problems including drag reduction \cite{rabault2019artificial, elhawary2020deep}, conjugate heat transfer \cite{beintema2020controlling, hachem2021deep} and swimming \cite{novati2017synchronisation, verma2018efficient}. But they implicitly consider physics information and sequentially make decisions. In contrast, we generalize the entire trajectories, which results in a global optimization with consideration of physical information learned by models. Recently, PINNs are also incorporated in PDE control \cite{mowlavi2023optimal}, but they require an explicit form of PDE dynamics, while our method is data-driven and can deal with a broader range of complex physical system control problems without explicit PDE dynamics.

\subsection{Diffusion Models}
\label{app:related_diffusion}
Diffusion models \cite{ho2020denoising} have significantly advanced in applications such as image and text generation \cite{dhariwal2021diffusion,nichol2021glide}, inverse design \cite{wu2024compositional,vlastelica2023diffusion}, inverse problem \cite{holzschuh2023solving}, physical simulation \cite{cachay2023dyffusion,price2023gencast}, and decision-making \cite{janner2022planning,ajay2022conditional,he2024diffusion}. 
In particular, recent progress in robot control shows that diffusion models have significant advantages over existing reinforcement learning methods for action planning \cite{chi2023diffusion,ze20243d}.
Generating diverse yet consistent samples poses a challenge. For diversity, methods \cite{liu2022compositional,bao2022equivariant,zhao2022egsde,du2023reduce} that integrate score estimates from various models have been effective. For consistency, guidance diffusion techniques \cite{dhariwal2021diffusion,ho2022classifier} have been utilized to generate condition-specific samples. Our approach differs by flattening the joint distribution to achieve better control by slightly expanding beyond the prior distribution range.

\section{Theoretical Analysis of Prior Reweighting}\label{app:theory}
\begin{proof}[Proof of Theorem \ref{prop:1}]
\begin{align*}
E(\gamma)= &\int \mathbb{I}_{Q(\varepsilon)}(\u,\w) p_{\gamma}(u,w|Y=1) \mathrm{d}(\u,\w) \\
= &\int \mathbb{I}_{Q(\varepsilon)}(\u,\w) \frac{p(Y=1|\u,\w) p_{\gamma}(\u,\w)}{p(Y=1)} \mathrm{d}(\u,\w) \\
= &\frac{\mathbb{E}_{(\u,\w)}[\mathbb{I}_{Q(\varepsilon)}(\u,\w) p(Y=1|\u,\w)]}{\mathbb{E}_{(\u,\w)}[p(Y=1|\u,\w)]} \\
= &\frac{\mathbb{E}_{(\u,\w)\}[\mathbb{I}_{Q(\varepsilon)}(\u,\w) p(Y=1|\u,\w)]}}{\mathbb{E}_{\u,\w}[\mathbb{I}_{Q(\varepsilon)}(\u,\w) p(Y=1|\u,\w)] + \mathbb{E}_{\u,\w}[\mathbb{I}_{Q(\varepsilon)^c}(\u,\w) p(Y=1|\u,\w)]} \\
= &\frac{1}{1 + \frac{\mathbb{E}_{\u,\w}[\mathbb{I}_{Q(\varepsilon)^c}(\u,\w) p(Y=1|\u,\w)]}{\mathbb{E}_{\u,\w}[\mathbb{I}_{Q(\varepsilon)}(\u,\w) p(Y=1|\u,\w)]}}
\end{align*}

Define
\begin{align*}
G(\gamma) = &\frac{\mathbb{E}_{\u,\w}[\mathbb{I}_{Q(\varepsilon)}(\u,\w) p(Y=1|\u,\w)]}{\mathbb{E}_{\u,\w}[\mathbb{I}_{Q(\varepsilon)^c}(\u,\w) p(Y=1|\u,\w)]} \\\\
= &\frac{\mathbb{E}_w[\mathbb{E}_u[\mathbb{I}_{Q(\varepsilon)}(\u,\w) p(Y=1|\u,\w)|\w]]}{\mathbb{E}_w[\mathbb{E}_u[\mathbb{I}_{Q(\varepsilon)^c}(\u,\w) p(Y=1|\u,\w)|\w]]} \\\\
= &\frac{\int \mathbb{E}_u[\mathbb{I}_{Q(\varepsilon)}(\u,\w) p(Y=1|\u,\w)|\w] p^{\gamma}(\w) \mathrm{d}\w}{\int \mathbb{E}_u[\mathbb{I}_{Q(\varepsilon)^c}(\u,\w) p(Y=1|\u,\w)|\w] p^{\gamma}(\w) \mathrm{d}\w}
\end{align*}

Then

$E(\gamma) = \frac{1}{1 + \frac{1}{G(\gamma)}}$. Since $G(\gamma)>0$, $E(\gamma)$ and $G(\gamma)$ have the same monotonicity.

\begin{align*}
G'(\gamma) = &\frac{\int \mathbb{E}_{u}[\mathbb{I}_{Q(\varepsilon)}(\u,\w) p(Y=1|\u,\w)|\w] p^{\gamma}(\w) \ln(p(\w)) \mathrm{d}\w \mathbb{E}_{\u,\w}[\mathbb{I}_{Q(\varepsilon)^c}(\u,\w) p(Y=1|\u,\w)]}{(\mathbb{E}_{\u,\w}[\mathbb{I}_{Q(\varepsilon)^c}(\u,\w) p(Y=1|\u,\w)])^2 } \\\\
& -\frac{\mathbb{E}_{\u,\w}[\mathbb{I}_{Q(\varepsilon)}(\u,\w) p(Y=1|\u,\w)]  \int \mathbb{E}_{u}[\mathbb{I}_{Q(\varepsilon)^c}(\u,\w) p(Y=1|\u,\w)|\w] p^{\gamma}(\w) \ln(p(\w)) \mathrm{d}\w}{(\mathbb{E}_{\u,\w}[\mathbb{I}_{Q(\varepsilon)^c}(\u,\w) p(Y=1|\u,\w)])^2} \\\\
= & \frac{\mathbb{E}_{\u,\w}[\mathbb{I}_{Q(\varepsilon)}(\u,\w) p(Y=1|\u,\w) \ln(p(\w))] \mathbb{E}_{\u,\w}[\mathbb{I}_{Q(\varepsilon)^c}(\u,\w) p(Y=1|\u,\w)]}{(\mathbb{E}_{\u,\w}[\mathbb{I}_{Q(\varepsilon)^c}(\u,\w) p(Y=1|\u,\w)])^2} \\\\
& - \frac{\mathbb{E}_{\u,\w}[\mathbb{I}_{Q(\varepsilon)}(\u,\w) p(Y=1|\u,\w)] \mathbb{E}_{\u,\w}[\mathbb{I}_{Q(\varepsilon)^c}(\u,\w) p(Y=1|\u,\w) \ln(p(\w))]}{(\mathbb{E}_{\u,\w}[\mathbb{I}_{Q(\varepsilon)^c}(\u,\w) p(Y=1|\u,\w)])^2}
\end{align*}

By definition, $F(\gamma)$ is a positive multiple of $G'(\gamma)$, which implies our conclusion:
\begin{itemize}
 \item If $F(1) < 0$, then $G'(1) < 0$ and thus $E(\gamma)$ decreases around 1. Hence, there exists $\gamma_{-} < 1$, s.t., $E(\gamma_{-})> E(1)$, thus (i) holds; 
 \item otherwise, for similar reason, (ii) holds. 
\end{itemize}
\end{proof}

\textbf{Remark}: Here $F(\gamma)$ can be interpreted as some kind of difference between "entropies" in $Q(\varepsilon)^c$ and $Q(\varepsilon)$. When $F(1) < 0$, it means that $Q(\varepsilon)^c$ has higher "entropies",  implying that the training trajectories are far from optimal. As a result, we may need to flatten the distribution of training trajectories, which corresponds to using the prior reweighting technique with $\gamma<1$. Since this is the most common case, we usually set  $\gamma<1$.

\section{Visualization Results}\label{app:vis}
\subsection{1D Burgers' Equation Visualization}\label{app:1d_vis}

We present more visualization results of our method and baselines under three settings: FOPC, POFC, and POPC in Figure \ref{fig:1d_vis_fopc}, Figure \ref{fig:1d_vis_pofc} and Figure \ref{fig:1d_vis_popc}, respectively. Under each setting, we present the results of five randomly selected samples from the test dataset. The goal of control is to make the final state $\u_T$ ($T=10$) close to the target state.

\begin{figure}[hbp]
\centering
\hfill\hfill
\includegraphics[width=1\textwidth]{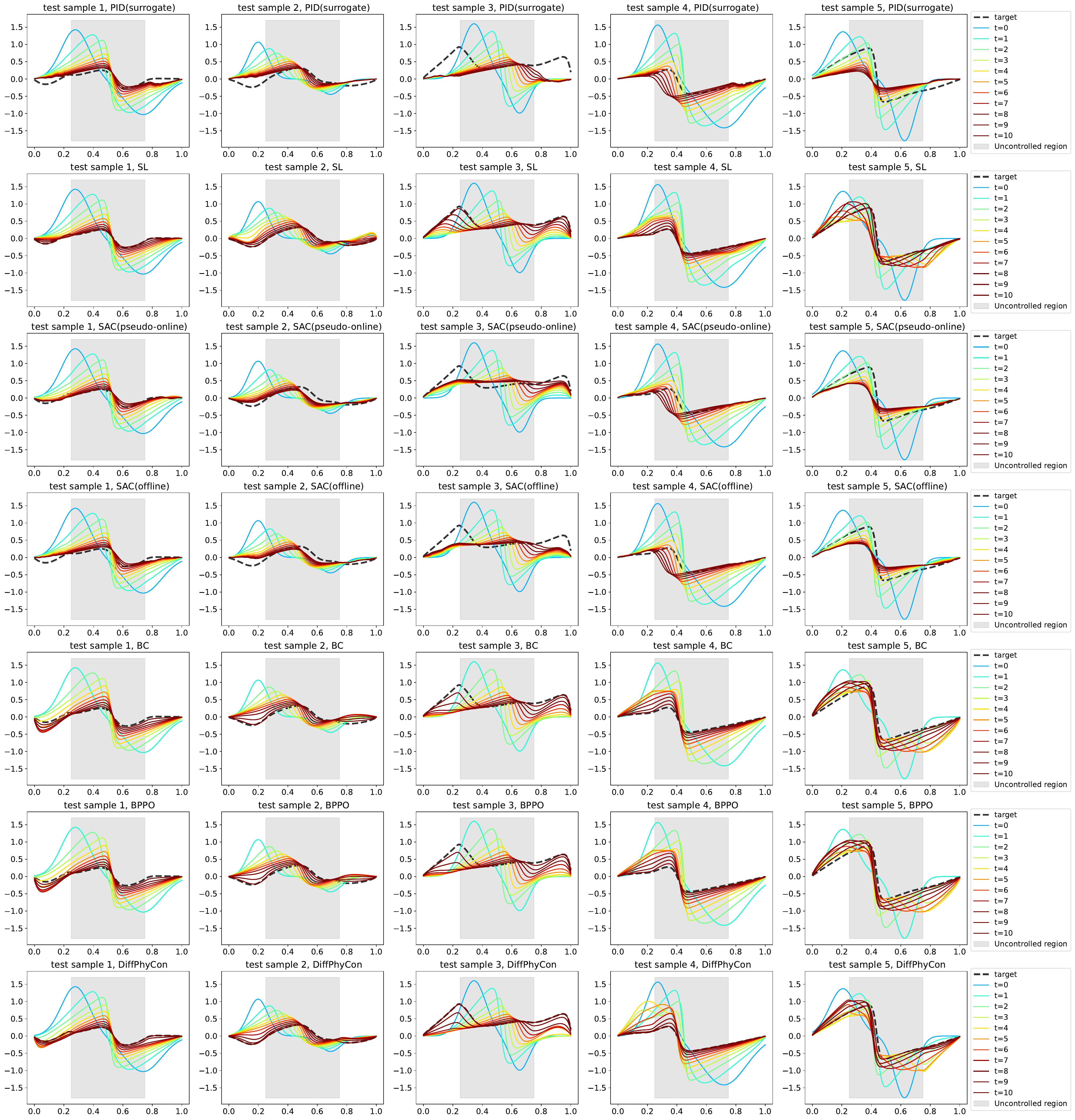}
\caption{\textbf{Visualizations results of 1D Burgers' equation control under the FO-PC (full observation, partial control) setting}. The curve for the system state $\u_t$ of each time step $t=0,\cdots,10$ under control is plotted for our method (\proj) and baselines. The $x$-axis is the spatial coordinate and the $y$-axis is the value of the system state.}
\label{fig:1d_vis_fopc}
\end{figure}

From these visualization results in Figure \ref{fig:1d_vis_fopc}, Figure \ref{fig:1d_vis_pofc}, and Figure \ref{fig:1d_vis_popc}, it can be observed that under the control of our \proj, the system state could converge smoothly to the target state given different initial states, and the final state $\u_T$ always coincides with the target state. Furthermore, this observation is consistent under all three settings: FOPC, POFC, and POPC, implying that our \proj is effective in addressing the partial observation and partial control challenges. In contrast, the baselines showed inferior results. Even the best baseline, BPPO, presents obvious mismatching with the target state on some samples.

\newpage

\begin{figure}[hbp]
\centering
\hfill\hfill
\includegraphics[width=1\textwidth]{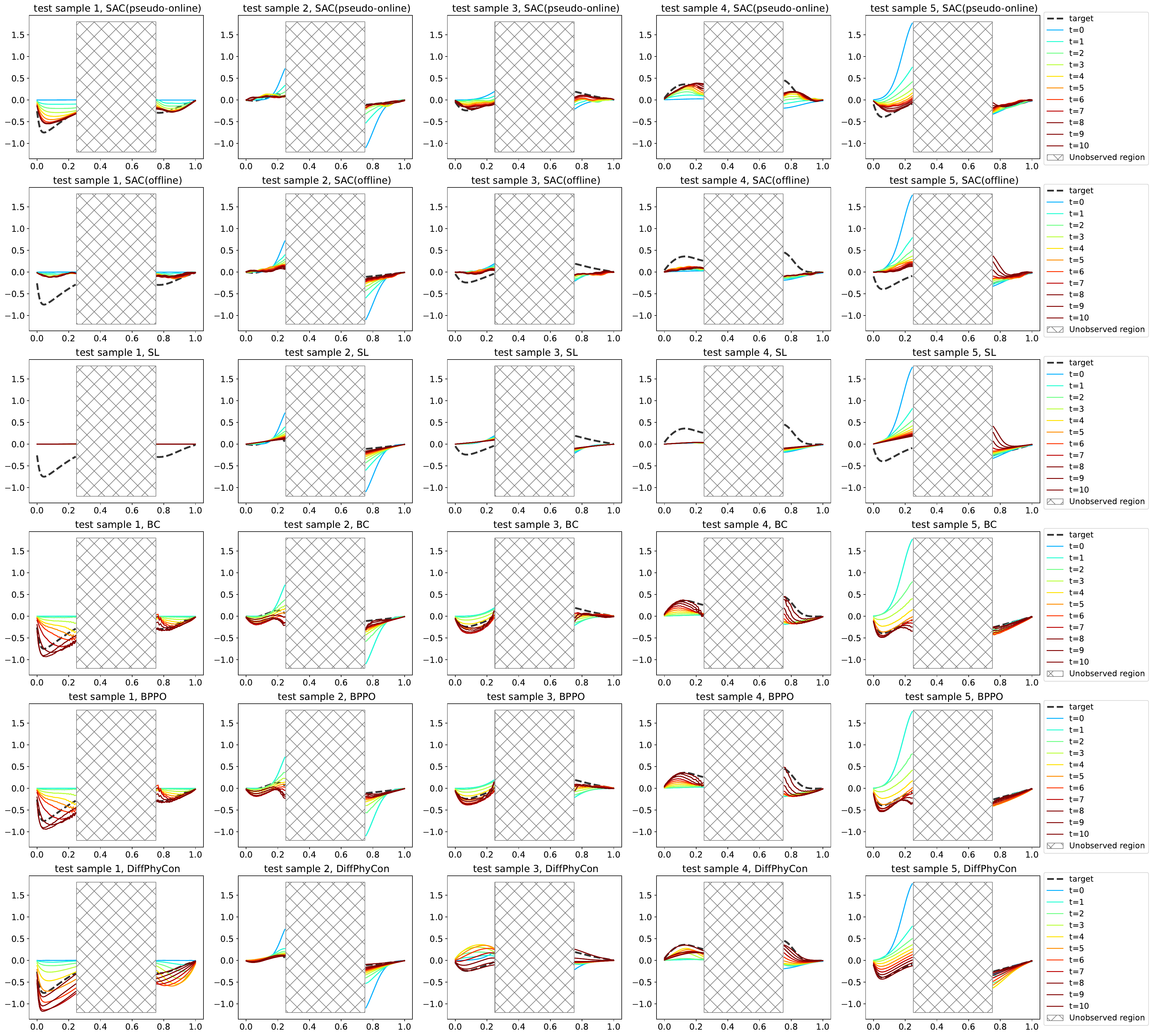}
\caption{\textbf{Visualizations results of 1D Burgers' equation control under the PO-FC (partial observation, full control) setting}. The curve for the system state $\u_t$ of each time step $t=0,\cdots,10$ under control is plotted for our method (\proj) and baselines. The $x$-axis is the spatial coordinate and the $y$-axis is the value of the system state.
}
\label{fig:1d_vis_pofc}
\vskip -0.2in
\end{figure}

\newpage

\begin{figure}[hbp]
\centering
\hfill\hfill
\includegraphics[width=1\textwidth]{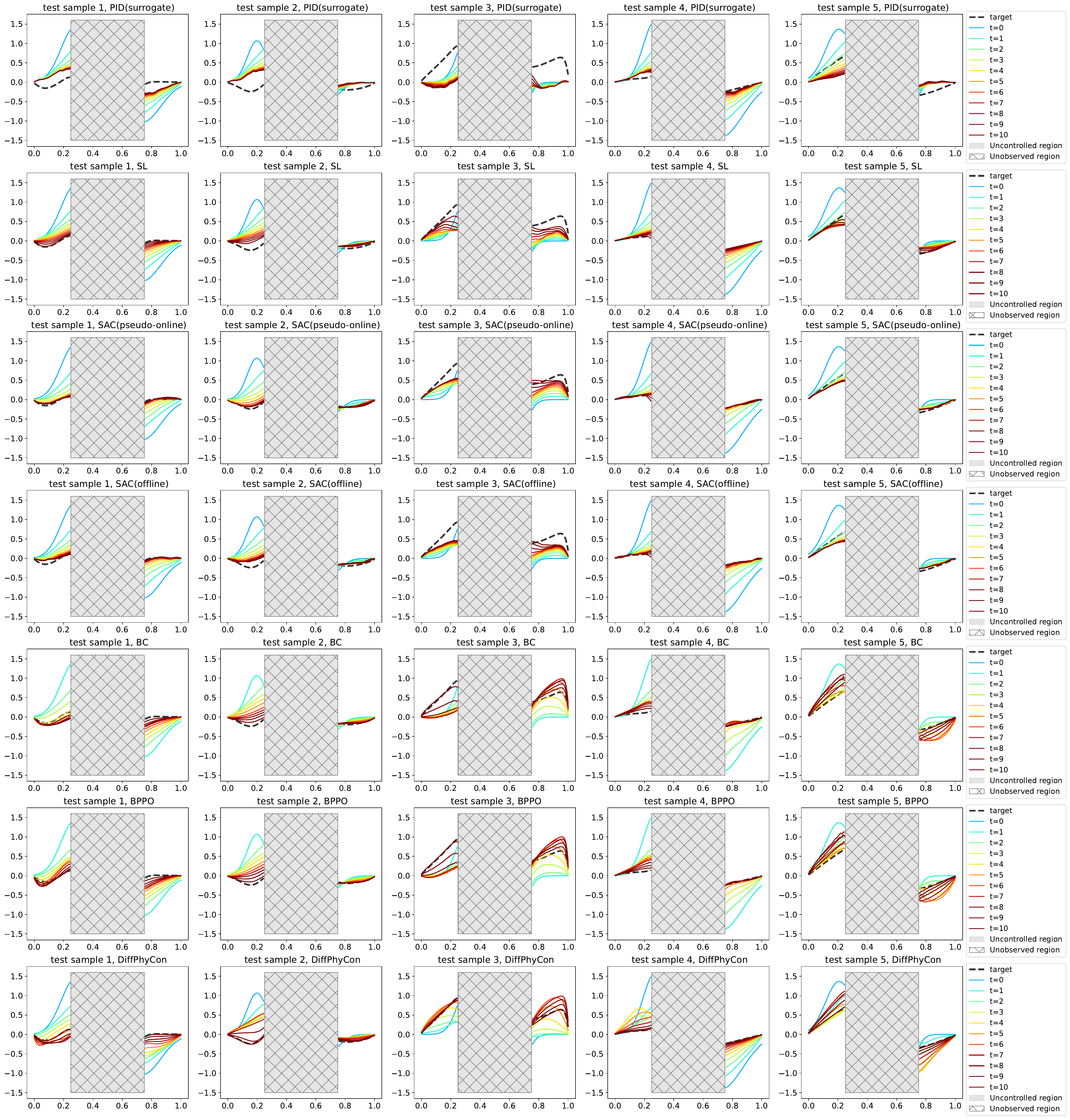}
\caption{
\textbf{Visualizations results of 1D Burgers' equation control under the PO-PC (partially observation, partially control) setting}. The curve for the system state $\u_t$ of each time step $t=0,\cdots,10$ under control is plotted for our method (\proj) and baselines. The $x$-axis is the spatial coordinate and the $y$-axis is the value of the system state.}
\label{fig:1d_vis_popc}
\end{figure}

\subsection{2D jellyfish control Visualization}\label{app:2d_vis}
We present more simulation results of our method in Figure \ref{fig:2d_vis_our}. Each line represents an example from the test dataset. We plot five snapshots of boundary and fluid field for each example.

\begin{figure}[htbp]
\centering
\hfill\hfill
\includegraphics[width=1\textwidth,height=0.75\textheight]{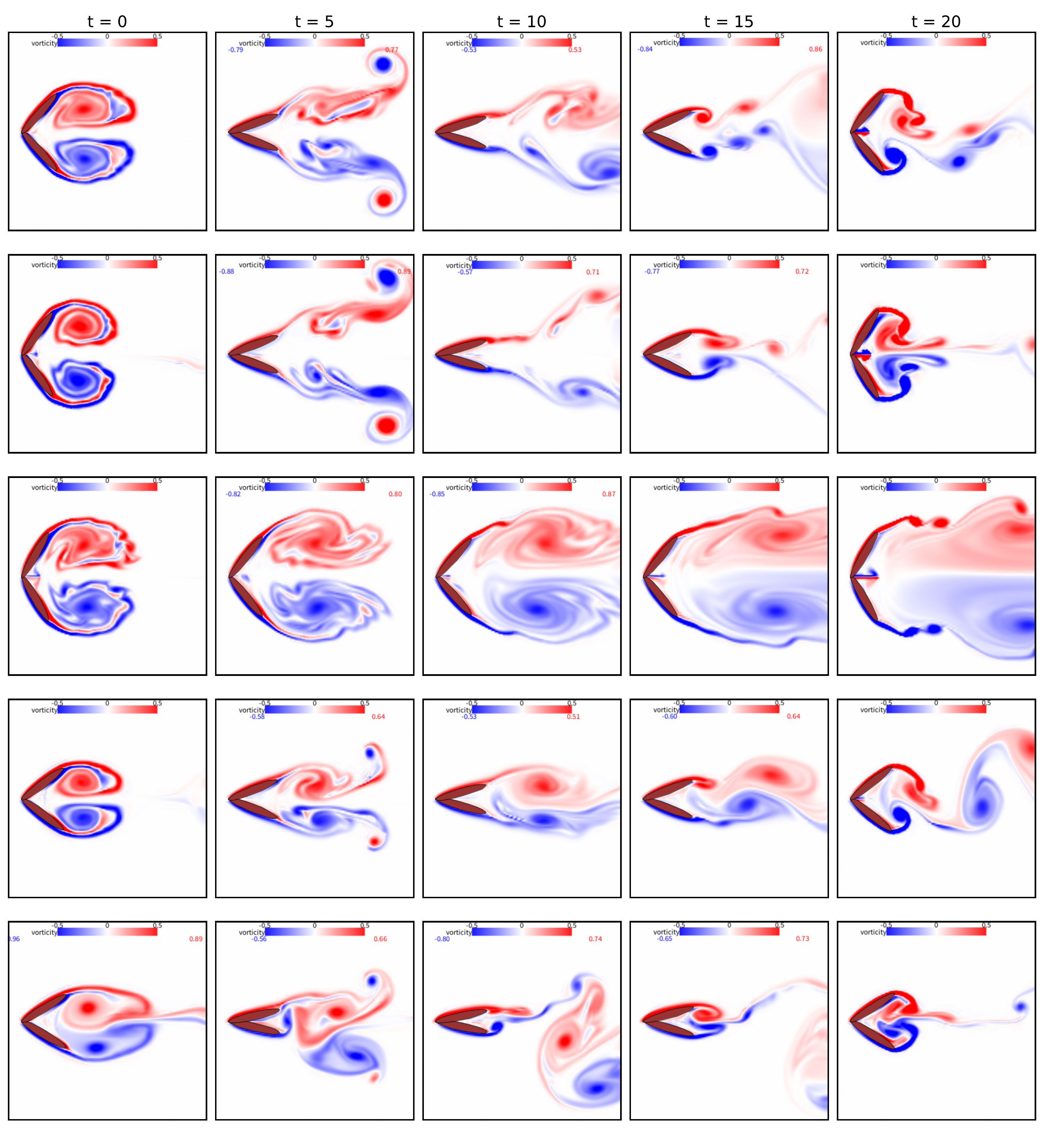}
\caption{\textbf{More examples of 2D jellyfish simulation  controlled by our method}.}
\label{fig:2d_vis_our}
\vskip -0.2in
\end{figure}

\subsection{2D Smoke Control Visualization}\label{app:smoke_vis}
We present fluid states and control signals generated by our method in Figure \ref{fig:smoke_vis_our}. 
Each line represents an example from the test dataset. We plot six snapshots for each example.

\begin{figure}[htbp]
\centering
\hfill\hfill
\includegraphics[width=1\textwidth]{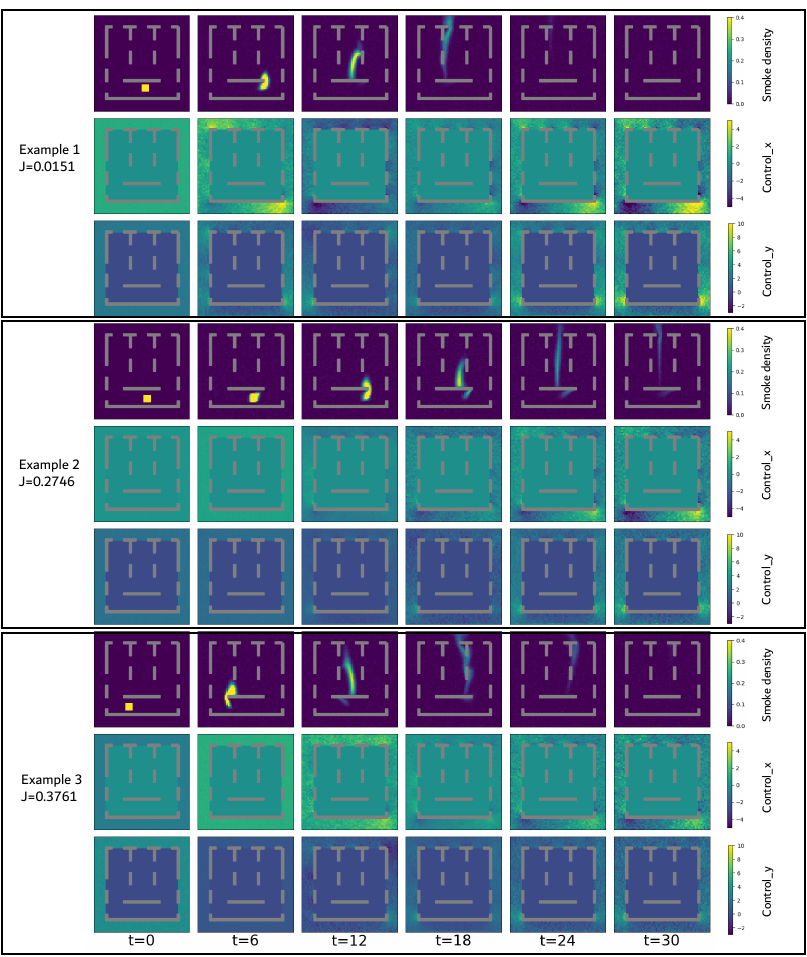}
\caption{\textbf{Examples of 2D smoke control results by our method}. We present three randomly selected test examples. For each example, we show
the generated smoke density map and control force fields in horizontal and vertical directions. Each
row depicts six frames of movement. The smoke density in the first row corresponds to that in Figure \ref{fig:2d_smoke_vis}.}
\label{fig:smoke_vis_our}
\vskip -0.2in
\end{figure}

\section{Additional Details for 1D Burgers' Equation Control}\label{app:1d_experiment}

\subsection{Data Generation}\label{app:1d_experiment_datagen}
We use the finite difference method (called solver or ground-truth solver in the following) to generate the training data for the 1D Burgers' equation. Specifically, the initial value $\u_0(x)$ and the control sequence $\w(t, x)$ are both randomly generated, and then the states $\u(t, x)$ are numerically computed using the solver. 

In the numerical simulation (using the ground-truth solver), a domain of $x\in [0,1],~t\in[0,1]$ is simulated. The space is discretized into $128$ grids and time into $10000$ steps. However, in the dataset, only 10 time stamps are stored. 
For the control sequence $\w$, its refreshing rate is $0.1^{-1}$, i.e., $\w(t,x),~t\in[0.1k, 0.1(k+1)],~k\in\{0,..,9\}$ does not change with $t$.
Therefore, the data size of each trajectory is $[11,128]$ for the state $\u$ and $[10,128]$ for the control $\w$.

In all settings, the initial value $\u(0,x)$ is a superposition of two Gaussian functions $\u(0,x)= \sum_{i=1}^{2} a_i e^{-\frac{(x-b_i)^2}{2\sigma_i^2}}$, where $a_i, b_i, \sigma_i$ are all randomly sampled from uniform distributions: $a_1\sim U(0, 2),~a_2\sim U(-2,0),~b_1\sim(0.2,0.4),~b_2\sim(0.6,0.8),~\sigma_1\sim U(0.05, 0.15),~\sigma_2\sim U(0.05, 0.15)$. Similarly, the control sequence $\w(x,t)$ is also a superposition of 8 Gaussian functions 
\begin{equation}
    \w(t,x)=\sum_{i=1}^{8} a_i e^{-\frac{(\x-b_{1,i})^2}{2\sigma_{1,i}^2}} e^{-\frac{(t-b_{2,i})^2}{2\sigma_{2,i}^2}}, 
\end{equation}
where each parameter is independently generated as follows: $b_{1,i}\sim U(0, 1),~b_{2,i}\sim U(0,1),~\sigma_{1,i}\sim U(0.05,0.2),~\sigma_{2,i}\sim U(0.05,0.2)$, while $a_1\sim U(-1.5,1.5)$ and for $i\ge 2$, $a_i\sim U(-1.5,1.5)$ or $0$ with equal probabilities. $\u(t,x),~(t\neq 0)$ is then numerically simulated (using the ground-truth solver) given $\u(0,x)$ and $\w(t,x)$ based on Eq. (\ref{eq:burgers}). The setting of the dataset generation is based on a previous work \cite{hwang2022solving}. We generated $90000$ trajectories for the training set and $50$ for the testing set. Each trajectory takes up $32$KB space and the size of the dataset sums up to $2$GB.

\subsection{Experimental Setting}
\label{app:1d_experiment_four_settings}
During inference, alongside the control sequence $\w(t,x)$, our diffusion model generates states $\mu(t,x)$, and some models produce surrogate states $\mu(t,x)$ when feeding the control $\w(t,x)$ into the corresponding surrogate model. However, our reported evaluation metric $\controlobj1d$ is always computed by feeding the control $\w(t,x)$ into the ground truth numerical solver to get $\u_{\text{g.t.}}(t,x)$ and computed following Eq. (\ref{eq:burgers_obj_J_actual}).
Followings are three different settings of our experiments.

\subsubsection{Partial Observation, Full Control}
In realistic scenarios, the system is often unable to be observed completely. Generally speaking, it is impractical to place sensors \textit{everywhere} in a system, so the ability of the model to learn from incomplete data is imperative. To evaluate this, we hide some parts of $\u$ in this setting and measure the $\controlobj1d$ of model control. 

Specifically, $\u(t, x),~x\in[\frac{1}{4},\frac{3}{4}]$ is set to zero in the dataset during training and $\u_0(x),~x\in[\frac{1}{4},\frac{3}{4}]$ is also set to zero during testing. 
In this partial observation setting $\Omega=[1,\frac{1}{4}]\cup[\frac{3}{4},1]$.
Since no information in the central $\frac{1}{2}$ space is ever known, the model does not know what will influence the control outcome of the unobserved states. Therefore, controlling the unobserved states is not a reasonable task and they are excluded from the evaluation metric.

This setting is particularly challenging not only because of the uncertainty introduced by the unobserved states but also the generation of the control in the central locations that implicitly affect the controlled $\u$ at $x\in\Omega$.

\subsubsection{Full Observation, Partial Control}
This is another setting of practical relevance, where only a fraction of the system can be controlled. The control sequence is enforced to be zero in the central locations of $x\in[\frac{1}{4},\frac{3}{4}]$. $\Omega$ is still $[0,1]$, and $\mathcal{J}$ is evaluated on all of the observed states, though.

Some modifications to the dataset should be mentioned. The generation of the data involves first generating $\w$ as before, followed by setting the central $\frac{1}{2}$ of $\w$ to zero. To compensate for the decreased control intensity so that the magnitude of $\u$ can be roughly comparable to the full control setting, we double the magnitude of $\w$. During the evaluation, the output control sequence is also post-processed to be zero in $x\in [\frac{1}{4},\frac{3}{4}]$.

It is worth noting that in this setting, even when the control energy is not limited at all, it is still challenging to find a perfect control since the model has to learn how to indirectly impose control on the central locations.

\subsubsection{Partial Observation, Partial Control}
The final setting is the combination of the previous two settings. Only $\Omega=[0,\frac{1}{4}]\cup[\frac{3}{4},1]$ is observed, controlled and evaluated. 

It is worth noting that some models require accessing the current state to produce output. If the model interacts with the ground truth solver instead of a surrogate model, then the result would be unfairly good since the information of the unobserved states is leaked through the interaction.

\subsection{Model}\label{app:1d_model_descriptions}
Since the training of models $\bepsilon_{\phi} \approx \nabla_\w \log p(\w)$ and $\bepsilon_{\theta} \approx \nabla_{\u, \w} \log p(\u,\w)$ are essentially the same and the latter model is exactly \proj-lite, we will introduce \proj-lite first.

\subsubsection{\proj-lite}\label{app:1dproj_model}
In general, \proj-lite follows the formulation of \cite{ho2020denoising} which is also described in the main text. The data of $\u$ and $\w$ is fed in as images of size $(N_t, N_x)$ where $N_t$ is the number of time steps (11 and 10 respectively) and $N_x$ is the spatial grids (128). Since the two $N_t$s for $\u$ and $\w$ are inconsistent, we zero-pad them into the size of $16$. Then, $\u$ and $\w$ are stacked as two channels and fed into the 2D DDPM model. 

A 2D UNet $\bepsilon_{\theta}$ is used to learn to predict $\bepsilon$. It is structured into three main components: the downsampling encoder, the central module, and the upsampling decoder.
The downsampling encoder is made up of four layers, each layer consisting of two ResNet blocks, one linear Attention block, and one downsampling convolution block. The central module also consists of two ResNet blocks and one linear Attention block. Each upsampling layer is the same as the downsampling layer except the downsampling block is replaced by the upsampling convolution block.

In our experiments, we found that the control result is best when learning the conditional probability distribution of $p(\w_{[0,T-1]},\u_{[1,T-1]}\mid \u_{0},\u_{T})$.
In summary, $\bepsilon_{\theta}$ takes in the current trajectory $\u$, control $\w$, step $k$, $\u_0$ and $\u_T$ as input, and predicts the noise of $\u$ and $\w$. Note that it is not trained to predict $\u_0$ and $\u_T$ which are used as a condition, but there are still model outputs at the corresponding locations for the data shape consistency across different design choices of \proj-lite.
The hyperparameters in different settings are listed in Table \ref{tab:diffusion_1d_single_model_hyperparameters}.

\subsubsection{\proj}

In terms of implementation, \proj is simply adding $\bepsilon_{\phi}(\w)$ to $\bepsilon_{\theta}(\u,\w)$ during inference as shown in Section \ref{sec:2ddpm}, where $\bepsilon_{\theta}$ is the output of the denoising network in \proj-lite while $\bepsilon_{\phi}$ is a new denoising network that is trained to generate $\w$ following the dataset distribution. Therefore, we only describe the model of $\bepsilon_{\phi}$ here.

$\bepsilon_{\phi}$ takes input of $\w, k$ as in the standard DDPM and $\u_0, \u_T$ as guidance conditioning. 
The output of $\bepsilon_{\phi}(\w)$ is of the same shape as $\w$, so it can be treated as a network learning to sample from $p(\w)\coloneqq \int p(\u,\w)\mathrm{d}\u$
The output of $\bepsilon_{\phi}(\w)$ at the locations of $\u$ is thus filled with zeros. The model hyperparameters are also listed in Table \ref{tab:diffusion_1d_single_model_hyperparameters}.

\begin{table}[ht]\small
  \begin{center}
    \caption{\textbf{Hyperparameters of the UNet architecture and training for the results of 1D Burgers' equation in Table \ref{tab:1d_main_table}}.}
     \label{tab:diffusion_1d_single_model_hyperparameters}
    \resizebox{1.0\linewidth}{!}{
    \begin{tabular}{l|l|l|l} 
    
    \hline
      \multirow{2}{*}{\text {Hyperparameter name}} &
      Full observation & Partial observation & Partial observation\\ &
      Partial Control & Full Control & Partial Control
      \\
      \hline
      \multicolumn{4}{c}{UNet $\bepsilon_{\phi}(\w)$}\\
      \hline
      Initial dimension & 32 & 32 & 32\\
      Downsampling/Upsampling layers &4&4&4\\
      Convolution kernel size &3&3&3\\
      Dimension multiplier &$[1,2,4,8]$&$[1,2,4,8]$&$[1,2,4,8]$\\
      Resnet block groups &8&8&8\\
      Attention hidden dimension &32&32&32\\
      Attention heads &4&4&4\\
      \hline
      \multicolumn{4}{c}{UNet $\bepsilon_{\theta}(\u,\w)$}\\
      \hline
      Initial dimension & 128 & 128 & 64\\
      Downsampling/Upsampling layers &4&4&4\\
      Convolution kernel size &3&3&3\\
      Dimension multiplier &$[1,2,4]$&$[1,2,4,8]$&$[1,2,4,8]$\\
      Resnet block groups &8&8&8\\
      Attention hidden dimension &32&32&32\\
      Attention heads &4&4&4\\
      \hline
      \multicolumn{4}{c}{Training}\\
      \hline
      Training batch size & 16 & 16& 16 \\
      Optimizer &Adam&Adam&Adam\\
      Learning rate &1e-4&1e-4&1e-4\\
      Training steps & 190000 & 170000 & 190000 \\
      Learning rate scheduler & cosine annealing & cosine annealing & cosine annealing\\
      \hline
      \multicolumn{4}{c}{Inference}\\
      \hline
      Sampling iterations &1000&1000&1000 \\ 
      Intensity of energy guidance $\J=\int \|\w\|^2\mathrm{d}x\mathrm{d}t$ & 0 & 0 & 0 \\
      Scheduler of energy guidance $\J=\int \|\w\|^2\mathrm{d}x\mathrm{d}t$ & cosine & cosine & cosine \\
      \hline
   \end{tabular}
   }
  \end{center}
\end{table}

\subsection{Training and Evaluation}

\paragraph{Training}
During training, the $\u_0$ and $\u_d$ without noise are fed into the model and the model outputs at the corresponding locations are excluded from the loss.
In the partial observation settings, the unobserved data is invisible to the model during both training and testing as introduced in Appendix \ref{app:1d_experiment_four_settings}. We simply pad zero in the corresponding locations of the model input and also exclude these locations in the training loss. Therefore, the model only learns the correlation between the observed states and control sequences.
In the partial control setting, we train our \proj on the dataset with control being zero in $x\in[\frac{1}{4},\frac{3}{4}]$. In this way, the model naturally learns to output zero at the ``non-controllable'' locations.

We use the MSE loss to train the denoising UNets and other training hyperparameters are listed in Table \ref{tab:diffusion_1d_single_model_hyperparameters}.

\paragraph{Inference}
During inference, $\u_{0}$ and $\u_{T}$ are set to the target $\u_0$ and $\u_d$ so that the DDPM generates samples satisfying the physical constraint that is also conditioned on the target ($\u_d$) or the constraint ($\u_0$).
In the partial observation setting, the $\u_0$ and $\u_T$ drawn from the testing set are all filled zero at the unobserved locations $x\in[\frac{1}{4},\frac{3}{4}]$, which is the same as the data used to train the UNets. In the partial control scenarios, 

During inference, we replace the denoising network's output $\bepsilon_{\theta}(\u,\w)$ with $\bepsilon_{\theta}(\u,\w) + (\gamma - 1)\bepsilon_{\phi}(\w)$. It is worth noting that $\bepsilon_{\theta}(\u,\w)$ denoises $u$ and $\w$ simultaneously while $\bepsilon_{\phi}(\w)$ only denoises $\w$. In our experiments, we found that adding a schedule to the output of the $\w$ network is beneficial. The results in Table \ref{tab:1d_main_table} are generated with a reverse Sigmoid schedule following
\begin{equation}\label{eq:app_1dburgers_twoddpm_noise_combined}
    \bepsilon_k = \bepsilon_{\theta}(\u_k,\w_k,k,\u_0, \u_T) + \frac{}{} (\gamma - 1) \beta_{K-k} \bepsilon_{\phi}(\w_k,k,\u_0,\u_T),
\end{equation}
where $\beta$ is defined as the noise schedule in \cite{ho2020denoising}.
The inference of \proj-lite is simply setting $\gamma=1$, which neglects the effect of the model $\bepsilon_\phi(\w)$.

When trying to regulate the control energy, however, it is not as natural to learn a conditional model. Therefore, we use the external guidance $\J=\int \w(x,t)\mathrm{d}x\mathrm{d}t$ to produce cost-limited control sequences that are shown in Figure \ref{fig:pareto1d}. The gradient of the external guidance is computed and added to $\bepsilon_k$ in Eq. \eqref{eq:app_1dburgers_twoddpm_noise_combined}. Note that we use the predicted clean sample $\hat \w_k$ and $\hat \u_k$ at the $k$-th step to compute $\nabla \J$ since they suffer less from being noisy and leading to deviated guidance. 
$\u_k$ and $\w_k$ are computed following
$
    \x_0 \approx \hat \x_k = \frac{1}{\sqrt{\bar \alpha_k}}\x_k - \frac{\sqrt{1 - \bar \alpha_k}}{\sqrt{\bar\alpha_k}}\bepsilon_k
$
where $\x$ represents $\u$ or $\w$ as
in \cite{ho2020denoising}. In our experiments, we use a cosine scheduling of the external guidance, and thus the final predicted noise would be $\bepsilon_k + \lambda \alpha_k \nabla \J$ where $\beta_k$ is the noise schedule in \cite{ho2020denoising} with the cosine schedule.

\paragraph{Evaluation}
After generating the trajectory $\u$ and control sequence $\w$, we feed the control sequence $\w$ into the ground-truth solver and simulate the final state $\u$ given the generated $\w$ and the initial condition $\u_0$ directly drawn from the testing dataset. The solver is the same as the one used in data generation in Appendix \ref{app:1d_experiment_datagen}. Finally, we compute $\controlobj1d$ following Eq. \eqref{eq:burgers_obj_J_actual}. In the partial observation setting, the MSE is computed only on the observed region, and in the control setting, the generated control will first be set to zero in the uncontrolled region before being fed into the solver.

\section{Jellyfish Movement Dataset}\label{app:2d_dataset_details}
We use the Lily-Pad simulator \cite{weymouth2015lily} to generate the Jellyfish Movement dataset, which serves as a benchmark for physical system control research and also the dataset for our 2D evaluation task. Lily-Pad adopts the Immersed Boundary Method (IBM) \cite{mittal2005immersed} to simulate fluid-solid dynamics.
The resolution of the 2D flow field is set to be $128\times128$. The flow field is assumed to be boundless in Lily-Pad.
The head of the jellyfish is fixed at $(25.6,64)$. Its two wings are represented by two identical ellipses, where the ratio between the shorter axis and the longer axis is 0.15. At each moment, the two wings are symmetric about the central horizontal line $y=64$. For each wing, we sample $M=20$ points along the wing to represent the boundary of the wing. The opening angle of the wings is defined as the angle between the longer axis of the upper wing and the horizontal line. It acts as the control sequence $\w$ in a 2D jellyfish control experiment.

Each trajectory starts from the largest opening angle and follows a cosine curve periodically with period $T^{'} =200$. Trajectories differ in initial angle, angle amplitude, and phase ratio $\tau$ (the ratio between the closing duration and a whole pitching duration). For each trajectory, the initial angle $\w_0$ is generated as follows: first, sample a random angle, called mean angle  $\w^{(m)}\in[20^\circ, 40^\circ]$, then sample a random angle amplitude $\w^{(a)}\in[10^\circ, \min(\w^{(m)}, 60^\circ-\w^{(m)})]$. The initial $\w_0$ is set as $\w_0=\w^{(m)} + \w^{(a)}$. The phase ratio $\tau$ is randomly sampled from $[0.2,0.8]$. The opening angle $\w_t$ of step $t$ decreases from $\w^{(m)}+\w^{(a)}$ to $\w^{(m)}-\w^{(a)}$ as $t$ grows from 0 to $\tau T^{'}$; then $\w_t$ increases from $\w^{(m)}-\w^{(a)}$ to $\w^{(m)}+\w^{(a)}$ as $t$ grows from $\tau T^{'}$ to $T^{'}$. Afterwards, $\w_t$ varies periodically for $t > T^{'}$. The range of $\w_t$ is $[\w^{(m)}-\w^{(a)}, \w^{(m)}+\w^{(a)}] \subset [10^\circ, 60^\circ]$.
For each trajectory, we simulate for $600$ simulation steps, i.e., 3 periods. To save space, we only save the piece of trajectory from $T^{'}=200$ to $3T^{'}=600$ steps with step size $10$ because the simulation from $t=0$ to $T^{'}=200$ is for initialization of the flow field. Then each trajectory is saved as a $\tilde{T}=(600-200)/10=40$ steps long sequence. An example of the simulated fluid field and the corresponding curve of opening angles are shown in Figure \ref{fig:lilypad}. 

Besides the positions of the boundary points of wings and the opening angles $\w$, we also use another kind of image-like representation of the boundaries of wings as this representation contains spatial information that can be more effectively learned along with physical states (fluid field) by convolution neural networks. For each trajectory, this image-like boundary representation is compatible with physical states in shape. At each time step, boundaries of two wings are merged and then represented as a tensor of shape [3, 64, 64], where it has three features for each grid cell: a binary mask indicating whether the cell is inside a boundary (denoted by 1) or in the fluid (denoted by 0), and a relative position $(\Delta \x, \Delta y)$ between the
cell center to the closest point on the boundary. 
For each trajectory, we save system states, opening angles, boundary points, boundary masks and offsets, and force data. They are specified as:

\begin{itemize}
    \item system states $\u$: shape $[\tilde{T}, 3, 64, 64]$. For each step, we save the states of the fluid field consisting of velocity in $\x$ and $y$ directions and pressure. To save space, we downsample the resolution from $128\times128$ to $64\times64$.
    \begin{itemize}
        \item velocity: $[\tilde{T}, 2, 64, 64]$.
        \item pressure: $[\tilde{T}, 1, 64, 64]$.
    \end{itemize}
    \item opening angles $\w$: shape $[\tilde{T}]$. For each step, we save the opening angle in radians. 
    \item boundary points: shape $[\tilde{T}, 2, M, 2]$. For each step, we save the boundary points on the upper and lower wings. Each wing consists of $M=20$ points and each point consists of 2 coordinates. To make boundary points compatible with the downsampling of states, coordinates of $\x$ and $y$ directions are shrunk to half ($64/128$) of the original values.
    \item boundary mask and offsets $b$: $[\tilde{T}, 3, 64, 64]$. For each step, we save the mask of merged wings with half coordinates of boundary points and offsets in both $\x$ and $y$ directions. The resolution is $64\times64$, compatible with that of the states.
    \begin{itemize}
        \item mask: $[\tilde{T}, 1, 64, 64]$.
        \item offsets: $[\tilde{T}, 2, 64, 64]$.
    \end{itemize}
    \item force: shape $[\tilde{T}, 2]$. For each step, the simulator outputs the horizontal and vertical force from the fluid to the jellyfish. The horizontal force is regarded as a thrust to jellyfish if positive and a drag otherwise.
\end{itemize}

We generate $n=30,000$ training trajectories and $n=200$ testing trajectories. Trajectories differ in the above specified parameters $\w^{(a)}, \w^{(m)}$ and $\tau$. Each trajectory occupies about 2MB of storage and the total dataset costs about 100GB. 

\section{Additional Details for 2D Jellyfish Movement Control}
\label{app:2d_experiment}

\subsection{Dataset Preparation}
Based on our generated dataset in Appendix \ref{app:2d_dataset_details}, we prepare training samples for the 2D jellyfish movement control task as follows. We use sliding time windows that contain $T=20$ successive time steps of states and boundaries as a sample, which corresponds to $T^{'} =200$ original simulation steps and constitutes exactly a period of wing movement. In this way, each trajectory can produce 20 samples. Therefore, we get 6 million training samples in total. In each training sample, the initial and the final time steps share the same opening angle due to periodicity, which serves as the conditions for control. For each test trajectory, we select the opening angle of the jellyfish in the initial time and the initial states as the control condition for both the initial and final time and state initial condition. 

\begin{figure}[t]
\begin{center}
    \includegraphics[width=0.7\linewidth]{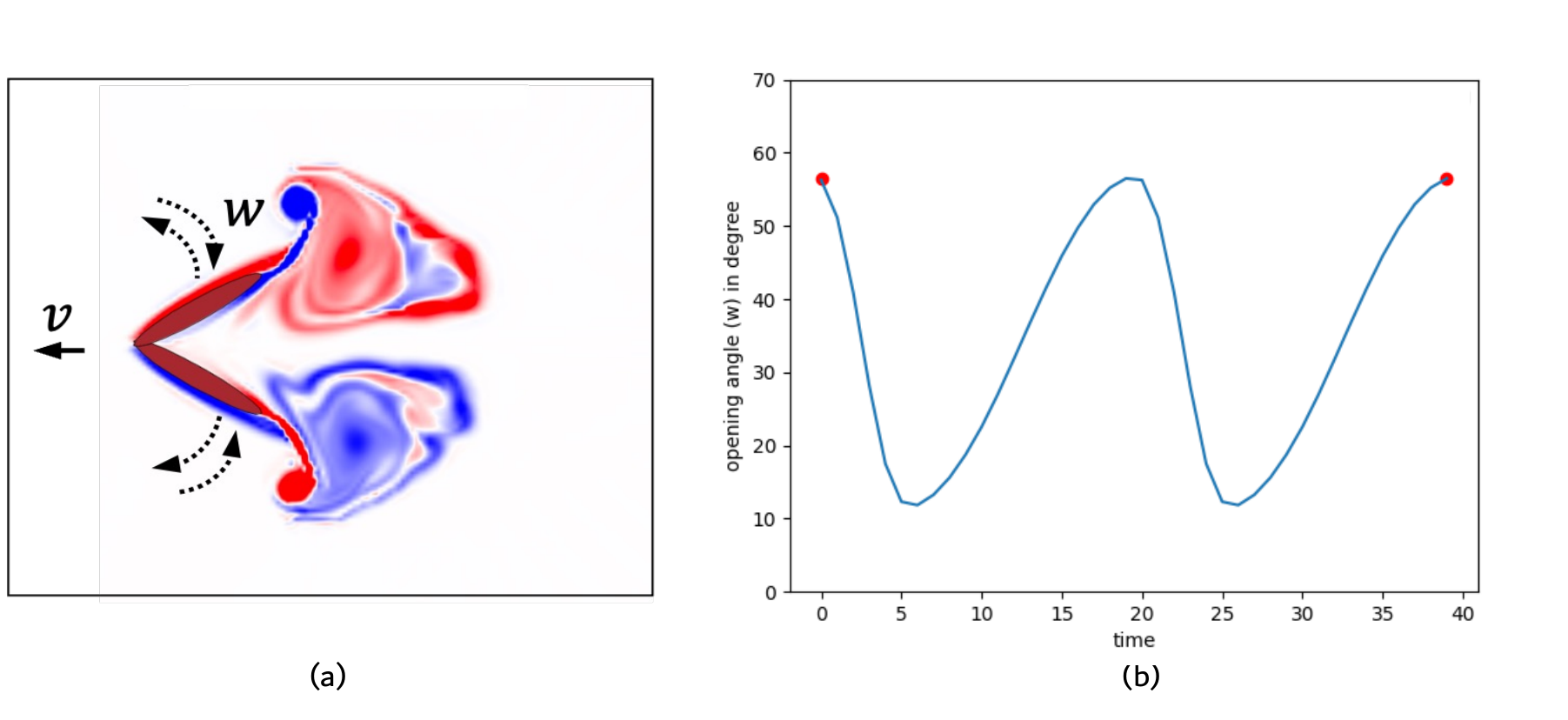}
\end{center}
\vspace{-10pt}
\caption{\textbf{Example of a flapping jellyfish in Lily-Pad simulator (a) and the corresponding curve of opening angle (b)}. The control signal is the opening angles $\w$ of the wings of a jellyfish.}
\vspace{-15pt}
\label{fig:lilypad}
\end{figure}

\subsection{Experimental Setting}

\subsubsection{Full Observation}
In this setting, we assume all the states of the fluid field are observable. That is, both the velocity of $x$ and $y$ directions and pressure are available in all the time steps of the training dataset and the initial time of the testing dataset. 
\subsubsection{Partial Observation}
In this setting, we assume only partial states are observed. A typical scenario in fluid simulation and control is that we can only observe pressure data while the velocity data is not easy to access. That is, only pressure is available in all the time steps of the training samples and the initial time of the testing samples, hence the state tensor is of shape $[\tilde{T}, 1, 64, 64]$. Notice that even if only pressure is available, we can still compute the force of fluid on the jellyfish and consequently the control objective because force is fully determined by the shape of the jellyfish and pressure. The challenge of this partial observation setting is that the velocity variable $v$ is missing in Eq. (\ref{eq:ns_eq}), which makes the traditional numerical solver no longer applicable to solve this physical system control problem. However, this challenge could be well addressed by our method since it could learn the relationship between control and pressure despite missing of the velocity data, and use the accessible control objective as guidance for flapping control.

\subsection{Model}
\subsubsection{Architecture}
We use a 3D U-Net as the backbone of our diffusion model, in both \proj-lite and \proj methods (detailed in the following subsection). 
In this paper, the architecture of the 3D U-Net we employed is inspired by \cite{ho2022video}. To better capture temporal conditional dependencies, we modify the previous space-only 3D convolution into space-time 3D convolut ion. Notably, we did not perform any scaling on the temporal dimension during downsampling or upsampling. Specifically, our U-Net consists of three main modules: the downsampling encoder, the middle module, and the upsampling decoder. 
The downsampling encoder is composed of three layers, each incorporating two residual modules, one spatial attention module, one temporal attention module, and one downsampling module.
The middle module consists of two residual modules, one spatial attention module, and one temporal attention module. Meanwhile, the upsampling decoder consists of four layers, each containing two residual modules, one spatial attention module, one temporal attention module, and one upsampling module.
The input shape of our U-Net is [batch size, frames, channels, height, width]. During convolution, the operation is performed on the [frames, height, width] dimensions. The output shape follows the same structure. Further details are provided in Table \ref{tab:3d-Unet}.

\begin{table}[ht]
  \begin{center}
    \caption{\textbf{Hyperparameters of 3D-Unet architecture}.}
     \label{tab:3d-Unet}
    \begin{tabular}{l|l} 
    \multicolumn{2}{l}{}\\
    \hline
      \text {Hyperparameter name} & {Value}\\
      \hline
      Kernel size of conv3d & (3, 3, 3)   \\
      Padding of conv3d & (1,1,1)  \\
      Stride of conv3d & (1,1,1)  \\
      Kernel size of downsampling & (1, 4, 4)   \\
      Padding of downsampling & (1, 2, 2)  \\
      Stride of downsampling &  (0, 1, 1)  \\
      Kernel size of upsampling & (1, 4, 4)   \\
      Padding of upsampling & (1, 2, 2)  \\
      Stride of upsampling &  (0, 1, 1)  \\
      attention heads & 4 \\
      \hline
   \end{tabular}
  \end{center}
\end{table}

\subsubsection{\proj-lite}
The \proj-lite method learns the denoising network of the joint distribution $p(\u,\w|\mathbf{c})$ where $\u$ is physical states, $\w$ is the opening angle, and the conditions $\c$ consist of the initial angle $\w_0$, the initial state $\u_0$ and the final angle $\w_T=\w_0$. We adopt the 3D U-Net as the backbone. To make the opening angle (of shape $[T]$), align with physical states (of shape $[T, 3, 64, 64]$ in full observation setting and $[T, 1, 64, 64]$ in partial observation setting) in shape, we expand the opening angle to shape $[T, 1, 64, 64]$ along spatial dimension by value copy. Besides, we also adopt the boundary mask and offsets representation, whose shape is $[T, 3, 64, 64]$, determined by the opening angles as an auxiliary model input because they contain explicit spatial features, which makes model learning more effective. Then states, boundary mask and offsets, and expanded opening angle are stacked along the channel dimension and we get a tensor of shape $[T, 7, 64, 64]$ in full observation setting or $[T, 5, 64, 64]$ in partial observation setting as the model input. 
The model output contains predicted noise
of states and open angles. Thus its shape is $[T, 4, 64, 64]$ in the full observation setting or $[T, 2, 64, 64]$ in the partial observation setting, where the last channel corresponds to the predicted noise of opening angles and other channels correspond to predicted noise of states. 

\begin{figure*}[t]
\vspace{15pt}  
\begin{center}
\centering
    \includegraphics[scale=0.5]{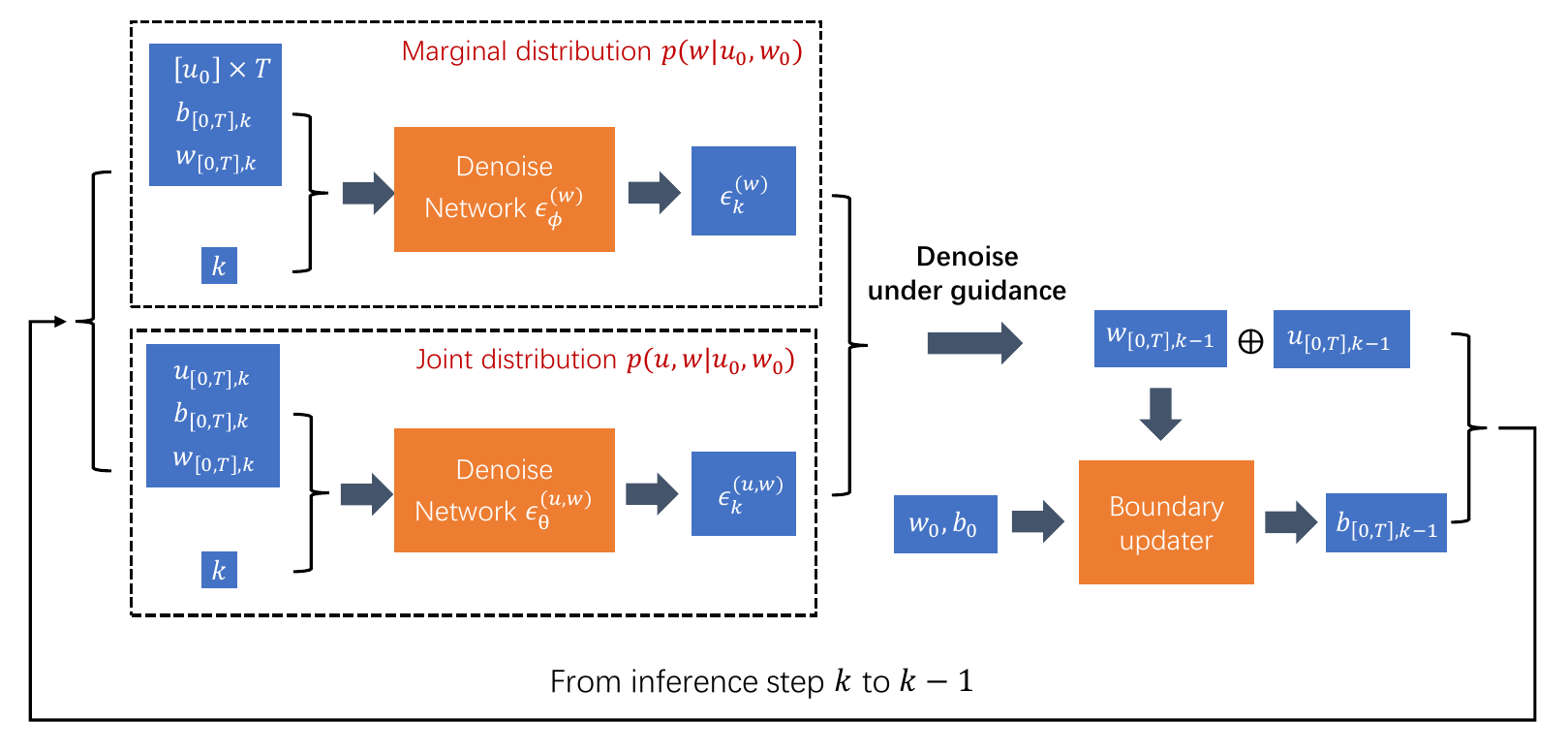}
\end{center}
\caption{\textbf{Inference of our \proj-lite and \proj in the 2D experiment.}}
\label{fig:fig_appendix_2D_inference}
\end{figure*}

\begin{table}[ht]
  \begin{center}
    \caption{\textbf{Hyperparameters of network architecture and training for the 2D experiment}.}
     \label{tab:2D_psi_arch}
    \begin{tabular}{l|l|l} 
    \multicolumn{3}{l}{}\\
    \hline
      \text {Hyperparameter name} & {full observation}  &{partial observation}\\
      \hline
      Batch size &16 &16\\
      Optimizer &Adam &Adam\\
      Initial learning rate &0.001&0.001\\
      Loss function &MSE&MSE\\
      \hline
   \end{tabular}
  \end{center}
\end{table}

\subsubsection{\proj}
\proj learns the denoising network of the joint distribution $p(\u,\w|\mathbf{c})$ and the marginal distribution $p(\w|\mathbf{c})$. The denoising network of $p(\u,\w|\mathbf{c})$ is exactly the same as the one introduced in the \proj-lite method. The denoising network of $p(\w|\mathbf{c})$ also adopts the 3D U-Net architecture. Its input size is the same as that of $p(\u,\w|\mathbf{c})$ in both full and partial observation settings. The difference is that the input states feature is replaced by the expansion of the initial state $\u_0$ along the time dimension by value copy. The output is the predicted noise of opening angles, whose shape is $[T, 1, 64, 64]$, no matter the full observation setting or partial observation setting.

\subsection{Training, Inference, and Evaluation}
\textbf{Training.} We use the MSE (mean squared error) between model prediction and the Gaussian noise as the loss function Eq. (\ref{eq:training_obj}). The batch size is chosen as 16 and the training involves 200,000 iterations. The learning rate starts from $1\times 10^{-3}$ and multiplies a factor of 0.1 at the 50000th and 150000th iterations.
Training details are provided in Table \ref{tab:2D_psi_arch}. The training is performed on two NVIDIA Tesla A100 GPUs with 80 GB memory for about 3 days.

\textbf{Inference.} The pipeline of inference is shown in Figure \ref{fig:fig_appendix_2D_inference}. Both diffused variables $\u_{[0,T]}$ and $\w_{[0,T]}$ are initialized from Gaussian prior and gradually denoised from denoising step $k=1000$ to $k=0$ based on denoising networks and guidance.
Because we introduce the boundary mask and offsets as auxiliary inputs, the model input and output are not consistent in shape. Thus we introduce a surrogate model (shown as "Boundary updater" block in Figure \ref{fig:fig_appendix_2D_inference}) to update boundary mask and offsets $b_{[0,T],k}$ for each denoising $k$. Specifically, at each time step $t\in[0,T]$, $b_{t,k}$ is estimated by the initial boundary mask and offsets $b_0$, and the difference of opening angle $\w_{t,k}-\w_0$ from time step $0$ to $t$, which is presented in the right part (after "Denoise under guidance") of Figure \ref{fig:fig_appendix_2D_inference}. Details about this surrogate model are presented in \ref{app:surrogate_bd_updater}. Notice that although this surrogate model is trained on noise-free data, we do not worry too much about its generalization to the noisy scalar $\w_{t,k}$ in inference because the estimated $\w_{t,k}$ does not deviate from the normalized range of noisy free $\w_t$ too much.

Our method introduces two kinds of inference: \proj-lite and \proj. In \proj-lite, we only use the denoising network of the joint distribution $p(\u,\w|\u_0,\w_0)$ for inference, while in \proj, we use the additional denoising network of the marginal distribution $p(\w|\u_0,\w_0)$ together with that of the joint distribution for inference. These two branches are plotted in the left part of Figure \ref{fig:fig_appendix_2D_inference}, where the notation $[\u_0]\times T$ means expand initial state $\u_0$ (of shape $[3,64,64]$) along time dimension by value copy to form a tensor of shape $[T,3,64,64]$.

As for guidance, we use a surrogate force model to approximate the force of fluid on jellyfish. This model is detailed in Subsection \ref{app:surrogate_force_model}. In denoising step $k$, its input consists of two parts: the first one is the noise-free state $\hat{\u}_{[0,T]}$ estimated from $\u_{[0,T],k}$ by Eq. (\ref{eq:estimate_x0}); the second one is the noise-free boundary mask and offsets $\hat{b}_{[0,T],k}$ estimated from noise-free $\hat{\w}_{[0,T]}$ by the surrogate model to update boundaries, where $\hat{\w}_{[0,T]}$ is also estimated from $\w_{[0,T],k}$ as in Eq. (\ref{eq:estimate_x0}). The model output is force. Here we only use the horizontal force. Notice that the force could be computed via the surrogate force model no matter whether states are fully or partial observation in that force is irrelevant to the velocity of the fluid. The control objective $\J$ in Eq. \eqref{eq:jellyfish_obj} is computed as a summation of force and $R(\hat{\w}_{[0,T]})$. We fix $\zeta=1000$ as a default setting in Eq. (\ref{eq:jellyfish_obj}) because this value can achieve a balance between scales of the average speed and the regularizer $R(\w)$. We also study the Pareto performance of varying $\zeta$ in Table \ref{tab:2D_results}.
Then the gradients of the objective $\J$ in terms of $\hat{\u}_{[0,T]}$ and $\hat{\w}_{[0,T]}$ are computed and used in guidance. For \proj-lite, these gradients are substracted from $[\u_{[0,T],k}, \w_{[0,T],k}]$ to generate $[\u_{[0,T],k-1}, \w_{[0,T],k-1}]$. For \proj, an additional term of noise $(\gamma-1)\bepsilon_{\phi}$ predicted from the denoising network of the marginal distribution $p(\w|\u_0,\w_0)$ should also be subtracted, as shown in the upper left part of Figure \ref{fig:fig_appendix_2D_inference}. The effect of this term is controlled by the scale of the hyperparameter $\gamma$, as studied in Appendix \ref{app:hyperparams}. The inference is performed on one NVIDIA Tesla-V100 GPU with 32 GB memory for about 3 hours for 50 testing samples.

\textbf{Evaluation.}
The inference outputs opening angles $\w_{[0,T]}$ of $T=20$ steps for 50 testing samples. In simulation, for each testing sample, the ground-truth first $T=20$ steps of the opening angles (which corresponds to 200 simulation steps) are directly input to the Lily-Pad simulator for the reason of generating initial states $u_0$ of fluid, which is followed by the predicted control sequences of opening angles (interpolated to 200 steps of opening angles). The simulator outputs the horizontal force of fluid on the jellyfish for each simulation step. Finally, average speed $\bar{v}$, energy cost $R(\w)$, and objective $\J$ are computed as metrics. The average speed
$\bar{v}=\frac{1}{T}\int_{0}^{T}v_t\mathrm{d}t\approx v_0+\frac{1}{T}\sum_{t=1}^{T-1}(T-t)F_t$ 
$\bar{v}$, where $v_0$ is the initial speed and $F_t$ is the horizontal thrust from the fluid. The mass of the jellyfish is assumed to be 1. The energy cost term $R(\w)=\sum_{t=1}^{T-1}(\w_{t+1}-\w_{t})^2$, where the control sequence $\w=(\w_1,\cdots,\w_T)$ represents the predicted opening angles.
The periodic term $d(\w_{0}, \w_{T})=\max(|\w_{T}-\w_{0}|-\bepsilon,0)$ is the constraint of periodic opening angles with a small threshold $\epsilon=0.01$.

\section{Additional Details for 2D Smoke Control}
\label{app:smoke_experiment}

\subsection{Dataset Preparation}

\begin{figure}[t]
\begin{center}
    \includegraphics[width=0.9\linewidth]{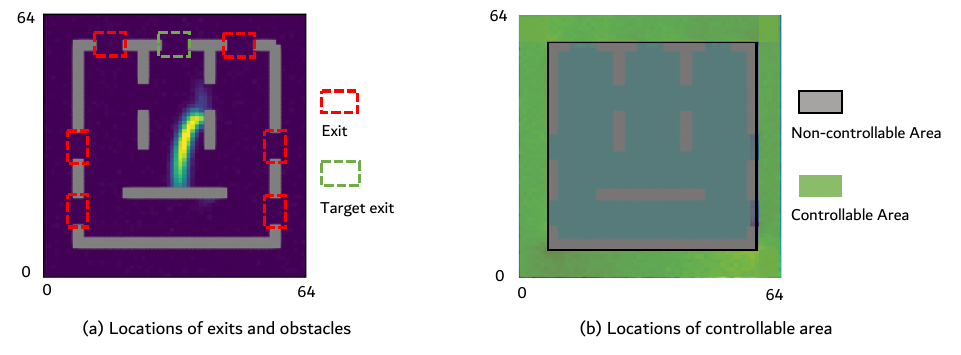}
\end{center}
\caption{\textbf{2D smoke indirect control task}. There are seven exits in total and the top middle one is the target exit (a). The control signals are only allowed to apply to peripheral regions (b). The control objective is to minimize the proportion of smoke failing to pass through the target exit.}
\label{fig:smoke}
\end{figure}

\begin{table}[t]
    \centering
    \begin{minipage}[t]{0.45\linewidth}
        \centering
        \caption{\textbf{Obstacle Positions}}
        \label{tab:obstacle_positions}
        \begin{tabular}{c|c}
            \hline
            \toprule
            Category & Position \\
            \midrule
            \multirow{1}{*}{Bottom} & (8:56,8:9)  \\
            \midrule
            \multirow{3}{*}{Left} & (8:9,8:12) \\
            & (8:9,20:28) \\
            & (8:9,36:56)  \\
            \midrule
            \multirow{3}{*}{Right} & (56:57,8:12) \\
            & (56:57,20:28)  \\
            & (56:57,36:56) \\
            \midrule
            \multirow{4}{*}{Up} & (8:12,56:57)  \\
            & (20:28,56:57)  \\
            & (36:44,56:57) \\
            & (52:56,56:57) \\
            \midrule
            \multirow{5}{*}{Inside Obstacles} & (24:25,32:40)  \\
            & (24:25,48:56)  \\
            & (40:41,32:40) \\
            & (40:41,48:56) \\
            & (20:44, 20:21) \\
            \bottomrule
        \end{tabular}
    \end{minipage}
    \hspace{0.05\linewidth}
    \begin{minipage}[t]{0.45\linewidth}
        \centering
        \caption{\textbf{Absorb Area}}
        \label{tab:absorb_positions}
        \begin{tabular}{c|c}
            \hline
            \toprule
            Category &  Absorb Area \\
            \midrule
            \multirow{2}{*}{Left} & (0:8, 11:21)  \\
            & (0:8, 27:37)  \\
            \midrule
            \multirow{2}{*}{Right} & (56:64, 11:21) \\
            & (56:64, 27:37)  \\
            \midrule
            \multirow{3}{*}{Up} & (11:21, 56:64)  \\
            & (27:37, 56:64) \\
            & (43:53, 56:64)  \\
            \bottomrule
        \end{tabular}
    \end{minipage}
\end{table}

We use the Phiflow solver \cite{holl2020learning} to generate the incompressible fluid with an indirect control dataset.
The resolution of the 2D flow field is set to be 64 $\times$ 64. The flow field is set to be boundless in Phiflow. We placed obstacles and absorbing areas in the middle of the fluid domain. In Figure \ref{fig:smoke}, we illustrate the locations of the exits, obstacles, and controllable areas. The specific positions of the obstacles and absorbing areas are detailed in Table \ref{tab:obstacle_positions} \& Table \ref{tab:absorb_positions}. 

At the beginning of each trajectory, we set the horizontal velocity component $v_x$ of the entire flow field to zero and the vertical component $v_y$ to an upward velocity of 1.0. Additionally, we randomly initialize a square smoke patch with dimensions of $4 \times 4$ in the area below the horizontal obstacles, specifically within the coordinates (11:50, 11:14). In the trajectories we designed, the smoke is required to make four turns. We randomly determine the positions where the smoke will turn, and calculate the smoke's horizontal ($v_x$) and vertical ($v_y$) velocity components, assuming a constant magnitude of velocity. By adjusting the parameters, we observed that when the $v_x$ added at the turning points is sampled from a Gaussian distribution $\mathcal{N}(mv_x, m^2v_x^2/16)$ and $v_y$ is sampled from a Gaussian distribution $\mathcal{N}(6v_y, 9v_y^2/4)$, where is $m$ sampled from a uniform distribution $U(2, 7)$, the success rate of the smoke navigating through is approximately 0.5387\% in the training set and 0.5286\% in the testing set, which meets our requirements. During the motion of the smoke, at turning points, the velocity in the central region is set to the velocity from the previous moment, while the velocity in the surrounding regions is assigned by randomly sampling from the distributions $\mathcal{N}(v_{turn}, {v_{turn}}^2/100)$. At non-turning points, the velocity in the central region remains the velocity from the previous moment, but the velocity in the surrounding areas is adjusted to the previous velocity plus noise sampled from $\mathcal{N}(0, 0.01)$. This procedure ensures that each control parameter varies at every moment and remains distinct from others at any given time. Such variability is beneficial for the subsequent training and generation of control parameters. 

We duplicated the density field into two versions: the original density field and the set-zero density field. The original density field is used for model training and remains unaffected by the absorbing areas. In contrast, the set-zero density field is employed to calculate the amount of smoke passing through each bucket. When set-zero density is present in the absorb areas, we sum up and record this density. Once the recording is complete, we reset the set-zero density in the absorb areas to zero to prevent the double-counting of emitted smoke. Ultimately, we document the quantity of smoke emitted from each bucket at every moment.

Our experimental data records the velocity and density of the entire flow field at each moment, as well as the velocity of the peripheral flow field (control field) after noise addition. 
We generated 40,000 training trajectories and 500 test trajectories, with each trajectory approximately 10 MB in size. Given a total of 40,500 trajectories, the entire dataset is estimated to be around 400 GB in size.

\subsection{Experimental Setting}
In smoke control, we only consider the full observation and full control setting. Namely, all fluid features, including smoke density, horizontal and vertical velocities are observable; and all green areas shown in Figure \ref{fig:smoke} are controllable. 

\subsection{Model}
\subsubsection{Architecture}
Similar to 2D jellyfish control, we use a 3D U-Net as the backbone of our diffusion model, in both \proj-lite and \proj methods (detailed in the following subsections). Details of the architecture are same to Table \ref{tab:3d-Unet}.

\subsubsection{\proj-lite}
Except for three fluid features, we use the accumulated proportion of smoke passing through the target exit as another feature. To make this feature (of shape $[T]$), align with fluid features (of shape $[T, 3, 64, 64]$) in shape, we expand it to shape $[T, 1, 64, 64]$ along spatial dimension by value copy. The horizontal and vertical control signals are represented by a tensor of shape $[T, 2, 64, 64]$, where uncontrollable positions are filled with zeros. Then these four state features and control signals are stacked along the channel dimension and we get a tensor of shape $[T, 6, 64, 64]$ as the model input. The model output contains predicted noise of the state features and control signals. Thus its shape is also $[T, 6, 64, 64]$. 

\subsubsection{\proj}
The denoising network of $p(\w|\mathbf{c})$ also adopts the 3D U-Net architecture. The input and output size is $[T, 2, 64, 64]$.

\subsection{Training, Inference, and Evaluation}
\textbf{Training.} The batch size is chosen as 6. The number of iterations and learning rate schedule are same to 2D jellyfish control. The training is performed on two NVIDIA Tesla A6000 GPUs with 48 GB memory for about 2 days.

\textbf{Inference.} The pipeline of inference is similar to  2D jellyfish control as shown in Figure \ref{fig:fig_appendix_2D_inference}. The difference is that we do not involve any surrogate model in this task. In each sampling step, we use the estimated proportion of smoke failing to exit through the target exit as guidance. 
The inference is performed on one NVIDIA Tesla A6000 GPU with 48 GB memory for about 3 hours for 50 testing samples.

\textbf{Evaluation.}
The inference outputs the control force $\w_{[0,T]}$ of $T=64$ steps for 50 testing samples. In simulation, for each testing sample, the initial state and the generated control signals are input to the simulator. The simulator outputs the proportion of smoke failing to exit through the target exit for the final time step. Then the proportions are averaged over 50 test samples. 

\section{1D Burgers’ Equation Control Baselines} \label{app:baseline1d}
\subsection{PID}\label{subsec:PID}
\begin{figure}[htpb]
  \centering
  \includegraphics[width=0.8\linewidth]{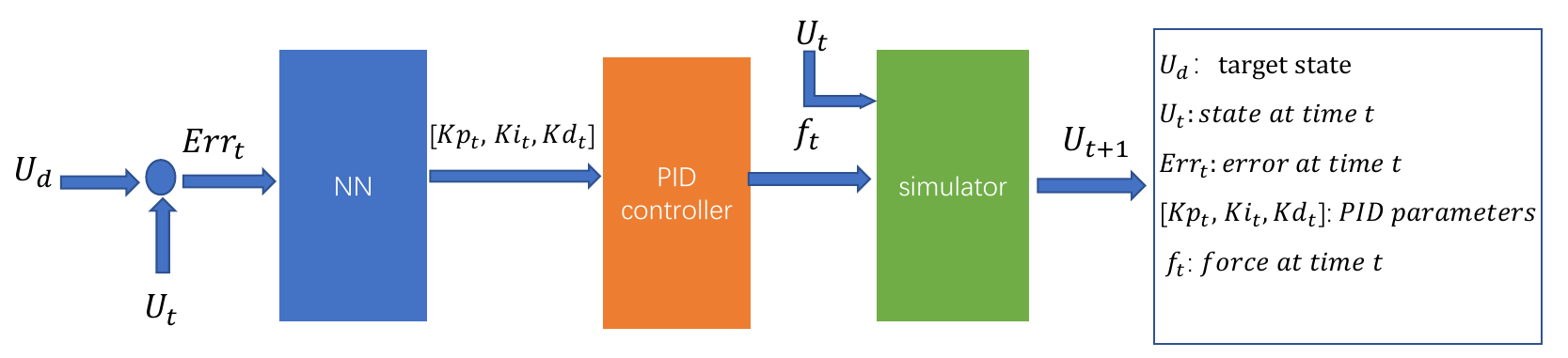}
  \caption{\textbf{The architecture of ANN-PID.} To use MIMO PID controller to control $U_t$ to $U_d$, we train a neural-network-based PID parameter planner to output MIMO PID parameters based on $Err_t$, then use the PID controller to output the control sequence $f_t$.} \label{fig:ANN_PI}
\end{figure}
Proportional Integral Derivative (PID) \cite{1580152} control is a versatile and effective control method widely used in various real-world control scenarios. It operates by utilizing the difference (error) between the desired target and the current state of a system. PID control is often considered the go-to option for many control problems due to its simplicity and usefulness. However, despite its popularity, PID control does encounter certain challenges, such as parameter adaptation and limitations when applied to Single Input Single Output (SISO) systems. In our specific context, the 1D Burgers' Equation Control problem presents a Multiple Input Multiple Output (MIMO) control scenario, which makes it infeasible to directly employ PID control to regulate the Burgers' equation. Inspired by the early works \cite{slama2019neural,9970581} using a neural network as a PID parameter adapter, we have integrated deep learning with PID control to tackle the MIMO control problem. As shown in Figure \ref{fig:ANN_PI}, ANN(artificial neural network) PID uses a neural network as a PID parameter adapter to output multiple sets of PID parameters and do multiple sets of SISO PID control.

The neural network to output PID parameters comprises two 1D convolutional layers, 2 fully connected layers, and 4 corresponding activation layers. We use the $L1$ loss of 
 the current state and target state as training loss and the Adam optimizer \cite{kingma2014adam} to train the model. Detailed information can be found in Table \ref{tab:ANN_PID_architecture}.

\begin{table}[ht]
  \begin{center}
    \caption{\textbf{Hyperparameters of network architecture and training for ANN-PID}.}
     \label{tab:ANN_PID_architecture}
    \begin{tabular}{l|l|l} 
    \multicolumn{3}{l}{}\\
    \hline
      \text {Hyperparameter name} & {Full observation}  &{Partial observation}\\
      \hline
      Kernel size of conv1d & 3 &3  \\
      Padding of conv1d & 1 & 1 \\
      Stride of conv1d & 1 & 1 \\
      Activation function & Softsign&Softsign\\
      Batch size &16 &16\\
      Optimizer &Adam &Adam\\
      Learning rate &0.0001&0.0001\\
      Loss function &MAE&MAE\\
      \hline
   \end{tabular}
  \end{center}
\end{table}

As PID itself is a SISO control method, ANN-PID uses a neural network to get multiple sets of PID parameters to do multiple SISO PID controls for MIMO control in the Burgers' equation. But here, ANN-PID requires the dimensions of inputs and outputs to be the same, so it can only cope with full observation, full control control problems, and partial observation, partial control control problems.
Besides, the ANN-PID controller has 2 training setups, including directly interacting with the solver and interacting with the 1D surrogate model in Appendix \ref{app:surrogate_model}.

\subsection{SAC}
\label{app:SAC1d}

\begin{table}[t]
  \begin{center}
    \caption{\textbf{Hyperparameters of 1D SAC.} The full observation partial control, partial observation full control, and partial observation partial control settings share the same hyperparameters.}
     \label{tab:SAC1d}
    \begin{tabular}{l|l} 
    \multicolumn{2}{l}{}\\
    \hline
      \text {Hyperparameter name} & {Value}  \\
      \hline
       \multicolumn{1}{l|}{Hyperparameters for 1D Burgers’ equation control:} & \multicolumn{1}{l}{}\\
      \hline
      Discount factor for reward & 0.5 \\
      Target smoothing coefficient & 0.05 \\
      Learning rate of critic loss & 0.0003 \\
      Learning rate of entropy loss & 0.003 \\
      Learning rate of policy loss & 0.003 \\
      Training batch size & 8192 \\
      Number of episodes &1500 \\
      Number of model updates per simulator step &50 \\
      Value target updates per step & 15 \\
      Size of replay buffer & 1000000 \\
      Number of trajectories interacted with the environment per step & 1 \\
      Number of layers of critic networks & 3 \\
      Number of hidden dimensions of critic networks & 4096 \\
      Number of layers of the policy network & 5 \\
      Number of hidden dimensions of the policy network & 4096 \\
      Activation function & ReLU \\
      Clipping's range of policy network's standard deviation output & $\left[e^{-20},e^2\right]$ \\
      \hline
    \end{tabular}
  \end{center}
\end{table}

The Soft Actor-Critic (SAC) algorithm \cite{haarnoja2018soft} is a cutting-edge reinforcement learning method. Conceptualized as an improvement over traditional Actor-Critic methods, SAC distinguishes itself by introducing an entropy regularization term into the loss function, which encourages the policy to explore more efficiently by maximizing both the expected cumulative reward and the entropy of the policy itself. 

Compared with Deep Deterministic Policy Gradient (DDPG) algorithm \cite{lillicrap2015continuous, pan2018reinforcement}, SAC's entropy regularization encourages more effective exploration and prevents early convergence to suboptimal policies, a limitation often seen with DDPG's deterministic approach. Additionally, SAC's twin Q-networks mitigate the overestimation bias that can affect DDPG's value updates, leading to more stable learning. The automatic tuning of the temperature parameter in SAC further simplifies the delicate balance between exploration and exploitation, reducing the need for meticulous hyperparameter adjustments. Consequently, these features render SAC generally more sample-efficient and robust, particularly in complex and continuous action spaces.

During training, experience for training is stored in a replay buffer and sampled randomly to update the networks. All data in the training set are in the replay buffer at the beginning. For offline SAC, the replay buffer is unchanged, while online SAC alternates between collecting experience by interacting with the environment and updating the networks with the replay buffer. And offline SAC only uses the surrogate model trained with the training set instead of the real environment to collect experience. The policy network is updated to maximize the expected return, considering both the Q-value and the entropy term. The critic networks are updated to minimize the distance between their Q-value predictions and the target Q-values. SAC also employs a target critic network for the critic networks, which are slowly updated with the weights of the main critic network to stabilize training. For the inference, SAC uses the policy network to determine the action by selecting the action with the highest probability. 

In practice, to help the system approximate the target state accurately and quickly, we need to include the distance between the states of every time step and the target state in the reward. So the reward function of time step $t$, state $u_t$, target state $u_T$ and action $\w_t$ here is defined as
\begin{align*}
    r(t, \u_t, y_T, \w_t) = -\int_{\Omega}|\u_t-\u_d|^2\mathrm{d}\x-{\alpha}\int_\Omega|\w_t|^2\mathrm{d}\x,
\end{align*}
where $\Omega$ is the space domain and $\alpha$ is the weight of energy. We take the Adam optimizer \cite{kingma2014adam} to train the networks and update the temperature parameter. The detailed values of hyperparameters are provided in Table \ref{tab:SAC1d}.

\subsection{Supervised Learning}
The paper \cite{hwang2022solving} proposes a supervised-learning-based control algorithm that takes a neural operator as a surrogate model to solve control problems. It contains two stages. In the first stage, we take a neural operator to learn the physical constraint as Appendix \ref{app:surrogate_model}. The three CNNs respectively reconstruct $u$, reconstruct $\w$ and learn the transition from $u_t$ to $u_{t+1}$. More details are in Appendix \ref{app:surrogate_model}. In the second stage, these three neural networks are used as surrogate models to calculate the gradient of the objective function with respect to the control input. We consider the control $\w$ as a learnable parameter and update it with the gradient. 

To enhance the accuracy, we adopt the LBFGS optimizer \cite{liu1989limited}, which is more accurate while slower than the Adam optimizer. We record the hyperparameters of the second stage in Table \ref{tab:aaai}.

\begin{table}[ht]
  \begin{center}
    \caption{\textbf{Hyperparameters of the second stage of the 1D supervised learning method}.}
     \label{tab:aaai}
     \resizebox{1.0\linewidth}{!}{
    \begin{tabular}{l|l} 
    \multicolumn{2}{l}{}\\
    \hline
      \text {Hyperparameter name} & {Value}  \\
      \hline
       \multicolumn{1}{l|}{Hyperparameters for 1D Burgers’ equation control: (full observation partial control)} & \multicolumn{1}{l}{}\\
      \hline
        Learning rate of $\w$ updating & 0.1 \\
        Number of epochs & 300 \\
        Weight of objective function loss & 500 \\
        Weight of reconstruction loss & 0.03 \\
        Termination tolerance on first order optimality of LBFGS optimizer & $4\times10^{-7}$ \\
        Termination tolerance on parameter changes LBFGS optimizer & $4\times10^{-7}$ \\
        \hline
        \text {Hyperparameter name} & {Value}  \\
        \hline
       \multicolumn{1}{l|}{Hyperparameters for 1D Burgers’ equation control: (partial observation partial/full control)} & \multicolumn{1}{l}{}\\ 
      \hline
        Learning rate of $\w$ updating & 0.1 \\
        Number of epochs & 300 \\
        Weight of objective function loss & 50000 \\
        Weight of reconstruction loss & 3 \\
        Termination tolerance on first order optimality of LBFGS optimizer & $4\times10^{-7}$ \\
        Termination tolerance on parameter changes LBFGS optimizer & $4\times10^{-7}$ \\
        \hline
    \end{tabular}
    }
  \end{center}
\end{table}

\subsection{BC} \label{app: bc1d}

The Behavior Cloning (BC) algorithm \cite{pomerleau1988alvinn} is a fundamental imitation learning method. BC aims to learn policies directly from expert demonstrations, leveraging supervised learning to map states to actions. This approach bypasses the need for exploration typically required in reinforcement learning by directly mimicking the behavior observed in the provided demonstrations. BC does not require interaction with the environment during training, which simplifies the learning process and reduces the computational resources needed. 

In Behavior Cloning, the policy network is trained using supervised learning techniques, where the objective is to minimize the difference between the predicted actions and the actions taken by the expert in the training data. The loss function is typically the mean squared error between the predicted and expert actions. The training data, consisting of state-action pairs, is collected from expert demonstrations and stored in a dataset. In the inference, we evaluated the same objective function as the SAC, which is the MSE between the final state after $T-1$ steps' control and the target. For the partially observed or partially controlled settings, we concatenate zeros in the masked dimensions to make the same dimension as the input of the policy network, but the zeros are not considered during the calculation of the metrics. The detailed values of hyperparameters are provided in Table \ref{tab:bc1d}.
\begin{table}[ht]
  \begin{center}
    \caption{\textbf{Hyperparameters of 1D BC.} The full observation partial control, partial observation full control, and partial observation partial control settings share the same hyperparameters.}
     \label{tab:bc1d}
    \begin{tabular}{l|l} 
    \multicolumn{2}{l}{}\\
    \hline
      \text {Hyperparameter name} & {Value}  \\
      \hline
       \multicolumn{1}{l|}{Hyperparameters for 1D Burgers’ equation control:} & \multicolumn{1}{l}{}\\
      \hline
      Learning rate & $1\times10^{-4}$ \\
      Training batch size & 512 \\
      Number of episodes & $5\times10^{5}$ \\
      Size of replay buffer & $2\times10^{6}$ \\
      Number of layers of policy networks & 2 \\
      Number of hidden dimensions of policy networks & 1024 \\
      Activation function & ReLU \\
      \hline
    \end{tabular}
  \end{center}
\end{table}

\subsection{BPPO} \label{app: BPPO1d}
The Behavior Proximal Policy Optimization (BPPO) algorithm \cite{zhuang2023behavior} is an advanced reinforcement learning method that builds upon the strengths of Proximal Policy Optimization (PPO) while incorporating elements from behavior cloning.  BPPO is an offline algorithm that monotonically improves behavior policy in the manner of PPO. Owing to the inherent conservatism of PPO, BPPO restricts the ratio of learned policy and behavior policy within a certain range, similar to the offline RL methods which make the learned policy close to the behavior policy. By leveraging the inherent conservatism of online on-policy algorithms, BPPO addresses the overestimation issue commonly encountered in offline RL settings. The algorithm starts by estimating a behavior policy using behavior cloning and then iteratively improves a target policy using the Proximal Policy Optimization (PPO) objective with a behavior constraint. Through this process of policy improvement, advantage estimation, and policy update, BPPO aims to refine the target policy while ensuring it remains close to the behavior policy. By combining the strengths of online on-policy methods with tailored offline RL techniques, BPPO demonstrates promising results on the D4RL benchmark, outperforming state-of-the-art offline RL algorithms.

During the training, BPPO first initializes the behavior policy $\pi_\beta$ and target policy $\pi_\theta$. Then, it estimates the behavior policy $\pi_\beta$ using behavior cloning to mimic the behavior demonstrated in the offline dataset. Next, the target policy $\pi_\theta$ is optimized by the PPO objective with a behavior constraint, ensuring that the target policy remains close to the behavior policy. After that, it estimates the advantage function $A^{\pi_\beta}$ using the behavior policy $\pi_\beta$ to evaluate the quality of actions taken by the target policy. Finally, updating the target policy by maximizing the PPO objective with that estimated advantage function, and adjusting the policy parameters to improve performance. In the implementation, a state value network and Q value network are pre-trained using the state, action, and reward from the offline dataset.

In practice, to help the system approximate the target state accurately and quickly, we need to include the distance between the states of every time step and the target state in the reward. So the reward function of time step $t$, state $u_t$, target state $u_d$ and action $\w_t$ here is defined as
\begin{align*}
    r(t, \u_t, \u_d, \w_t) = -\int_{\Omega}|\u_t-\u_d|^2\mathrm{d}\x-{\alpha}\int_\Omega|\w_t|^2\mathrm{d}\x,
\end{align*}
where $\Omega$ is the space domain and $\alpha$ is the weight of energy. We take the Adam optimizer \cite{kingma2014adam} to train the networks and update the temperature parameter. The detailed values of hyperparameters are provided in Table \ref{tab:BPPO1d}.

\begin{table}[ht]
  \begin{center}
    \caption{\textbf{Hyperparameters of 1D BPPO.} The full observation partial control, partial observation full control, and partial observation partial control settings share the same hyperparameters.}
     \label{tab:BPPO1d}
    \begin{tabular}{l|l} 
    \multicolumn{2}{l}{}\\
    \hline
      \text {Hyperparameter name} & {Value}  \\
      \hline
       \multicolumn{1}{l|}{Hyperparameters for 1D Burgers’ equation control:} & \multicolumn{1}{l}{}\\
      \hline
        State value network: & \\
        Learning rate of value network & $1\times10^{-4}$ \\
        Steps of value network & $2\times10^{6}$ \\
        Number of layers of value network & 3 \\
        Batch size of value network & 512 \\
        Number of hidden dimensions of value network & 512 \\
        \hline
        Q value network: & \\
        Learning rate of Q network & $1\times10^{-4}$ \\
        Steps of Q network & $2\times10^{6}$ \\
        Number of layers of Q network & 2 \\
        Batch size of Q network & 512 \\
        Number of hidden dimensions of Q network & 1024 \\
        Target Q network updates per step & 2 \\
        Soft update factor & 0.005 \\
        Discount factor for reward & 0.99 \\
        \hline
        Behavior cloning: & \\
        Learning rate of BC & $1\times10^{-4}$ \\
        Training batch size of BC & 512 \\
        Number of episodes of BC & $5\times10^{5}$ \\
        \hline
        BPPO: & \\
        Number of episodes of BPPO & $1\times10^{2}$ \\
        Number of layers of policy networks  & 2 \\
        Number of hidden dimensions of policy networks & 1024 \\
        Learning rate of BPPO & $1\times10^{-5}$ \\
        Training batch size of BPPO & 512 \\
        Clip ratio of BPPO & 0.25 \\
        Weight decay factor & 0.96 \\
        Weight of advantage function & 0.9 \\
        Size of replay buffer & $2\times10^{6}$ \\
        Activation function & ReLU \\
      \hline
    \end{tabular}
  \end{center}
\end{table}

\section{2D Jellyfish and Smoke Control Baselines} \label{app:baseline2d}

\subsection{MPC and SL}\label{subsec:MPC}
Model Predictive Control (MPC) \cite{schwenzer2021review} is a control strategy that solves an optimization problem repeatedly to determine the optimal control inputs for a dynamic system. It operates over a finite prediction horizon, optimizing a cost function and applying only the first control action. In the 2D jellyfish movement control problem, MPC uses the control sequences $\w$ and the fluid states as internal state variables. 
Without the need to train a control agent model, MPC relies on 2D surrogate models mentioned in Appendix \ref{app:surrogate_model} to estimate future states based on the current state and control. We use backpropagation to compute the gradient, update the control action sequences, and optimize the control objective $\mathcal{J}$ in Eq. (\ref{eq:jellyfish_obj}). For every time step, we get the optimized sequences from this time step forward in this way, and only the first control sequence of the optimized control sequences is applied. Compared with MPC, the Supervised learning (SL) method \cite{hwang2022solving} only optimizes the entire control sequences and employs the entire sequences.

MPC is an optimization technique that aims to optimize the control of complex dynamic systems by considering future predictions. This approach can optimize performance measures over a future time horizon, handle systems with multiple variables and constraints, adapt to changes in the system behavior, and offer good performance even in the presence of nonlinearity. Nevertheless, using MPC comes with a high cost, both in terms of computational resources and time. Additionally, adapting the optimization hyperparameters for MPC can be a challenging task. In our experiment, both MPC and SL face difficulties when trying to generate smooth opening angle control curves, even when constraints $R(\hat{\w})$ are included in the optimization objective $\mathcal{J}$. In the case of multi-objective optimization problems, it becomes even more challenging for them to simultaneously achieve both the higher speed (bigger $\bar{v}$) and the control curves smoothness (smaller $R(\hat{\w})$).

\subsection{SAC}
For the 2D case, the algorithm and basic architecture of SAC are the same as the 1D case in Appendix \ref{app:SAC1d}. When designing the reward function for the 2D jellyfish movement control, we find the periodic condition of the opening angle curves is hard to constrain. So to satisfy the periodic condition better, we include the distance between $\w_t$ of every time step $t$ and $\w_0$. Also, we both consider the squared and absolute error of $(\w_t-\w_0)$ since they respectively constrain the periodic condition when $(\w_t-\w_0)$ is large and small. As a result, the reward function of time step $t$, force $F_t$, opening angle $(\w_{t-1}, \w_t)$ and condition angle $\w_0$ is defined as 
\begin{align*}
    r(t, \w_{t-1}, \w_t, \w_0)=(T-t)*F_t-\lambda_1(\w_t-\w_{t-1})^2-\lambda_2((\w_t-\w_{0})^2 + |\w_t-\w_{0}|), 
\end{align*}
where $\lambda_1$, $\lambda_2$ are weights of different terms. As for the 2D incompressible fluid control, we set the reward function as the percentage of smoke passing through the target bucket.

We take the Adam optimizer \cite{kingma2014adam} to update the weights of networks and the temperature parameter as in the 1D experiment. The hyperparameters are reported in Table \ref{tab:SAC2d}. In particular, we take the best checkpoint to evaluate the final performance, thus the actual number of training episodes for each setting ranges from about 100 to about 300.

\begin{table}[ht]
  \begin{center}
    \caption{\textbf{Hyperparameters of 2D SAC.} The full observation and partial observation settings share the same hyperparameters.}
     \label{tab:SAC2d}
    \begin{tabular}{l|l} 
    \multicolumn{2}{l}{}\\
    \hline
      \text {Hyperparameter name} & {Value}  \\
      \hline
       \multicolumn{1}{l|}{Hyperparameters for 2D Jellyfish movement control:} & \multicolumn{1}{l}{}\\
      \hline
      Weight of the constraint of periodic condition $\beta$ & 0.001 \\
      Discount factor for reward & 0.5 \\
      Target smoothing coefficient & 0.05 \\
      Learning rate of critic loss & 0.0003 \\
      Learning rate of entropy loss & 0.0003 \\
      Learning rate of policy loss & 0.0003 \\
      Training batch size & 2048 \\
      Number of episodes & 350 \\
      Number of model updates per simulator step & 20 \\
      Value target updates per step & 15 \\
      Size of replay buffer & 11400001 \\
      Number of trajectories interacted with the environment per step & 5 \\
      Activation function & ELU \\
      Clipping's range of policy network's standard deviation output & $\left[e^{-5},e^{-2}\right]$ \\
        \hline
    \end{tabular}
  \end{center}
\end{table}

\subsection{BC}
For the 2D task, the algorithm and basic architecture of BC are the same as the 1D case in Appendix \ref{app: bc1d}. In Behavior Cloning, the policy network is trained using supervised learning techniques, where the objective is to minimize the difference between the predicted actions and the actions taken by the non-expert in the training data. The loss function is typically the mean squared error between the predicted actions and actions in training data. The training data, consisting of state-action pairs, is collected from non-expert random demonstrations and stored in a dataset. In the inference, we evaluated the same objective function as the SAC. The detailed values of hyperparameters are provided in Table \ref{tab:bc2d}.

\begin{table}[ht]
  \begin{center}
    \caption{\textbf{Hyperparameters of 2D BC.} The full observation and partial observation settings share the same hyperparameters.}
     \label{tab:bc2d}
    \begin{tabular}{l|l} 
    \multicolumn{2}{l}{}\\
    \hline
      \text {Hyperparameter name} & {Value}  \\
      \hline
       \multicolumn{1}{l|}{Hyperparameters for 2D control:} & \multicolumn{1}{l}{}\\
      \hline
      Learning rate & $1\times10^{-4}$ \\
      Training batch size & 512 \\
      Number of episodes & $5\times10^{3}$ \\
      Size of replay buffer & $2\times10^{6}$ \\
      Activation function & ELU \\
      Clipping's range of policy network's standard deviation output & $[-5,-2]$ \\
        \hline
    \end{tabular}
  \end{center}
\end{table}

\subsection{BPPO}
For the 2D task, the algorithm and basic architecture of BPPO are the same as the 1D case in Appendix \ref{app: BPPO1d}. In jellyfish movement control, when designing the reward function, we find the periodic condition of the opening angle curves is hard to constrain. So to satisfy the periodic condition better, we include the distance between $\w_t$ of every time step $t$ and $\w_0$. Also, we both consider the squared and absolute error of $(\w_t-\w_0)$ since they respectively constrain the periodic condition when $(\w_t-\w_0)$ is large and small. As a result, the reward function of time step $t$, force $F_t$, opening angle $(\w_{t-1}, \w_t)$ and condition angle $\w_0$ is defined as 
\begin{align*}
    r(t, \w_{t-1}, \w_t, \w_0)=(T-t)*F_t-\lambda_1(\w_t-\w_{t-1})^2-\lambda_2((\w_t-\w_{0})^2 + |\w_t-\w_{0}|), 
\end{align*}
where $\lambda_1$, $\lambda_2$ are weights of different terms. The hyperparameters are reported in Table \ref{tab:BPPO2d}.

\begin{table}[ht]
  \begin{center}
    \caption{\textbf{Hyperparameters of 2D BPPO.} The full observation and partial observation settings share the same hyperparameters.}
     \label{tab:BPPO2d}
    \begin{tabular}{l|l} 
    \multicolumn{2}{l}{}\\
    \hline
      \text {Hyperparameter name} & {Value}  \\
      \hline
       \multicolumn{1}{l|}{Hyperparameters for 2D control:} & \multicolumn{1}{l}{}\\
      \hline
        State value network: & \\
        Learning rate of value network & $1\times10^{-4}$ \\
        Steps of value network & $2\times10^{3}$ \\
        Number of layers of value network & 3 \\
        Batch size of value network & 512 \\
        \hline
        Q value network: & \\
        Learning rate of Q network & $1\times10^{-4}$ \\
        Steps of Q network & $2\times10^{3}$ \\
        Number of layers of Q network & 3 \\
        Batch size of Q network & 512 \\
        Target Q network updates per step & 2 \\
        Soft update factor & 0.005 \\
        Discount factor for reward & 0.99 \\
        \hline
        Behavior cloning: & \\
        Learning rate of BC & $1\times10^{-4}$ \\
        Training batch size of BC & 512 \\
        Number of episodes of BC & $5\times10^{3}$ \\
        \hline
        BPPO: & \\
        Number of episodes of BPPO & $1\times10^{2}$ \\
        Number of layers of policy networks  & 3 \\
        Learning rate of BPPO & $1\times10^{-5}$ \\
        Training batch size of BPPO & 512 \\
        Clip ratio of BPPO & 0.25 \\
        Weight decay factor & 0.96 \\
        Weight of advantage function & 0.9 \\
        Size of replay buffer & $2\times10^{6}$ \\
        Activation function & ELU \\
      \hline
    \end{tabular}
  \end{center}
\end{table}

\section{Surrogate Models}\label{app:surrogate_model}
\subsection{1D Surrogate Model}
For the control problem of 1D Burgers' equation, our 1D surrogate model is based on the previous paper \cite{hwang2022solving}, which uses 2 autoencoders to model dynamics in the latent space. The neural simulator architecture and training details are shown in Table \ref{tab:1d_surrogate_model}.

\begin{table}[ht]
  \begin{center}
    \caption{\textbf{Hyperparameters of 1D surrogate model}.}
     \label{tab:1d_surrogate_model}
    \begin{tabular}{l|l|l|l} 
    \hline
      \multirow{2}{*}{\text {Hyperparameter name}} &
      Full observation, & Partial observation, & Partial observation,\\ &
      partial control & full control & partial control
      \\
      \hline
      \multicolumn{4}{c}{Autoencoder of state}\\
      \hline
      Convolution kernel size &5&5&5\\
      Convolution padding &2&2&2\\
      Activation function &ELU&ELU&ELU\\
      Latent vector size &256&128&128\\
      \hline
      \multicolumn{4}{c}{Autoencoder of force}\\
      \hline
      Convolution kernel size &5&5&5\\
      Convolution padding &2&2&2\\
      Activation function &ELU&ELU&ELU\\
      Latent vector size &256&256&256\\
      \hline
      \multicolumn{4}{c}{Training}\\
      \hline
      Training batch size & 5100 & 5100& 5100 \\
      Optimizer &Adam&Adam&Adam\\
      Learning rate &1e-3&1e-3&1e-3\\
      Training epochs & 500 & 500 & 500 \\
      Learning rate scheduler & cosine annealing & cosine annealing & cosine annealing\\
      \hline
   \end{tabular}
  \end{center}
\end{table}

\subsection{2D Jellyfish Force Models}
\label{app:surrogate_force_model}
\textbf{Dataset.} In 2D jellyfish movement control experiments, we train a force surrogate to approximate the computation of the average speed of the jellyfish for the guidance of inference, which is implemented by a neural network and is thus differentiable. The training data consists of pressure, boundary, and force data in the training trajectories. Each training trajectory amounts to $\tilde{T}=40$ training samples. Therefore, we have 1.2 million training samples and 4 thousand testing samples in total.  

\textbf{Model.} The model's input contains pressure, boundary mask, and offsets with shape $4\times 64 \times 64 $ at a certain time step, and the output is the corresponding forces of $x$ and $y$ directions. The model architecture is the down-sampling part of a U-Net \cite{ronneberger2015u} that embeds the input features into a 512-dimensional hidden representation; then we use a linear function with output dimension two to output forces.

\textbf{Training.}  We use MSE (mean squared error) loss between the ground truth and predicted forces to train the force surrogate model. The optimizer is Adam \cite{kingma2014adam}. The batch size is 64. The model is trained for 10 epochs. The learning rate starts from $1 \times 10^{-4}$ and multiplies a factor of 0.1 every three epochs. After training, the relative $l_2$ test error is 0.4\%.

\subsection{2D Jellyfish Boundary Mask and Offsets Updater}\label{app:surrogate_bd_updater}
\textbf{Dataset.} In 2D jellyfish movement control experiments, we train a boundary mask and offsets updater surrogate to approximate the transition of boundary mask and offsets from time step $0$ to $t$. Thus each training trajectory amounts to $\tilde{T}-1=39$ training samples. 

\textbf{Model.} The input is the boundary mask and offsets at time step $0$ with shape $3\times 64 \times 64$, and the difference of the opening angle from $0$ to $t$. The output is the boundary mask and offsets at time step $t$ with shape $3\times 64 \times 64$. The model architecture is the U-Net \cite{ronneberger2015u} with additional scalar input, similar to the denoising network in DDPM \cite{ho2020denoising}, where the input scalar diffusion step is replaced by the angle difference in our model. 

\textbf{Training.}  We use MSE (mean squared error) loss between the ground truth and predicted boundary mask and offsets to train this surrogate model. Hyperparameters of training are the same as those of the force surrogate model.

\subsection{2D Jellyfish Simulator}
\textbf{Dataset.} In 2D jellyfish movement control experiments, we need to train a surrogate model as a solver of the physical dynamics for the baseline methods like SAC (online) and MPC, because of their iterative nature. Conversely, our diffusion method \textit{does not} need this surrogate model. This model approximates the transition of states under the boundary condition from time step $t$ to $t+1$. Thus each training trajectory amounts to $\tilde{T}-1=39$ training samples. We train two versions of this model for full/partial observation settings.

\textbf{Model.} This surrogate model is also implemented by the U-Net \cite{ronneberger2015u}. The model input is the states, boundary mask, and offsets at time $t$ with shape $6\times 64 \times 64$ for the full observation setting and $4\times 64 \times 64$ for the partial observation setting. The output is the predicted states at time step $t+1$, with shape $3\times 64 \times 64$ for the full observation setting and $1\times 64 \times 64$ for the partial observation setting. 

\textbf{Training.}  We use MSE loss between the ground truth and predicted states to train this surrogate model. Hyperparameters of training are the same as those of the force surrogate model.

\section{Additional Results of Experiments}
\label{app:additional_results}

\begin{table}[t]\small
\centering
\caption{\textbf{More results of 1D Burgers's equation control.} \textit{Italic} font denotes unfair results, bold font denotes the best model among fairly evaluated baselines, and underline denotes the second best model among them.}
\begin{tabular}{l|cc|cc|cc}
\hline
\toprule
& \multicolumn{1}{c|}{PO-FC}             & \multicolumn{1}{c|}{FO-PC}             & \multicolumn{1}{c}{PO-PC}             \\
& $\J$ $\downarrow$ & $\J$ $\downarrow$ & $\J$ $\downarrow$  \\
    \midrule
    SAC (online) & 0.01567  & 0.034092  & 0.02768 \\
    \midrule
    \proj-lite & \underline{0.01139}  & \textbf{0.00037}  & \textbf{0.00494}  \\
    \proj & \textbf{0.01103}  & \textbf{0.00037}  & \textbf{0.00494}  \\
\bottomrule
\end{tabular}
\label{tab:1d_sac_table}
\end{table}

\subsection{\proj for 1D Burgers' Equation}
\label{app:1d_more}

\begin{figure}
    \centering
    \includegraphics[width=0.95\textwidth]{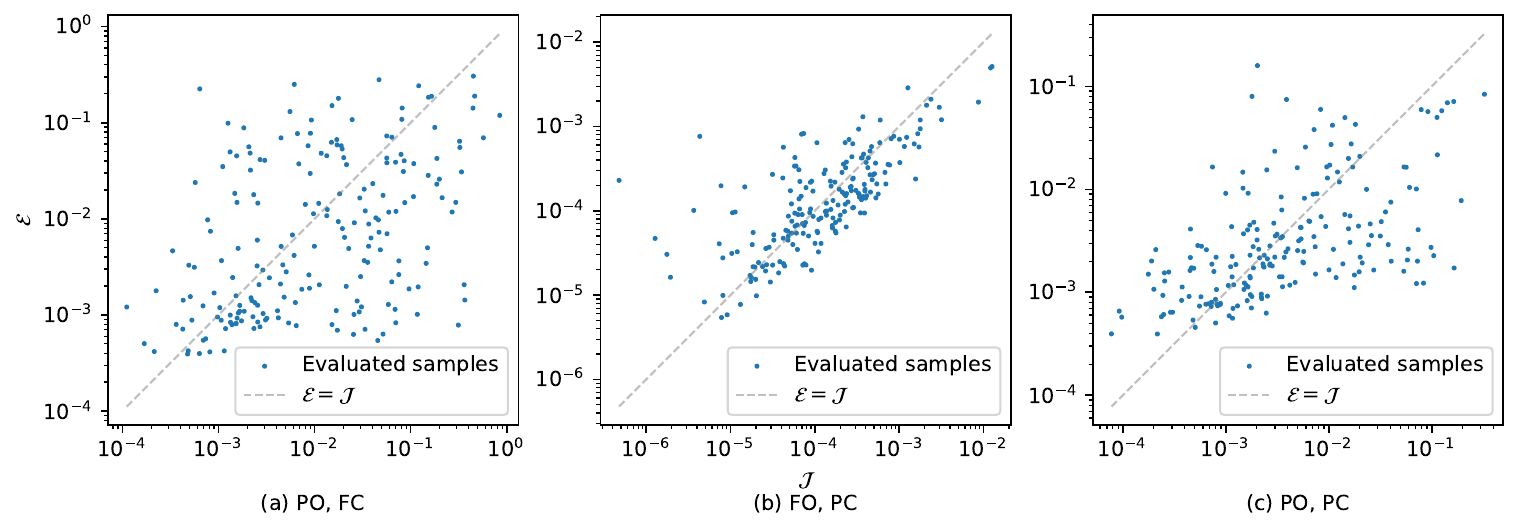}
    \caption{\textbf{$\J$ and state MSE $\mathcal{E}$ of different samples in different settings.}}
    \label{fig:app_cor_J_actual_and_state_mse}
\end{figure}

Since online RL methods are effective in a wide domain of control tasks, we tested the performance of the classical RL method, SAC (online), as shown in \autoref{tab:1d_sac_table}. 
The online SAC belongs to a different setting than the offline training setting considered in our work. Even with access to more information from the environment, online SAC is less competitive to our \proj in Burgers' control task.

As \proj can jointly generate the trajectory $\u$ and control sequence $\w$, we demonstrate the MSE of the diffused trajectory $\u$ from the trajectory computed by the ground-truth solver $\u_{\text{ground truth}}$ given the diffused control sequence $\w$. 
This should be differentiated from $\J$ reported in Table \ref{tab:1d_main_table} which is the MSE of the trajectory computed by the ground-truth solver given $\w$ from the target trajectory $\u^*$.
The former MSE is denoted as $\mathcal{E}\coloneqq\|\u-\u_{\text{ground truth}}\|$ in Figure \ref{fig:app_cor_J_actual_and_state_mse}. 

It can be observed from Figure \ref{fig:app_cor_J_actual_and_state_mse} that there is a clear correlation between $\J$ and $\mathcal{E}$. This is because, the higher the prediction error is, the less accurate the solution to the inverse control problem would be. The correlation proves that our \proj learns the dynamics of the system (i.e., the dependency of $\u$ on $\w$), based on which the control sequence is generated. In partial observation settings, the randomness is larger than in the full observation setting, which demonstrates the challenge of controlling partially observed systems. 
The reference line denotes $\J=\mathcal{E}$ where the control objective completely depends on the prediction accuracy. In Figure \ref{fig:app_cor_J_actual_and_state_mse} (b), samples cluster tightly around the line, which implies \proj can find near-optimal control sequence based on the predicted trajectory in the corresponding level of prediction accuracy. This ability is remarkable in that \proj does not rely on explicit gradient descent like those in the adjoint method and SL.

\subsection{2D Jellyfish Movement}
\label{app:2d_more}

In this subsection, we record extensive results of 2D jellyfish movement control. To further compare baselines with \proj in more settings, we change $\zeta$ in $\mathcal{J}$ and results of $\zeta=500$, $\zeta=1000$ and $\zeta=2000$ are listed in Table \ref{tab:2D_full} and Table \ref{tab:2D_part}. The results show that \proj outperforms other baselines in diverse settings with different objectives.

\begin{table}[t]\small
\centering
\caption{\textbf{Full observation setting of 2D jellyfish movement control.} Bold font denotes the best model, and underline denotes the second best model.}
\resizebox{1.0\linewidth}{!}{
\begin{tabular}{l|ccc|ccc|ccc}
\hline
\toprule
             & \multicolumn{3}{c|}{$\zeta=500$}             & \multicolumn{3}{c|}{$\zeta=1000$}    & \multicolumn{3}{c}{$\zeta=2000$} \\
            & $\bar{v}$ $\uparrow$ &  $R(\w)$ $\downarrow$ & $\mathcal{J}$ $\downarrow$ & $\bar{v}$ $\uparrow$ &  $R(\w)$ $\downarrow$ & $\mathcal{J}$ $\downarrow$ & $\bar{v}$ $\uparrow$ &  $R(\w)$ $\downarrow$& $\mathcal{J}$ $\downarrow$  \\\midrule
MPC                    &    -38.21      &    0.1743      &   212.49     &      25.72     &    0.0112       &     109.17  &   44.24       &    0.0755      &    31.29    \\
SL                     &    90.09      &       0.2570   &    166.89    &     -76.94     &     0.1286      &     205.57  &     -62.96     &   0.0648       &   127.8     \\
SAC (pseudo-online)    &   -134.13      &      0.0204    &   154.49     &     -166.96    &     0.0069      &      178.14  &  -309.18        &     0.0040     &   313.19     \\
SAC (offline)          &    -159.83      &     0.0119     &    171.73    &    -158.66     &   0.0069        &     165.58  &  -278.52        &     0.0051     &    283.58    \\
BC   &    15.61     &   0.0589   &  43.33  &   30.48      &  0.0629      &   32.44   &   33.5      &  0.0677    &       34.23       \\
BPPO   &   88.64      &   0.0776   &  -11.07  &   \underline{107.67}      &  0.0867      &  \underline{-20.93}   & -20.12        &  0.0621    &    82.27          \\
\midrule
\proj-lite             &   \underline{151.57}       &    0.0939      &    \underline{-57.69}    &     95.04     &  0.0746         &     -20.47  &    \underline{150.53}      &   0.0923       &   \underline{-58.23}      \\
\proj                  &    \textbf{310.02}      &   0.2344        &   \textbf{-75.64}      &     \textbf{279.87}     &   0.2058        &    \textbf{-74.11}  &     \textbf{270.56}      &    0.1755      &     \textbf{-95.08}        \\ 
\bottomrule 
\end{tabular}
}
\label{tab:2D_full}
\end{table}

\begin{table}[t]\small
\centering
\caption{\textbf{Partial observation setting of 2D jellyfish movement control.} Bold font denotes the best model, and underline denotes the second best model.}
\resizebox{1.0\linewidth}{!}{
\begin{tabular}{l|ccc|ccc|ccc}
\hline
\toprule
             & \multicolumn{3}{c|}{$\zeta=500$}             & \multicolumn{3}{c|}{$\zeta=1000$}    & \multicolumn{3}{c}{$\zeta=2000$} \\
            & $\bar{v}$ $\uparrow$ &  $R(\w)$ $\downarrow$ & $\mathcal{J}$ $\downarrow$ & $\bar{v}$ $\uparrow$ &  $R(\w)$ $\downarrow$& $\mathcal{J}$ $\downarrow$ & $\bar{v}$ $\uparrow$ &  $R(\w)$ $\downarrow$ & $\mathcal{J}$ $\downarrow$  \\\midrule
MPC                    &    -179.85      &  0.2534        &   433.3     &    -150.51       &    0.1791       &     329.59       &   -116.75       &   0.1397       &    256.46            \\
SL                     &    44.04      &    0.2895      &   245.43     &    -102.98       &   0.1188        &     221.79       &    -152.45      &    0.0661      &       218.53         \\
SAC (pseudo-online)    &    21.87      &    0.0798      &    57.96    &     -153.09      &    0.0057       &     158.82       &   -206.43       &     0.0048     &       211.2         \\
SAC (offline)          &    -149.32      &   0.0156       &   164.93     &    -206.21       &    0.0058       &     211.96       &   -258.96       &    0.0029      &     261.87           \\
BC          &    35.49     &   0.057   &  21.48  &    20.08     &   0.0556     &    35.48  &  \underline{35.14}   & 0.0594  &     \underline{24.27}       \\
BPPO          &    \underline{104.28}    &  0.0746    &  \underline{-29.67}  &    \underline{54.83}     &    0.0518    &  \underline{-3.02}    &  16.97       &   0.0598   &     42.83         \\
\midrule
\proj-lite             &    -0.35      &    0.0782      &   78.56     &    2.92          &      0.0779     &       74.97       &   -10.89      &     0.0760     &      86.93      \\
\proj                  &    \textbf{173.27}      &     0.1636     &  \textbf{-9.66}     &    \textbf{150.21}        &     0.1269      &      \textbf{-23.32}          &    \textbf{101.25}      &   0.1025       &    \textbf{1.21}   \\ 
\bottomrule 
\end{tabular}
}
\label{tab:2D_part}
\end{table}

\begin{table}[ht]
  \begin{center}
    \caption{\textbf{Results of myopic failure modes of SAC on 2D jellyfish movement control}.}
\begin{tabular}{l|cccc}
\hline
$\lambda_0$ &{Average speed ($\bar{v}$)} & \multicolumn{1}{l}{$R(\w)$} & objective  $\J$  & periodicity error    \\ \hline
SAC (weight=100)	& 59.82 &	0.0263 &	-33.48 &	0.35187       \\ 
SAC (weight=1000)	& -166.96 &	0.0112	& 178.14 &	0.02299 \\ 
\proj    &  279.87	& 0.2058 &	-74.11 &	0.01157 \\ 
 \hline
\end{tabular}
\label{tab:sac_myopic}
\end{center}
\end{table}

\subsection{Myopic failure modes of SAC}\label{app:myopic_failure}
To further confirm the advantage of diffusion models over the baseline SAC in alleviating the myopic failure modes, we perform additional 2D experiments. Specifically, when solving the jellyfish movement problem using SAC, we incorporated a constraint to achieve periodic motion of the jellyfish in the reward function, namely, we added a hyperparameter as the weight of the term $d(w_T, w_0)$ in Eq. \eqref{eq:jellyfish_obj}. The challenge caused by the periodicity condition is that the average speed and $R(\w)$ are calculated over the whole horizon while the periodicity condition is only evaluated at the final time step of the horizon.
Thus, when the control objective consists of all three terms, it requires global optimization over the whole horizon to balance different terms. In Table \ref{tab:sac_myopic}, we present the original results (weight=1000, corresponding to Table \ref{tab:2D_results}) and the results after reducing the weighting of this term (weight=100). It can be observed that while the speed of the jellyfish increases (though still not reaching the results of \proj), the $R(\w)$ and constraint term for periodicity sharply rise. 
From these results and the additional results in Table \ref{tab:2D_full} and Table \ref{tab:2D_part} where we test the performance of SAC under different weights $\zeta$ of $R(\w)$, we conclude that SAC struggles to provide satisfactory policies that simultaneously satisfy multiple conflicting constraints with different time coverage, resulting in a myopic failure model. The reason may be that it is difficult to estimate the effect of each term in the control objective in each time step in the iterative planning fashion \cite{ajay2022conditional}, which this issue could be well addressed by our diffusion models with an inherent nature of global optimization.

\begin{table}[t]\small
\centering
\caption{\textbf{Efficiency comparison on 1D Burgers’ equation.} Inference time is tested on a Tesla-V100 GPU with 8 CPUs.}
\begin{tabular}{lll}
\toprule
Methods                           & Training time (hours) & Inference time (seconds) \\
\midrule
\proj-lite                  & 1.7 (1 A100-80G GPU, 8 CPUs)  & 21.13    \\
\proj                       & 4.4 (1 A100-80G GPU, 8 CPUs)  & 58.97    \\
\proj-DDIM (8 sampling steps)& 1.7 (1 A100-80G GPU, 8 CPUs) & 0.53     \\
PID                            &  0.1 (1 A100-80G GPU, 8 CPUs)    & 3.61     \\
SL                             & 2.6 (1 V100-32G GPU, 12 CPUs)    & 74.85    \\
SAC                            & 10.5 (1 A6000-48G GPU, 16 CPUs)    & 0.11     \\
BC                             & 8.8 (1 V100-32G GPU, 12 CPUs)    & 1.22         \\
BPPO                           & 8.9 (1 V100-32G GPU, 12 CPUs)    & 0.82     \\
\bottomrule 
\end{tabular}
\label{tab:1d_efficiency}
\end{table}

\begin{table}[t]\small
\centering
\caption{\textbf{Efficiency comparison on 2D jellyfish movement.} Inference time is tested on a Tesla-V100 GPU with 8 CPUs.}
\begin{tabular}{lll}
\toprule
Methods                         & Training time (hours) & Inference time (seconds) \\
\midrule
\proj              & 62 (2 A100-80G GPUs, 32 CPUs)           &  252.2   \\
\proj-DDIM (8 sampling steps) & 62 (2 A100-80G GPUs, 32 CPUs) &   12.6   \\
MPC               & 52.1 (1 A100-80G GPU, 16 CPUs)	                 & 1401.7 \\
SL                & 52.1 (1 A100-80G GPU, 16 CPUs)	                 &   133.5  \\
SAC              & 9.5 (1 A100-80G, 16 CPUs)	                  &   0.2   \\
BC                & 2.8 (1 A100-80G, 16 CPUs)	                 &  0.99        \\
BPPO              & 3.0 (1 A100-80G GPU, 16 CPUs)                 &  1.08    \\
\bottomrule 
\end{tabular}
\label{tab:2d_efficiency}
\end{table}

\subsection{Efficiency Comparison}\label{app:efficiency}
We test the training and inference (control) efficiency of \proj and baselines on both 1D Burgers’ equation and 2D jellyfish movement tasks. 
On the 1D Burgers’ equation task, we test the efficiency of partial observation and full control (PO-FC); on the 2D jellyfish movement task, we test the efficiency of full observation (FO).

The results are presented in Table \ref{tab:1d_efficiency} and Table \ref{tab:2d_efficiency}.
For training efficiency, since models with different sizes are trained on different machines, we report the training time of models along with the machine information.
The inference efficiency is tested on a single Tesla-V100 GPU with 8 CPUs and the average inference times over all test samples are reported. 
For training efficiency, we have the following observations. On 1D Burgers’ equation, our \proj is comparable with SL and more efficient than reinforcement learning (RL) methods like SAC, BC, and BPPO. On 2D jellyfish movement, our \proj  costs similar training time compared with MPC and SAC, but more time than RL methods.
For inference efficiency, these results reveal that although \proj is not as efficient as RL methods, it has competitive efficiency compared to
SL and MPC. In particular, by using the fast sampling method DDIM \cite{song2020denoising}, \proj could be accelerated significantly by reducing the number of sampling steps.

\section{Effect of Hyperparameters}\label{app:hyperparams}
\subsection{Effect of $\gamma$}\label{app:hyperparam_gamma}

\textbf{1D Burgers' Equation control}.
In the 1D Burgers control task, the prior reweighting is intrinsically unnecessary. The evaluated control target follows the same distribution as the training set, and the samples in the training set are defined to be the optimal control. Therefore, the generated control sequences following $p(w|u_0,u_T)$ are already the (near) optimal distribution which alleviates the need for prior reweighting. 

We tested different scheduling of the prior reweighting $\gamma$ and found their performance similar. Finally, we use the same cosine schedule as that in the DDPM noise for clarity. Our experiment results also revealed that the performance of \proj in 1D Burgers equation control is insensitive to the prior reweighting intensity $\gamma$ as we reported in Table \ref{tab:1d_main_table}. Besides, we conducted a more comprehensive experiment on the prior reweighting intensity $\gamma$ in all three settings (FO-PC; PO-FC; PO-PC) as in Table \ref{tab:ablation_1D_gamma_FOPC}, \ref{tab:ablation_1D_gamma_POFC}, and \ref{tab:ablation_1D_gamma_POPC}. The results show that prior reweighting does not significantly impact the performance of \proj in 1D Burgers control task, where State MSE denotes the deviation between generated $u$ and the ground truth $u_{\text{gt}}$ given by the numerical simulator based on generated $f$.
\begin{table}[tb]
    \centering
    \caption{\textbf{Results of different $\gamma$ in \proj on FO-PC 1D Burgers equation control task.}}
    \begin{tabular}{l|l|l|l|l|l}
    \toprule
        $\gamma$ & 0.0 & 0.3 & 0.5 & 0.7 & 1.0 \\ \midrule
        $\mathcal{J}_{\text{actual}}$ & 0.00038 & 0.00037 & 0.00037 & 0.00037 & 0.00037 \\ 
        State MSE & 0.00063 & 0.00043 & 0.00034 & 0.00028 & 0.00025 \\ \bottomrule
    \end{tabular}
    \label{tab:ablation_1D_gamma_FOPC}
\end{table}

\begin{table}[tb]
    \centering
    \caption{\textbf{Results of different $\gamma$ in \proj on PO-FC 1D Burgers equation control task.}}
    \begin{tabular}{l|l|l|l|l|l}
    \toprule
        $\gamma$ & 0.0 & 0.3 & 0.5 & 0.7 & 1.0 \\ \midrule
        $\mathcal{J}_{\text{actual}}$ & 0.01159 & 0.01155 & 0.01152 & 0.01148 & 0.01139 \\ 
        State MSE & 0.02643 & 0.02642 & 0.02642 & 0.02641 & 0.02639 \\ \bottomrule
    \end{tabular}
    \label{tab:ablation_1D_gamma_POFC}
\end{table}

\begin{table}[tb]
    \centering
    \caption{\textbf{Results of different $\gamma$ in \proj on PO-PC 1D Burgers equation control task.}}
    \begin{tabular}{l|l|l|l|l|l}
    \toprule
        $\gamma$ & 0.0 & 0.3 & 0.5 & 0.7 & 1.0 \\ \midrule
        $\mathcal{J}_{\text{actual}}$ & 0.00493 & 0.00493 & 0.00493 & 0.00493 & 0.00494\\ 
        State MSE & 0.00675 & 0.00667 & 0.00665 & 0.00664 & 0.00663 \\ \bottomrule
    \end{tabular}
    \label{tab:ablation_1D_gamma_POPC}
\end{table}

\begin{table}[ht]
  \begin{center}
    \caption{\textbf{Results of different $\gamma$ in \proj on 2D jellyfish movement control task}.}
\begin{tabular}{c|c|ccc}
\toprule
$\gamma_1$ &$\xi$ & \multicolumn{1}{c}{Average speed ($\bar{v}$)} & \multicolumn{1}{l}{$R(\w)$} & objective  $\J$ \\ 
\midrule
0.6&0.4      & 410.6                                                   & 0.2581                                   & -152.51  \\
0.7&0.3      & 279.87                                                 & 0.2058                                   & -74.11  \\
0.8&0.2      & 197.18                                                 & 0.1312                                    & -65.99   \\
0.9&0.1      & 76.97                                                    & 0.0741                                   & -2.84    \\
1.0&0        & 95.04                                                 & 0.0746                                    & -20.47   \\
1.1&-0.1     & 81.41                                                    & 0.0742                                   & -7.21   \\
1.2&-0.2     & 84.56                                                   & 0.0736                                    & -10.93   \\
1.3&-0.3     & 65.12                                                   & 0.0725                                   & 7.38    \\
1.4&-0.4     & 65.02                                                & 0.0734                                  & 8.43     \\
1.5&-0.5     & 64.07                                                     & 0.0752                                    & 11.1     \\ \bottomrule
\end{tabular}
\label{tab:effect_gamma}
\end{center}
\end{table}

\begin{figure}[ht]
\begin{center}
    \includegraphics[width=0.5\textwidth]{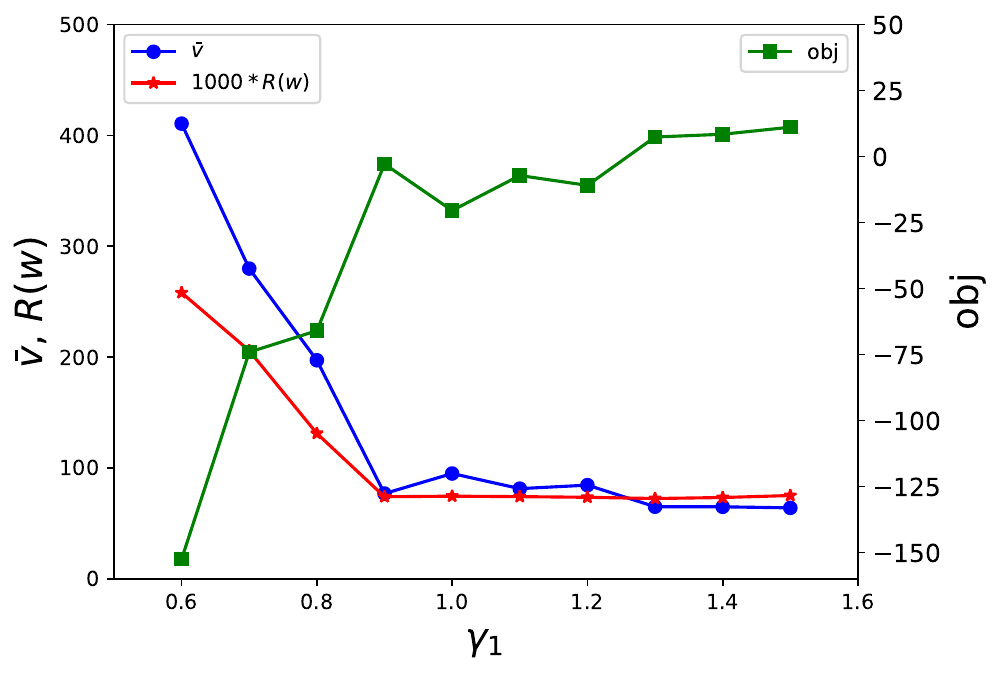}
\end{center}
\caption{\textbf{Results of different $\gamma$ in \proj on 2D jellyfish movement control task}.}
\label{fig:figure_effect_gamma}
\vskip -0.1in
\end{figure}

\textbf{2D jellyfish movement control}. Performance of \proj is determined by the hyperparameter $\gamma$. Since diffusion models denoise gradually,
 we use a varying sequence of $\gamma=\{\gamma_k\}_{k=1}^K$ to subtract prior in \proj.
Specifically, the schedule of $\gamma$ is set as $\gamma_k=1-\xi\cdot \beta_{K-k}, k=1,\cdots,K$, where $\xi$ is a fixed coefficient to control the scale of $\gamma$ and $\beta=\{\beta_k\}_{k=0}^K$ is the schedule of variances of noise in DDPM \cite{ho2020denoising}. In our implementation, we use the sigmoid schedule of $\beta$ \cite{jabri2022scalable}. 
The total number of inference steps is $K=1000$. Thus we only need to tune $\xi$ to examine the effect of $\gamma$. When $\xi<0$, \proj is prone to restrict $\w$ within its prior distribution of training dataset in inference. When $\xi>0$, \proj is more likely to generate new kinds of $\w$ beyond training ones. When $\xi=0$, \proj degenerates to \proj-lite. In 2D experiments, We set default $\xi=0.3$ and the corresponding $\gamma_1=0.7$ as we empirically find this value performs well and steadily. We present the performance of \proj on the 2D jellyfish movement control task under different $\gamma$ in Figure \ref{fig:figure_effect_gamma} and Table \ref{tab:effect_gamma}. We can observe that the performance increases along with decreasing of $\gamma_1$. When $\gamma_1<0.6$, invalid generated control sequences emerge because the prior is largely overlooked. Thus the valid interval for $\gamma$ of prior reweighting on this task is $[0.6,1.0]$. It is interesting to find that when $\gamma_1>1$, the performance decreases. This may be caused by the strict constraint of the prior distribution of $p(\w,\c)$, which results in generating control sequences similar to those from training datasets and thus not good. 

\subsection{Effect of $\lambda$}\label{app:hyperparam_lambda}

\textbf{1D Burgers' Equation control}.
In 1D Burgers control, \proj uses guidance conditioning since it empirically performs better than explicit guidance as shown in Table \ref{tab:ablation_1D_condition_or_guidance}. Therefore, our model does not use $\lambda$ to generate the optimal control. However, when we consider the energy cost as in Figure \ref{fig:pareto1d}, we use explicit guidance to control the magnitude of control signals as it is a more natural way to incorporate the objective $\lambda \mathcal{J}_{\text{energy}}$ into the overall diffusion objective $E_\theta ^{(\gamma)}(u,w,c) + \lambda \mathcal{J}_{\text{energy}} $ where the condition $c$ is $u_{T}=u_{d}$. 

The guidance intensity uses a cosine schedule \cite{nichol_improved_2021} where $\lambda=\lambda_0 \beta_k,~k=1,2,...,K$ where $\beta_k$ starts from $\beta_1=0.001$ to $\beta_K=1$. The schedule decreases as the diffusion step decreases during inference, where the gradually decreasing guidance intensity provides strong guides in the initial stage while reducing the impact on the near-clean samples.

In such a setting, varying $\lambda$ changes how the generated samples are shifted towards lower $\mathcal{\text{energy}}$ regions, but an excessively large $\lambda$ causes the sample to deviate from the physically feasible region. Results in Table \ref{tab:ablation_1D_lambda_FOPC}, \ref{tab:ablation_1D_lambda_POFC}, and \ref{tab:ablation_1D_lambda_POPC} show that the behavior of \proj under different $\lambda$ is expected. Note that ``physically feasible'' samples here are samples with lower $\mathcal{J}_{\text{actual}}$.

\begin{table}[hbt]
    \centering
    \caption{\textbf{Results of two different guidance type of \proj in 1D Burgers equation control task.}}
    \begin{tabular}{l|l|l|l}
    \toprule
        Guidance Type & FO-PC & PO-FC & PO-PC \\ \midrule
        Explicit Guidance & 0.02789 & 0.03257 & 0.05584 \\ 
        Conditioning Guidance & 0.00037 & 0.01139 & 0.00494 \\ \bottomrule
    \end{tabular}\label{tab:ablation_1D_condition_or_guidance}
\end{table}

\begin{table}[htb]
    \centering
    \caption{\textbf{Results of different $\lambda$ for $\mathcal{J}_{\text{energy}}$ of \proj in FO-PC 1D Burgers equation control task.}}
    \resizebox{1.0\linewidth}{!}{
    \begin{tabular}{l|l|l|l|l|l|l|l|l}
    \toprule
        $\lambda$ & 0 & 10 & 100 & 500 & 1000 & 10000 & 100000 & 1000000  \\ \midrule
        $\mathcal{J}_{\text{actual}}$ & 0.00037 & 0.00037 & 0.00037 & 0.00036 & 0.00037 & 0.00147 & 0.02202 & 0.06444 \\
        State MSE & 0.00025 & 0.00025 & 0.00024 & 0.00024 & 0.00024 & 0.00083 & 0.00802 & 0.01838 \\ 
        $\mathcal{J}_{\text{energy}}$ & 1320.98206 & 1309.31104 & 1237.06958 & 1061.67761 & 946.07617 & 584.32098 & 225.50281 & 38.25005 \\ \bottomrule
    \end{tabular}
    }
    \label{tab:ablation_1D_lambda_FOPC}
\end{table}
\begin{table}[htb]
    \centering
    \caption{\textbf{Results of different $\lambda$ for $\mathcal{J}_{\text{energy}}$ of \proj in PO-FC 1D Burgers equation control task.}}
    \resizebox{1.0\linewidth}{!}{
    \begin{tabular}{l|l|l|l|l|l|l|l|l}
    \toprule
        $\lambda$ & 0 & 10 & 100 & 500 & 1000 & 10000 & 100000 & 1000000  \\ \midrule
        $\mathcal{J}_{\text{actual}}$ & 0.01139 & 0.01152 & 0.01252 & 0.01906 & 0.02100 & 0.04411 & 0.08472 & 0.11934 \\ 
        State MSE & 0.02639 & 0.02641 & 0.02661 & 0.02919 & 0.02990 & 0.03857 & 0.05472 & 0.06471 \\ 
        $\mathcal{J}_{\text{energy}}$ & 862.98633 & 860.05164 & 839.11932 & 770.45819 & 704.68665 & 311.21915 & 69.87619 & 8.93373 \\ \bottomrule
    \end{tabular}
    }
    \label{tab:ablation_1D_lambda_POFC}
\end{table}
\begin{table}[htb]
    \centering
    \caption{\textbf{Results of different $\lambda$ for $\mathcal{J}_{\text{energy}}$ of \proj in PO-PC 1D Burgers equation control task.}}
    \resizebox{1.0\linewidth}{!}{
    \begin{tabular}{l|l|l|l|l|l|l|l|l}
    \toprule
        $\lambda$ & 0 & 10 & 100 & 500 & 1000 & 10000 & 100000 & 1000000  \\ \midrule
        $\mathcal{J}_{\text{actual}}$ & 0.00494 & 0.00500 & 0.00549 & 0.00794 & 0.01426 & 0.02411 & 0.03979 & 0.08153 \\ 
        State MSE & 0.00663 & 0.00676 & 0.00768 & 0.01048 & 0.01525 & 0.02766 & 0.03635 & 0.05459 \\ 
        $\mathcal{J}_{\text{energy}}$ & 1314.90698 & 1297.32849 & 1192.89026 & 945.77179 & 793.28442 & 331.45758 & 108.39325 & 16.13734 \\ \bottomrule
    \end{tabular}
    }
    \label{tab:ablation_1D_lambda_POPC}
\end{table}

\textbf{2D jellyfish movement control}. In our jellyfish movement control task, we use explicit guidance with $\lambda$ involved. Since diffusion models denoise gradually, similar to the hyperparameter $\gamma$, we use a varying sequence of $\lambda=$ {$\lambda_k$}$ =\lambda_0 \times ${ $\beta_k$}, where {$\beta_k$} is the schedule of variances of noise in DDPM [2] increasing monotonically from $\beta_1=0.0003$ to $\beta_K=1$. The motivation of this schedule is that the confidence of gradient estimation becomes stronger along with the denoising process from $k=K$ to $k=1$. The results are shown in the Table \ref{tab:effect_lambda_2d}. We can observe that as $\lambda_0$ (which determines the $\lambda$ schedule) increases from 0 to 0.6, the performance of \proj improves consistently. Compared to Table \ref{tab:2D_results}, \proj always outperforms baselines for a wide range of $0.1 \leq \lambda_0\leq0.5$. When $\lambda_0\geq0.6$, the results are no longer valid since the generated control sequences may be unfeasible. This reflects that $\lambda$ should not be overly large to make it safe, which is a consistent conclusion with our previous analysis.

\begin{table}[ht]
  \begin{center}
    \caption{\textbf{Results of different $\lambda$ on 2D jellyfish movement control}.}
\begin{tabular}{l|ccc}
\toprule
$\lambda_0$ &{Average speed ($\bar{v}$)} & \multicolumn{1}{l}{$R(\w)$} & objective  $\J$      \\ \midrule
0.6       & -                                           & -                            & -       \\ 
0.5       & 362.82                                      & 0.192                        & -170.79 \\ 
0.4       & 322.04                                      & 0.182                        & -140.04 \\ 
0.3       & 279.87                                      & 0.206                        & -74.11  \\ 
0.2       & 192.77                                      & 0.114                        & -78.57  \\ 
0.1       & 103.89                                      & 0.091                        & -12.32  \\ 
0         & -86.69                                      & 0.042                        & 128.67  \\ \bottomrule
\end{tabular}
\label{tab:effect_lambda_2d}
\end{center}
\end{table}

\input{archive/2d_flow}

\section{Limitation and Future Work}\label{app:future_work}
There are still several limitations of \proj that inspire future work. Firstly, our approach is data-driven, so it is not guaranteed to achieve optimal solutions and also lacks estimation of the error gap between the generated control sequence and the optimal ones. 
Secondly, the training of \proj is currently conducted in an offline fashion, without interaction with a ground-truth solver. Incorporating solvers into the training framework could facilitate real-time feedback, enabling the model to adapt dynamically to the environment and discover novel strategies and solutions. 
Furthermore, our proposed \proj presently operates in an open-loop manner, as it does not consider real-time feedback from solvers. Integrating such feedback would empower the algorithm to adjust its control decisions based on the evolving state of the environment.

\section{Social Impact Statements}\label{app:social_impact}
The method we propose offers a means to actively interact with the physical world and achieve specific control objectives. This approach presents significant opportunities for various domains, including fluid control, robotic control, controllable nuclear fusion, and more. However, it is imperative to remain vigilant about potential negative consequences that may arise from this technology to prevent its application to unethical or illegal control issues.

%% file: archive/2d_flow.tex
\section{Extensions to Finer-grained Jellyfish Movement Control Task}
\label{app:finer_2d}

To test the scalability of DiffConPDE in handling high-dimensional and complex dynamic systems, we extend the jellyfish movement control experiment from rigid boundaries of wings to soft ones.

\subsection{Experimental Setting}

Similar to previous setups in Appendix \ref{app:2d_experiment}, this extension aims to control the movement of a flapping jellyfish with two wings in a 2D fluid field where fluid flows with constant initial speed. However, in this task, we parametrize the jellyfish's boundary as the coordinate change ($\Delta x$, $\Delta y$) for each cell within the boundary, which serves as the control sequence $\w$. Undoubtedly, this is a high-dimensional and complex control task. Firstly, the control sequence $\w$ is elevated to three dimensions. In addition to ensuring optimization towards the control objective, maintaining consistency in the movement of different cells within the jellyfish boundary is also crucial.  Bearing this in mind, we have chosen to employ DiffConPDE-lite for this experimental setup to minimize potential disruptions to the jellyfish boundary.  Regarding the PDE state, the experiment adopts the full observation setup. All other settings remain the same with Appendix \ref{app:2d_experiment}.

\subsection{Data Generation}
In this extension, we leverage data previously generated under the vanilla setting.  Additionally, we utilize opening angles $\Theta$, The coordinates of the rotation center $\mathbf{h}$, and the boundary mask to generate the required control sequence $\w$.  We represent $\w$ as a tensor of shape $[\tilde{T}, 2, 64, 64]$. Each cell has two features representing the changes in the two coordinates. For each trajectory, $\w$ is generated as follows: Firstly, for the initial $\w_0$ and final $\w_T$, we set all elements in the tensor to $0$. Then, for each $ w_{[1, T-1]}$ corresponding to the cells within the boundary, we first obtain the relative coordinates $c = (x, y)$ of the cell with respect to the rotation center using the boundary mask and the coordinates of the rotation center $\mathbf{h}$. Afterwards, we construct the rotation matrix $M_t(\Theta_t)$ at different time steps using the opening angle $\Theta_{[1, T-1]}$:
\begin{equation*}
    M_t(\Theta_t) =
    \left[
    \begin{array}{cc}
        \cos(\Theta_t) & -\sin(\Theta_t) \\
        \sin(\Theta_t) & \cos(\Theta_t)
    \end{array}
    \right].
\end{equation*}
The coordinate change $w_t$ can be obtained using the following formula:
\begin{equation*}
    w_t = c M_t - c
\end{equation*}
Additionally, for cells that are not within the boundary, the corresponding $w_{[1, T-1]}$ is also directly set to $(0, 0)$. Therefore, all the data involved in this experiment are as follows: 
\begin{itemize}
    \item Coordinate change of cells $\w$: shape $[\tilde{T}, 2, 64, 64]$. For each step, we save the coordinate change of each cell within the boundary.
        \item PDE states $u$: shape $[\tilde{T}, 3, 64, 64]$. For each step, we save the states of the fluid field consisting of velocity in $x$ and $y$ directions and pressure. To save space, we downsample the resolution from $128\times128$ to $64\times64$.
    \begin{itemize}
        \item velocity: $[\tilde{T}, 2, 64, 64]$.
        \item pressure: $[\tilde{T}, 1, 64, 64]$.
    \end{itemize}
    \item opening angels $\Theta$: shape $[\tilde{T}]$. For each step, we save the opening angle in radians. 
    \item boundary mask and offsets $b$: $[\tilde{T}, 3, 64, 64]$. For each step, we save the mask of merged wings with half coordinates of boundary points and offsets in both $\x$ and $y$ directions. The resolution is $64\times64$, compatible with that of the states.
    \begin{itemize}
        \item mask: $[\tilde{T}, 1, 64, 64]$.
        \item offsets: $[\tilde{T}, 2, 64, 64]$.
    \end{itemize}
    \item force: shape $[\tilde{T}, 2]$. For each step, the simulator outputs the horizontal and vertical force from the fluid to the jellyfish. The horizontal force is regarded as a thrust to jellyfish if positive and a drag otherwise.
\end{itemize}

\subsection{Model}
Similar to the vanilla setting, this experiment also employs U-net as the backbone for the diffusion model. The U-net architecture remains consistent. The input of the U-net includes PDE state $\u ([T, 3, 64, 64])$, control sequence $\w ([T, 2, 64, 64])$, initial boundary mask and offset $([1, 3, 64, 64])$. It is worth noting that to align the initial mask and offset with others in terms of shape, we expand them along the time dimension, resulting in the final shapes of $[T, 3, 64, 64]$. Therefore, the shape of the input to the model is $[T, 8, 64, 64]$, and the output of the model includes both noise of the predicted state and the control sequence, with a shape of $[T, 5, 64, 64]$.

\subsection{Training, Inference, and Evaluation}
\paragraph{Training.}

During the training phase, the diffusion-flow task adopts a similar setup, utilizing the Mean Squared Error (MSE) between the model predictions and Gaussian noise as the loss function (Eq.\ref{eq:training_obj}). The batch size is chosen as 16, and the training involves 180,000 iterations. The learning rate starts at $1\times 10^{-4}$ and increases by a factor of 0.1 at the 50,000th and 150,000th iterations. More training details are presented in Table \ref{tab:2D_psi_arch_flow}.

\begin{table}[ht]
  \begin{center}
    \caption{\textbf{Hyperparameters of network architecture and training for the Finer-grained Jellyfish Boundary Control Task}.}
     \label{tab:2D_psi_arch_flow}
    \begin{tabular}{l|l} %
    \multicolumn{2}{l}{}\\
    \hline
      \text {Hyperparameter name} & {full observation} \\
      \hline

      Batch size &16 \\
      Optimizer &Adam \\
      Learning rate &0.001\\
      Loss function &MSE\\
      \hline
   \end{tabular}
  \end{center}
\end{table}

\paragraph{Inference.}
The pipeline of inference is shown in Figure \ref{fig:2d_flow_inference}.
Both diffused variables $\u_{[0, T]}$ and $\w_{[0, T]}$ are initialized from Gaussian prior and gradually denoised from denoising step $k = 1000$ to $k = 0$ based on denoising networks and guidance. Similar to the vanilla setting, due to the introduction of the initial boundary mask and offset as additional inputs, the number of channels in the model's input and output are inconsistent.
For guidance, we fix $\lambda = 5 \times 10^{-5}$ and we also utilized a surrogate model to approximate the force of fluid on the jellyfish. In the $k$-th denoising step, the inputs of this model include the noise-free state $\hat{\u}_{[0, T],k}$, $\hat{\w}_{[0, T], k}$ estimated from $\u_{[0, T], k}$,$\w_{[0, T], k}$ by Eq. \eqref{eq:estimate_x0}, initial boundary mask and offset. The model output is the force. The remaining aspects such as model architecture, training details, etc., remain the same. 
As for the regularization term $R(\w)$,  we approximate the opening angle of the jellyfish's wings at time $t$ using the change of the coordinate $(\Delta x_t$, $\Delta y_t)$ and the initial coordinates $(x, y)$ of each cell to achieve regularization.  Firstly, we utilize the following formula to approximate the angle change of each cell:
\begin{equation*}
\hat{\theta_t} = \sqrt{(\Delta x_t^2 + \Delta y_t^2)/(x^2 + y^2 )}
\end{equation*}
Then, we calculate the average value of the angles of different cells, which serves as an approximation of the angle between the two wings of the jellyfish at time $t$, denoted as $\hat{\Theta}_t$. It is evident that optimizing this term intuitively minimizes the coordinate changes of each cell, aligning with the objective of regularization.
The control objective $\J$ in Eq. \eqref{eq:jellyfish_obj} is computed as a summation of force and $R(\hat{\w}_{[0, T]})$, and we fix $\zeta = 1000$.  Then the gradients of the objective $\J$ in terms of $\hat{\u}_{[0, T]}$ and $\hat{\w}_{[0, T]}$ are computed and used in guidance.   
Because this experiment selects DiffConPDE-lite, all these gradients are directly subtracted from $[\mathbf{u}_{[0, T],k},\mathbf{w}_{[0, T],k}]$ to generate $[\mathbf{u}_{[0, T],k-1},\mathbf{w}_{[0, T],k-1}]$.

\begin{figure}[h]
\begin{center}
    \includegraphics[width=0.8\textwidth]{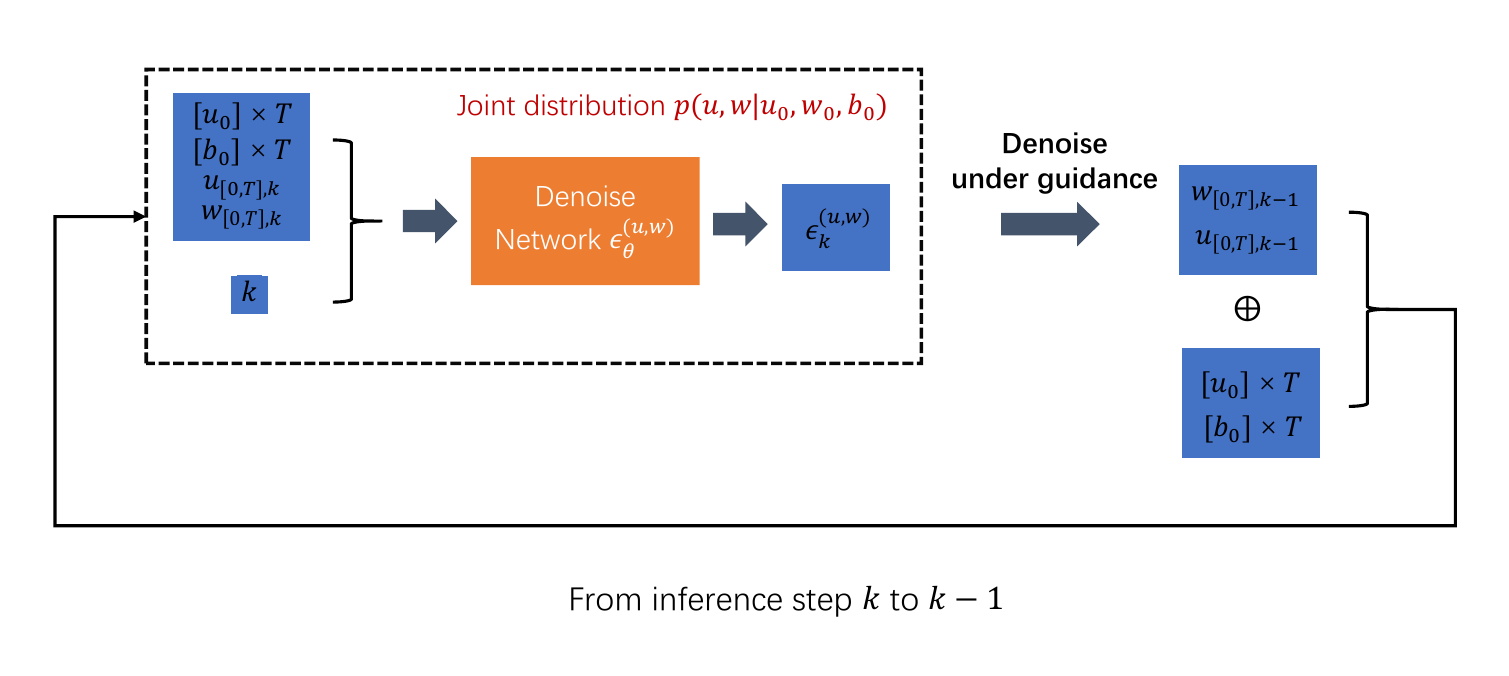}
\end{center}
\caption{\textbf{Inference of our finer-grained jellyfish boundary control task}.}
\label{fig:2d_flow_inference}
\vskip -0.1in
\end{figure}

\paragraph{Evaluation}
The inference outputs coordinate change for each cell of $T = 20$ steps for 50 testing samples. To align with the evaluation metrics of the vanilla setting, we first need to calculate the opening angle $\Theta_t$ using the coordinate change of the cells at time $t$. Firstly, at time $t \in [1, T-1]$, for each cell, we obtain the rotated coordinates $(x'_t, y'_t)$ using the initial coordinates $(x, y)$ and the coordinate change $(\Delta x_t, \Delta y_t)$. Then, combined with the coordinates of the rotation center $\mathbf{h}$, we can respectively obtain the radial distance of the cell relative to the rotation center before and after rotation, denoted as $r$ and $r'$. Then, according to the law of cosines, we can obtain the rotation angle $\delta \theta$ of the cell at time $t$ compared to the initial time:
$$
|\Delta \theta_t| = \frac{{r'_t}^2 + r_t^2 - (\delta x_t^2 + \delta y_t^2)}{2r'_tr_t}.
$$
Afterward, we take the average of the angle changes of all cells within the boundary as the angle change of the jellyfish's two wings at time $t$, denoted as $\Delta \Theta_t$.
Finally, based on the average value of the flow in the $x$-direction at time 
$t$, the opening or closing status of the jellyfish relative to the initial moment is determined. If it is negative, it indicates opening; otherwise, it suggests closing. This process allows us to obtain the specific angle change for each time step compared to the first moment:
\begin{eqnarray*}
\Theta_t = 
\begin{cases}
    \Theta_0 + |\Delta \Theta_t| &     \Delta \bar{x}_t < 0\\
    \Theta_0 - |\Delta \Theta_t| &    \Delta \bar{x}_t > 0\\
    \Theta_0 &                \Delta \bar{x}_t = 0.
\end{cases}
\end{eqnarray*}
After obtaining the theta sequence, we can follow the vanilla evaluation procedure.

\subsection{Results}
We present the results in Table \ref{tab:diff_flow_exp1}. From the table, it can be observed that our method optimized the objective $\mathcal{J}$ to 84.31 for this task. Compared to the results based on the original setting in Section 4.2, our method is still competitive with baselines (shown in Table \ref{tab:2D_results}) in this more challenging setting. An example of optimized soft boundary shapes is illustrated in Figure \ref{fig:diff_flow}.
These results demonstrate the good scalability of our method when facing higher-dimensional and complex control tasks. 
\begin{table}[h]
  \begin{center}
\caption{\textbf{Performance of \proj-lite on the finer-grained jellyfish boundary control task}.}
\begin{tabular}{ccc}
\toprule
 \multicolumn{1}{l}{Average speed ($\bar{v}$)} & \multicolumn{1}{l}{$R(\w)$} & objective ($\mathcal{J}$)              \\ 
 \midrule
-33.91 & 0.0504 & 84.31 \\
 \bottomrule
\end{tabular}
\label{tab:diff_flow_exp1}
\end{center}
\end{table}

\begin{figure}[h]
\begin{center}
    \includegraphics[width=1.0\textwidth]{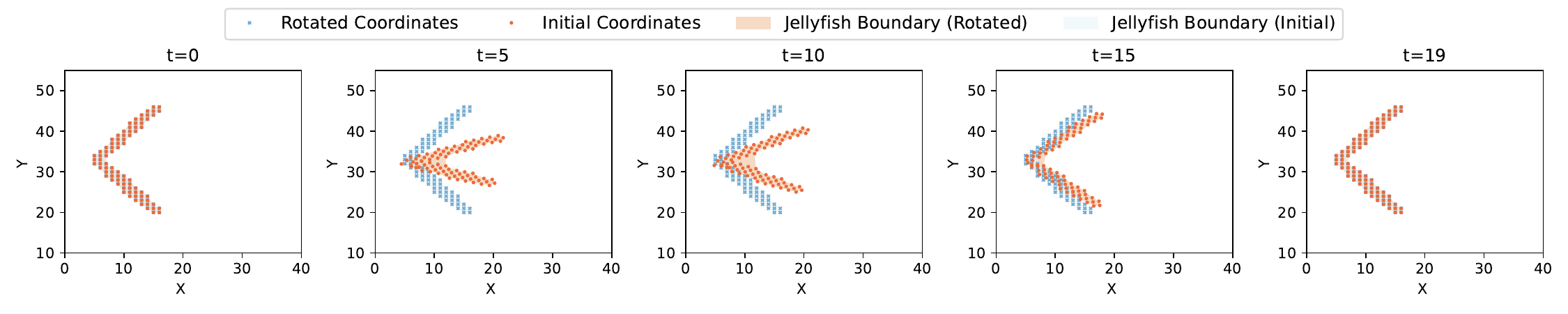}
\end{center}
\caption{\textbf{Inference results of soft boundaries on the finer-grained jellyfish control task}.}
\label{fig:diff_flow}
\vskip -0.1in
\end{figure}